\documentclass[12pt]{article}
\usepackage{amsmath,amsthm,amssymb}
\usepackage{float}
\usepackage{bm}
\usepackage{graphicx}
\usepackage{subfigure}
\usepackage{enumerate}
\usepackage{hyperref}
\usepackage{natbib}
\usepackage{url} 
\usepackage[dvipsnames]{xcolor}
\usepackage{algorithm}
\usepackage{algorithmic}

\hypersetup{
    colorlinks=true,
    linkcolor=Maroon,   
    urlcolor=NavyBlue,
    citecolor=Green,
}

\newcommand{\blind}{1}

\newtheorem{theorem}{Theorem}
\newtheorem{lemma}{Lemma}
\newtheorem{corollary}{Corollary}
\newtheorem{proposition}{Proposition}

\theoremstyle{definition}

\newtheorem{definition}{Definition}

\newtheorem{assumption}{Assumption}
\newtheorem{remark}{Remark}

\newcommand{\argmin}{\operatorname*{\arg\min}}
\newcommand{\E}{\operatorname{\mathbb{E}}} 
\newcommand{\Var}{\operatorname{Var}}

\newcommand{\bG}{{\mathbf G}}

\newcommand{\bI}{{\mathbf I}}

\newcommand{\bR}{{\mathbf R}}

\newcommand{\bU}{{\mathbf U}}

\newcommand{\bX}{{\mathbf X}}

\newcommand{\bp}{{\mathbf p}}
\newcommand{\br}{{\mathbf r}}

\newcommand{\bu}{{\mathbf u}}

\newcommand{\bx}{{\mathbf x}}

\newcommand{\bz}{{\mathbf z}}

\newcommand{\bGamma} {\boldsymbol{\Gamma}}

\providecommand{\Pbb}{\mathbb{P}}
\providecommand{\Ebb}{\mathbb{E}}
\providecommand{\ind}{\mathbf{1}}
\providecommand{\cD}{\mathcal{D}}
\providecommand{\FDP}{\operatorname{FDP}}
\providecommand{\FDR}{\operatorname{FDR}}
\providecommand{\Regret}{\operatorname{Reg}}
\providecommand{\KL}{\operatorname{KL}}
\providecommand{\Rbb}{\mathbb{R}}
\begin{document}

\def\spacingset#1{\renewcommand{\baselinestretch}%
{#1}\small\normalsize} \spacingset{1}


\newcommand{\Title}{Online Individualized Treatment Selection under Partial Feedback: False Discovery Control and Minimax Optimality}

\if1\blind
{
  \title{\bf \Title}
    \author{Runzhe Gu\thanks{These authors contributed equally to this work.}\\
    Center for Data Science, Zhejiang University\\
    and \\
    Zinan Zhao\footnotemark[1]\\
    Department of Statistics, University of Michigan\\
    and \\
    Wenguang Sun\\
    Center for Data Science, Zhejiang University\\
    and \\
    Xintao Xia\\
    Center for Data Science, Zhejiang University}
    \date{}
  \maketitle
} \fi

\if0\blind
{
  \bigskip
  \bigskip
  \bigskip
  \title{\bf \Title}
  \author{}
  \date{}
  \maketitle
  \medskip
} \fi

\bigskip
\begin{abstract}
We study online individualized treatment selection when only the outcome under the assigned treatment is observed. Our objective is to preferentially treat individuals whose potential outcomes under treatment exceed a covariate-specific control reference, while limiting false discoveries among treated individuals. We propose Partial-Feedback Adaptive Individualized Treatment Selection (PAITS), which initializes learning from historical data and combines sequential estimation of local null risks, randomized exploration, and feedback-driven updates of the treatment selection threshold. Under regularity conditions, we establish asymptotic false discovery rate control and sublinear gap-weighted assignment regret relative to a population oracle. We further establish a minimax lower bound on assignment regret under prescribed uniform nuisance-estimation requirements and give conditions under which PAITS attains the corresponding constrained minimax rate up to logarithmic factors. These results quantify the tradeoff between accurate treatment assignment and nuisance learning under partial feedback. Simulation studies and a real-data analysis using the Jester recommendation dataset illustrate the finite-sample performance of PAITS.
\end{abstract}

\noindent%

\newpage
\spacingset{1.65} 

\section{Introduction}
\label{sec:intro}

The increasing availability of large-scale, sequentially collected data has made online decision-making an important problem in modern data analysis. Online decision-making involves sequentially selecting actions using current covariates and past observations, while learning from their feedback to improve subsequent decisions \citep{sutton2018reinforcement,lattimore2020bandit,perchet2013multi,agarwal2014taming}. This setting arises naturally in many applications. In recommendation systems, a platform sequentially recommends items to users, observes responses such as clicks, purchases, or viewing behavior, and uses these responses to improve future recommendations \citep{li2010contextual}. In online advertising, an advertiser repeatedly chooses which advertisement to display and updates its allocation strategy according to observed user engagement or conversion outcomes \citep{geng2020online}. In adaptive clinical trials, treatments are assigned to incoming patients based on their characteristics and the outcomes accumulated from previously treated patients, so that the treatment rule can be improved as the trial progresses \citep{villar2015multi}. Similar problems also arise in dynamic pricing and resource allocation.

We consider sequential decisions about whether to assign treatment to each arriving individual, such as recommending a particular item rather than retaining a default option, or administering an aggressive treatment rather than the control condition. Each decision is made before the current outcome is known, using the individual's characteristics and all historical data. Only the outcome under the chosen option is observed; the outcome that the same individual would have experienced under the alternative remains unknown. We refer to this as \emph{partial feedback} \citep{li2010contextual}. With larger outcomes regarded as preferable, we aim to preferentially assign treatment to individuals whose outcomes under treatment exceed a covariate-specific control reference, such as the mean outcome under control for individuals with similar pretreatment characteristics. This formulation is motivated by established clinical practice in which outcomes among comparable patients receiving alternative treatments inform individualized treatment decisions \citep{kent2018personalized}. Assigning treatment to an individual whose outcome does not exceed this reference constitutes a \emph{selection error}. Failure to meet this criterion can be costly: patients may experience side effects, inconvenience, and financial costs without attaining the desired response, while unhelpful recommendations may consume users' attention and displace more relevant content.

Because treatment must be assigned before outcomes are observed, prediction models can guide the selection of individuals likely to meet the response criterion \citep{jin2023selection}. Many existing treatment-rule methods, however, optimize population-average outcomes \citep{qian2011performance,zhao2012estimating}. Favorable average performance need not imply a low selection-error rate: exceptionally strong responses among a few recipients may mask failures among a substantial proportion of those treated. This motivates explicitly limiting the proportion of treated individuals who fail to meet the response criterion. We formalize this objective through false discovery rate (FDR) control, defining the FDR as the expected proportion of selection errors among all treated individuals \citep{benjamini1995controlling,jin2023selection}.

The online nature of the problem adds a challenge to the error-control requirement. Individuals arrive sequentially, and treatment decisions rely on prediction models that are updated as outcomes accumulate. Under partial feedback, these decisions determine which treatment-specific outcomes become available. Repeatedly following current predictions may therefore leave one treatment condition insufficiently observed in parts of the covariate space \citep{perchet2013multi}. The procedure must balance exploiting current predictions to guide treatment selection with exploring alternative assignments to improve future estimation \citep{lai1985asymptotically,auer2002finite,lattimore2020bandit}. In our setting, this exploration--exploitation tradeoff must be addressed while controlling FDR across all treatment assignments, including those made during exploration. Moreover, because the control reference is unknown, even an observed treated outcome does not directly reveal whether a selection error has occurred; the feedback used for error control must itself be estimated.

In this paper, we develop an online framework for individualized treatment selection that integrates sequential learning with selection error control under partial feedback. Our objective is to learn a rule that treats as many individuals as possible while controlling the expected proportion whose outcomes fail to exceed the control reference. Starting from historical data, the procedure sequentially updates estimates of the control reference and treated-outcome distribution as new outcomes become available. These estimates guide treatment assignment through an adaptive selection threshold, while randomized exploration supports continued learning under both treatment conditions.

\subsection{Related Work}
\label{subsec:intro-related-work}

Our work is related to the literature on individualized treatment rules (ITRs), which use patient characteristics to guide treatment assignment, typically with the objective of maximizing expected outcomes \citep{manski2004statistical,qian2011performance}. Representative approaches include outcome-weighted learning \citep{zhao2012estimating} and residual-weighted learning \citep{zhou2017residual}, doubly robust and semiparametric policy learning \citep{zhang2012robust,zhao2015doubly,athey2021policy,zhou2023offline}, and empirical welfare maximization \citep{kitagawa2018who}. Recent developments address distributional robustness and estimation efficiency \citep{mo2021learning,mo2022efficient}, adaptively collected data and limited overlap \citep{zhan2024policy,jin2025policy}, and safety through asymmetric utilities, counterfactual harm constraints, or improvement over an existing policy \citep{benmichael2024asymmetric,li2023trustworthy,wu2025safe,benmichael2025safe}. Existing methods often emphasize average outcomes rather than individual response criteria. We instead aim to control the expected proportion of treated individuals whose outcomes do not exceed a covariate-specific control reference.

FDR is one of the most widely used error criteria in multiple testing. The Benjamini--Hochberg procedure controls FDR through a step-up rule applied to ordered $p$-values \citep{benjamini1995controlling}, with subsequent work addressing dependence and adaptive estimation of the null proportion \citep{benjamini2001control,storey2002direct}. Empirical Bayes methods instead estimate local null probabilities for ranking and selection \citep{efron2001empirical,sun2007oracle}. For sequentially arriving hypotheses, alpha-investing allocates an error budget over time \citep{foster2008alpha}, and later procedures establish online FDR control with adaptive thresholds \citep{javanmard2018online,ramdas2018saffron}. Further developments combine multiple testing with adaptive sampling \citep{xu2021unified}. Many existing testing procedures take valid $p$-values as inputs, but such inputs are not directly available in our setting. Our procedure instead relies on local null risks estimated from adaptively collected outcomes under partial feedback.

A related line of work studies selection procedures with explicit statistical error control. \citet{jin2023selection} combine conformal $p$-values with multiple testing to select individuals whose unobserved outcomes exceed prespecified thresholds. The ACS framework extends conformal selection to adaptive information acquisition and interactive model updates while retaining FDR control \citep{gui2025acs}. For treatment-effect discovery, \citet{duan2024interactive} develop $I^3$, which uses masking and progressive unmasking of treatment assignments to identify treatment-effect signals among participants in randomized experiments with finite-sample FDR control. Adaptive conformal inference updates its calibration level using observed coverage errors to control long-run miscoverage \citep{gibbs2021adaptive,gibbs2024conformal}. Related work studies decaying step sizes \citep{angelopoulos2024decaying}, more general online risk control \citep{feldman2023risk}, and the relationship between regret guarantees and calibration \citep{ramalingam2025relationship}. Online selective conformal inference further controls miscoverage among selected observations while allowing prediction scores to evolve \citep{humbert2025online}. \citet{liang2026optimal} study cumulative deviations of training-conditional coverage from its nominal level under unknown distribution drift. These works, however, do not address the joint problem of controlling FDR and achieving low assignment regret in online individualized treatment selection.

\subsection{Our contributions and paper organization}
\label{subsec:intro-contributions}

We propose Partial-Feedback Adaptive Individualized Treatment Selection (PAITS), an online procedure for learning individualized treatment assignments with asymptotic FDR control. Our main contributions are as follows.

First, we formulate individualized treatment assignment as a sequential selection problem, with selection errors defined relative to a covariate-specific control reference. The formulation requires balancing exploration and exploitation while accounting for errors from all treatment assignments. We characterize the oracle rule that maximizes the population treatment rate subject to a marginal FDR constraint, providing a target for sequential learning and a reference for regret analysis.

Second, we construct an online procedure that combines feedback-driven threshold adaptation with randomized exploration and achieves asymptotic FDR control. The threshold update uses plug-in null indicators based on an estimated control reference. We derive convergence rates that account for nuisance estimation, threshold adaptation, and exploration, with guarantees covering treatment assignments, including exploratory treatments.

Third, we show that the proposed method with specific estimators is minimax optimal up to logarithmic factors among procedures satisfying prescribed uniform nuisance-estimation requirements. Our regret upper bound separates the contributions of nuisance estimation, threshold adaptation, and exploration, while lower bounds quantify the assignment cost of meeting these learning requirements. These results characterize the trade-off between improving nuisance-estimation accuracy and reducing assignment regret.

The remainder of the paper is organized as follows. Section~2 formulates the online treatment selection problem and defines the FDR and regret criteria. Section~3 presents PAITS, including nuisance estimation, randomized exploration, and adaptive cutoff updates. Section~4 establishes asymptotic FDR control and regret bounds, together with constrained minimax optimality. Sections~5 and~6 examine finite-sample performance through numerical experiments, and Section~7 concludes the paper.

\section{Problem Formulation and Performance Criteria}

This section formulates the online treatment assignment problem and introduces the error-control and regret criteria used to evaluate the proposed procedure.

\subsection{Individualized treatment assignment}
\label{subsec:causal-setting}

We study individualized treatment assignment in a sequential setting. An initial rule trained on historical data guides treatment for each arriving unit and is adaptively estimated using the outcomes generated by its own assignments. We adopt the potential-outcomes framework \citep{rubin1974estimating}. Let $(\bX,Y(0),Y(1))$ be a random vector, where $\bX$ denotes the pretreatment covariates and $Y(a)\in\mathbb R$ denotes the potential outcome under treatment $a\in\{0,1\}$. Let $A\in\{0,1\}$ denote the treatment assignment, with $A=0$ and $A=1$ corresponding to control and treatment, respectively. Under consistency, the observed outcome equals the potential outcome under the treatment received: $Y=Y(A)=AY(1)+(1-A)Y(0)$. We further assume that the treatment levels are well defined and that there is no interference across units. Together with consistency, these conditions constitute the standard \textit{stable unit treatment value assumption} (SUTVA; \citealp{imbens2015causal}).

Before online treatment assignment begins, we observe a historical sample $\mathcal D_{\mathrm{hist}}=\{(\bX_i^{\mathrm{hist}},A_i^{\mathrm{hist}},Y_i^{\mathrm{hist}})\}_{i=1}^{n_{\mathrm{hist}}}$, which is used to construct the initial individualized treatment rule. The historical sample initializes the procedure and is retained in subsequent nuisance fits. For clarity, we focus on a randomized trial. Specifically, treatment assignments are independent across units
and satisfy
$$
\mathbb P\{A_i^{\mathrm{hist}}=1|\bX_i^{\mathrm{hist}},Y_i^{\mathrm{hist}}(0),Y_i^{\mathrm{hist}}(1)\}=\pi,
$$
where $\pi\in(0,1)$. Under random assignment, the treatment-specific conditional distribution functions $F_a(y\mid \mathbf{x})$, $a\in\{0,1\}$, are identified from the observed-data distribution. Consequently, their functionals, including the conditional means $\mu_a(\mathbf x)=\mathbb E\{Y(a)\mid \mathbf X= \mathbf x\}$ and the conditional quantiles, are also identified.

After the historical stage, units arrive sequentially, and we assign treatment to each unit upon arrival. We assume that $\{\bX_t,Y_t(0),Y_t(1)\}_{t\geq 1}$ are \textit{independent and identically distributed} (i.i.d.) copies of $(\bX,Y(0),Y(1))$ and are independent of $\mathcal D_{\mathrm{hist}}$. This condition is similar to the time-homogeneity condition standard in reinforcement learning, under which the environment remains invariant over time \citep[see, e.g.,][]{agarwal2014taming,sutton2018reinforcement}. Throughout, we regard larger outcome values as more desirable.

Our objective is to preferentially assign treatment to units whose potential outcomes under treatment exceed covariate-specific baselines derived from the control outcome distribution. Our formulation is motivated by established practices in clinical research \citep{huang2009biomarker,kent2018personalized} and econometric policy analysis \citep{kitagawa2018who,chernozhukov2013inference}. Let $b_0(\mathbf x)$ denote a prespecified scalar summary of the conditional control distribution $F_0(\cdot\mid\mathbf  x)$, such as its mean, median, or a prespecified quantile. We refer to $b_0(\mathbf x)$ as the control reference for units with covariates $\mathbf x$. Following standard terminology in causal inference, we refer to $F_1(\cdot\mid\bx)$ and $b_0(\bx)$ as \emph{nuisance functions}. When unit $t$ arrives, let $\mathcal F_{t-1}$ denote the information available before the arrival of unit $t$. The treatment assignment $A_t\in\{0,1\}$ is chosen using only $\mathcal F_{t-1}$ and $\mathbf X_t$ and thus:
\begin{equation}
    A_t\perp\!\!\!\perp\{Y_t(0),Y_t(1)\}\mid \mathcal F_{t-1},\bX_t.
    \label{eq:generic-sequential-randomization}
\end{equation}
Following assignment, we observe only the potential outcome under the assigned treatment, $Y_t=A_tY_t(1)+(1-A_t)Y_t(0)=Y_t(A_t)$, while the counterfactual outcome remains unobserved. We formulate treatment assignment as an online multiple-testing problem. For the unit arriving at time $t$, consider the null and alternative hypotheses
\begin{equation}
\mathbb H_{0,t}:Y_t(1)\leq b_0(\mathbf X_t),
\qquad
\mathbb H_{1,t}:Y_t(1)>b_0(\mathbf X_t).
\label{eq:reference-hypotheses}
\end{equation}
Rejection of $\mathbb H_{0,t}$ corresponds to treatment assignment, reflecting our goal of prioritizing units whose potential outcomes under treatment are likely to exceed their control reference.

\subsection{False discovery rate control and regret}
\label{subsec:treatment-error-regret}

The hypotheses in \eqref{eq:reference-hypotheses} turn sequential treatment assignment into a multiple-testing problem. We identify treatment assignment with rejection, so that $A_t=1$ corresponds to rejecting $H_{0,t}$. At horizon $T$, let $\mathcal R_T=\{t\le T: A_t=1\}$ and $\mathcal H_{0, T}=\{t\le T: Y_t(1)\le b_0(\bX_t)\}$ denote the rejection set and the set of true null hypotheses, respectively. The numbers of treated individuals and false discoveries are therefore
\begin{equation*}
    R_T=|\mathcal R_T|=\sum_{t=1}^T A_t,
    \qquad
    V_T=|\mathcal R_T\cap\mathcal H_{0,T}|=\sum_{t=1}^T A_t\ind\{Y_t(1)\le b_0(\bX_t)\}.
\end{equation*}
The corresponding false discovery proportion and false discovery rate are
\begin{equation*}
    \FDP_T=\frac{V_T}{R_T\vee 1},
    \qquad
    \FDR_T=\E(\FDP_T).
\end{equation*}
$\FDP_T$ is the realized proportion of treated units whose outcomes under treatment do not exceed their covariate-specific control reference, and $\FDR_T$ is its expectation. This criterion complements population-average policy value: a policy may perform well on average while treating an unacceptably large proportion of units who fail to achieve desired performance. For a prespecified level $\alpha\in(0,1)$, our primary objective is to control $\FDR_T$ at level $\alpha$.

We seek to maximize the treatment rate subject to FDR control at the prespecified level. To assess assignment efficiency, we introduce an oracle-based regret criterion. Motivated by local false discovery rate methods for large-scale multiple testing \citep{efron2001empirical,sun2007oracle,sun2012multiple}, we define the oracle local null risk by
\begin{equation}
    \ell(\bx)=\Pbb(Y_t(1)\le b_0(\bX_t)\mid \bX_t=\bx)=F_1\{b_0(\bx)\mid \bx\}.
    \label{eq:problem-oracle-local-risk}
\end{equation}
For the oracle comparison, we use the \emph{marginal false discovery rate} (mFDR; \citealp{storey2003positive,genovese2002operating}) defined, whenever $\Ebb(R_T)>0$, by
\begin{equation*}
    \operatorname{mFDR}_T
    :=\frac{\Ebb(V_T)}{\Ebb(R_T)}.
\end{equation*}
For a fixed rule with positive treatment probability, $\FDR_T=\operatorname{mFDR}_T\Pbb(R_T>0)$. Thus, $\operatorname{mFDR}_T\le\alpha$ implies $\FDR_T\le\alpha$ for every $T$, and $\lim_{T\to\infty}|\FDR_T-\operatorname{mFDR}_T|=0$. Analogous to the relationship between local and tail-area Bayesian FDR \citep{efron2010large}, averaging $\ell(\bX)$ over selected individuals gives the population selection-error rate: $\Ebb[\ell(\bX)\mid\ell(\bX)\le\lambda]=\Pbb\{Y(1)\le b_0(\bX)\mid\ell(\bX)\le\lambda\}$. We therefore define the oracle cutoff as the supremum of thresholds for which the average local null risk among selected units does not exceed $\alpha$:
$$\lambda^\star=\sup\{\lambda\in[0,1]:\Ebb[\ell(\bX)|\ell(\bX)\leq\lambda]\leq\alpha\}.$$ 
Define the oracle assignment by $A^\star=\ind\{\ell(\bX)\le\lambda^\star\}$. We consider covariate-based assignment rules, possibly randomized, satisfying $\operatorname{mFDR}_T\le\alpha$, and evaluate them by the population treatment rate $\Ebb(A)$ and the probability of a successful treatment assignment $\Ebb[\{1-\ell(\bX)\}A]$. The following lemma establishes the optimality of $A^\star$ under both criteria.

\begin{lemma}[Optimality of the oracle assignment]
\label{lem:oracle-assignment-optimality}
Suppose that $\lambda^\star>\alpha$, $\Pbb\{\ell(\bX)\le\lambda^\star\}>0$, $\Pbb\{\ell(\bX)=\lambda^\star\}=0$, and $\Ebb[\ell(\bX)\mid\ell(\bX)\le\lambda^\star]=\alpha$. Then $A^\star$ maximizes both criteria among all possibly randomized covariate-based assignment rules satisfying $\operatorname{mFDR}_T\le\alpha$.
\end{lemma}

FDR control limits selection errors but does not by itself characterize assignment efficiency. We therefore compare the procedure with the oracle rule. This comparison yields a gap-weighted regret criterion that measures cumulative departures from the oracle. Let $\mathcal F_{t-1}$ denote the information available before unit $t$ arrives, including the historical data. For a nonanticipating assignment sequence $\{A_t\}_{t\ge1}$, each decision uses only $\mathcal F_{t-1}$, the current covariates $\bX_t$, and independent randomization. Let $A_t^\star=\ind\{\ell(\bX_t)\le\lambda^\star\}$ denote the oracle assignment for unit $t$. Define the predictable and expected cumulative gap-weighted assignment regrets as
\begin{equation}
\begin{aligned}
    \widehat{\Regret}_T&=\sum_{t=1}^T\Ebb\bigg[|\ell(\bX_t)-\lambda^\star|\cdot|A_t-A_t^\star|\bigg|\mathcal F_{t-1}\bigg],
    \\
    \Regret_T&=\Ebb\bigg[\sum_{t=1}^T|\ell(\bX_t)-\lambda^\star|\cdot|A_t-A_t^\star|\bigg]=\Ebb[\widehat{\Regret}_T].
\end{aligned}
\label{eq:weighted-selection-regret}
\end{equation}
This definition is motivated by gap-weighted regret in contextual bandits \citep{perchet2013multi,agarwal2014taming}. An assignment incurs no loss when it agrees with the oracle and a loss of $|\ell(\bX_t)-\lambda^\star|$ otherwise. The criterion captures two departures from the oracle: treating an individual whose local null risk exceeds the selection threshold and withholding treatment from one whose risk falls below it. The gap measures the strength of the oracle's preference between the two assignments: disagreements far from the threshold are penalized more heavily, whereas those near the boundary incur little loss. Thus, regret accounts for both the frequency of disagreements and their importance under the oracle criterion. Because it is defined through the realized assignment $A_t$, it includes losses from both plug-in decisions and randomized exploration.

\section{Adaptive Individualized Treatment Selection under Partial Feedback}
\label{sec:method}

We propose \emph{Partial-Feedback Adaptive Individualized Treatment Selection} (PAITS), which combines adaptive nuisance estimation, adaptive cutoff updates, and randomized exploration. The cutoff updates target asymptotic FDR control, while suitable nuisance estimators and an appropriately chosen exploration schedule yield the constrained minimax regret rate up to logarithmic factors under uniform regularity conditions. We next detail these components and present the complete PAITS procedure.

\subsection{Local null risk estimation and plug-in treatment selection}
\label{subsec:plugin-selection}
\label{subsec:paits-selection-update}

The oracle local null risk \eqref{eq:problem-oracle-local-risk} is the conditional probability that treating a unit given its covariates produces a false discovery. Motivated by the oracle characterization, PAITS determines treatment selection by replacing $\ell(\mathbf x)$ with its plug-in estimate and comparing the estimate with an adaptively updated cutoff. Both the local null risk estimator and the cutoff are updated adaptively as outcomes accumulate.

Let $\cD_a$ denote the historical observations in $\mathcal D_{\mathrm{hist}}$ assigned treatment status $a\in\{0,1\}$. When unit $t$ arrives, the avaliable training sample for $a\in\{0,1\}$ is $\cD_{a,t}=\cD_a\cup\{(\mathbf X_s,Y_s):s<t,\ A_s=a\}$. Let $\widehat F_{1,t}(\cdot\mid\bx)$ and $\widehat b_{0,t}(\bx)$ denote estimators of the conditional treated-outcome distribution and the control reference, fitted using $\cD_{1,t}$ and $\cD_{0,t}$, respectively. Combining the two estimators gives the plug-in estimate of the local null risk,
\begin{equation}
    \widehat\ell_t(\bx)=\widehat F_{1,t}\{\widehat b_{0,t}(\bx)\mid \bx\}.
    \label{eq:paits-plugin-local-null-risk}
\end{equation}
To adjust for the selective sampling induced by adaptive treatment assignment, we use stabilized inverse-propensity-weighted spline or local-linear regression. Propensity weighting adjusts for covariate imbalance induced by adaptive treatment assignment, following weighted estimation methods for contextual bandits \citep{chen2021statistical}. Each observation is weighted using its treatment-assignment probability at the time of collection, allowing both exploratory and plug-in assignments to contribute to nuisance estimation. Appendix~\ref{app:weighted-nuisance-proofs} provides the estimator constructions and their uniform convergence guarantees.

We select units by thresholding their estimated local null risks. Let $\lambda_t\in[\alpha,1]$ denote the cutoff before the time-$t$. The resulting plug-in treatment recommendation is
\begin{equation}
    G_t=\ind\{\widehat\ell_t(\bX_t)\leq\lambda_t\}.
    \label{eq:paits-plugin-recommendation}
\end{equation}
Thus, $G_t=1$ recommends treatment when the estimated local null risk does not exceed $\lambda_t$. Using $A_t=G_t$ may prevent $\widehat\ell_t$ from consistently estimating $\ell$. Because we observe only the outcome under the assigned treatment, a deterministic plug-in rule may stop sampling one treatment in some regions of the covariate space. This is the classic exploration--exploitation problem in contextual bandits: without additional distributional conditions, a purely greedy assignment rule may fail to learn the optimal policy \citep{kannan2018smoothed}.

\subsection{Randomized exploration and adaptive data collection}
\label{subsec:epsilon-greedy-design}

To balance exploration and exploitation in sequential treatment assignment, PAITS incorporates $\varepsilon$-greedy randomization, which is widely used in reinforcement learning under action-dependent feedback \citep{langford2007epoch}. Fix $\varepsilon_0\in(0,1)$ and $\chi\in(0,1)$, and let $ \varepsilon_t=\varepsilon_0t^{-\chi}$. Then $\varepsilon_t\to0$, $\sum_{t\geq1}\varepsilon_t=\infty$, and $T^{-1}\sum_{t=1}^T\varepsilon_t=O(T^{-\chi})$. The expected number of explorations diverges, whereas their expected proportion vanishes, and the exponent $\chi$ governs the rate of this transition.

After observing $\bX_t$ and computing $G_t$, PAITS draws an exploration indicator $B_t$ using fresh randomization independent of the current potential outcomes, with $\Pbb(B_t=1)=\varepsilon_t$. If $B_t=1$, it independently draws $Z_t\sim\operatorname{Bernoulli}(\pi_{\mathrm{exp}})$ for a fixed $\pi_{\mathrm{exp}}\in(0,1)$; set $Z_t=0$ when $B_t=0$. Define the indicator of treatment assignment under the plug-in rule by $S_t=(1-B_t)G_t$, and the realized treatment assignment by $A_t=S_t+B_tZ_t$. The two terms correspond to mutually exclusive assignment mechanisms: $S_t=1$ indicates treatment assigned by the plug-in rule, whereas $B_tZ_t=1$ indicates treatment assigned through exploration. Conditional on exploration, we assign treatment and control with probabilities $\pi_{\mathrm{exp}}$ and $1-\pi_{\mathrm{exp}}$, respectively. All observed outcomes, whether obtained under the plug-in rule or through exploration, are used for nuisance estimation.

\subsection{Adaptive learning of the treatment selection cutoff}
\label{subsec:paits-cutoff-update}

Given the current estimate $\widehat\ell_t$, PAITS updates the selection cutoff using outcome feedback from units treated by the plug-in rule. Specifically, when $S_t=1$, we have $A_t=1$, and the observed treated outcome provides feedback for updating $\lambda_t$. PAITS summarizes this feedback through the plug-in null indicator
\begin{equation}
    \widehat E_t
    =
    \ind\{Y_t\leq\widehat b_{0,t}(\bX_t)\}.
    \label{eq:paits-plugin-null-indicator}
\end{equation}
Thus, $\widehat E_t=1$ records a plug-in null event: the observed treated outcome fails to exceed the estimated control reference. For a prespecified positive step-size sequence $\{\eta_t\}$, define  $\Pi_{[\alpha,1]}(u)=\min\{1,\max\{\alpha,u\}\}$ and update
\begin{equation}
    \lambda_{t+1}
    =
    \begin{cases}
    \Pi_{[\alpha,1]}\big[
        \lambda_t+\eta_t\{\alpha-\widehat E_t\}
    \big], & S_t=1,\\
    \lambda_t, & S_t=0,
    \end{cases}
    \label{eq:paits-cutoff-update}
\end{equation}
where $\eta_t>0$ is the step size. This update is a projected Robbins--Monro stochastic approximation \citep{robbins1951stochastic}. When $\widehat E_t=1$, indicating a plug-in null event, the cutoff decreases, and subsequent treatment recommendations become more conservative. When $\widehat E_t=0$, the cutoff increases and subsequent recommendations become more aggressive. Projection onto $[\alpha,1]$ enforces the admissible range of the cutoff.

\begin{remark}
     The cutoff recursion is related to adaptive conformal inference (ACI; \citealp{gibbs2021adaptive}) and its selective extension, OnlineSCI \citep{humbert2025online}, which uses sequential binary feedback. In their conformal settings, miscoverage is observable once the response is revealed. In our setting, however, the true null indicator is not directly observable after treatment, because the control reference $b_0$ is unknown. PAITS therefore updates its selection cutoff using an estimated null indicator. Beyond error control, PAITS addresses the exploration--exploitation tradeoff. To assess assignment efficiency, we characterize an oracle rule and measure gap-weighted regret relative to this rule.
\end{remark}

The cutoff is updated only when $S_t=1$, that is, when treatment is assigned by the plug-in rule. This restriction mirrors the definition of $\lambda^\star$, which depends on the average local null risk among selected units. On these rounds, $\widehat E_t$ provides plug-in feedback about the null frequency within the selected population. By contrast, exploratory treatment is assigned independently of the current cutoff and therefore provides no direct feedback about the null frequency among units selected by that cutoff. This exclusion only applies to the cutoff update. Units treated during exploration remain in the treated set and therefore contribute to the FDP and FDR, while all exploratory assignments enter the regret criterion through $A_t$. Algorithm~\ref{alg:paits} summarizes the complete PAITS procedure.

\begin{algorithm}[ht]
\caption{Partial-feedback adaptive individualized treatment selection (PAITS)}
\label{alg:paits}
\begin{algorithmic}[1]
\REQUIRE Historical samples $\mathcal D_{\mathrm{hist}}$; target FDR level $\alpha$; horizon $T$; initial cutoff $\lambda_1\in[\alpha,1]$; step sizes $\{\eta_t\}_{t=1}^T$; exploration parameters $(\varepsilon_0,\chi,\pi_{\mathrm{exp}})$.

\STATE Set $\cD_{a,1}\gets\cD_a$ for each $a\in\{0,1\}$ and fit $\widehat F_{1,1}$ using $\cD_{1,1}$ and $\widehat b_{0,1}$ using $\cD_{0,1}$.
\FOR{$t=1,\ldots,T$}

    \STATE Observe $\bX_t$, compute $\widehat\ell_t(\bX_t)$ by \eqref{eq:paits-plugin-local-null-risk}, and set $G_t\gets\ind\{\widehat\ell_t(\bX_t)\leq\lambda_t\}$.

    \STATE Set $\varepsilon_t\gets\varepsilon_0t^{-\chi}$ and draw $B_t\sim\operatorname{Bernoulli}(\varepsilon_t)$, $Z_t\sim\operatorname{Bernoulli}(\pi_{\mathrm{exp}})$.
 
    \STATE Set $S_t\gets(1-B_t)G_t$ and $A_t\gets S_t+B_tZ_t$. \STATE Assign treatment $A_t$ and observe $Y_t=Y_t(A_t)$.

    \IF{$S_t=1$}
        \STATE Construct $\widehat E_t$ according to
        \eqref{eq:paits-plugin-null-indicator} and update $\lambda_{t+1}$ according to \eqref{eq:paits-cutoff-update}.
    \ELSE
        \STATE Set $\lambda_{t+1}\gets\lambda_t$.
    \ENDIF

    \STATE Update $\cD_{1,t+1}$ and $\cD_{0,t+1}$ accordingly.
    \STATE Refit $\widehat F_{1,t+1}$ using $\cD_{1,t+1}$ and
    $\widehat b_{0,t+1}$ using $\cD_{0,t+1}$.
\ENDFOR
\end{algorithmic}
\end{algorithm}

\section{Theoretical Guarantees}
\label{sec:theory}

We first establish asymptotic FDR control and regret bounds for PAITS. We then examine the minimax optimality of the resulting regret rate. We begin by introducing the notation. Let $E_t=\ind\{Y_t(1)\leq b_0(\bX_t)\}$ denote the true null indicator for unit $t$. We write $N_T=\sum_{t=1}^T S_t$ and $Q_T=\sum_{t=1}^T B_tZ_t$ for the numbers of units treated by the plug-in rule and through exploration, respectively. The total number of treated units is $R_T=\sum_{t=1}^T A_t=N_T+Q_T$. We use the $\mathcal F_{t-1}=\sigma(\mathcal D_{\mathrm{hist}},(\bX_s,B_s,S_s,Z_s,A_s,Y_s)_{1\le s<t})$ to describe the information available before unit $t$ arrives. It contains the historical data and all information collected from previous online units. We use $\mathcal G_t=\mathcal F_{t-1}\vee\sigma(\bX_t)$ to denote the join of two sigma-fields. 

\subsection{FDR control and regret bounds}
\label{subsec:fdr-regret}

We establish asymptotic FDR control and regret bounds under the following regularity conditions.

\begin{assumption}
\label{ass:local-density}
Let $U=Y(1)-b_0(\bX)$. For $P_X$-almost every $\bx$, the conditional distribution of $U$ given $\bX=\bx$ has a density function $f_{U\mid \bX}(u\mid x)$. And there exist constants $\delta_U>0$ and
$0<C_U<\infty$ such that
$$
    \operatorname*{ess\,sup}_{\bx\sim P_\bX}\sup_{|u|\le\delta_U}f_{U\mid \bX}(u\mid \bx)\le C_U, \quad u\in[-\delta_U,\delta_U].
$$
\end{assumption}

Assumption~\ref{ass:local-density} limits the concentration of treated outcomes near the control reference, uniformly over covariate values. It holds for continuous-outcome models whose conditional densities are uniformly bounded in a neighborhood of $b_0(\bx)$, while imposing no density restrictions beyond that neighborhood.

\begin{assumption}[Oracle-score regularity]
\label{ass:oracle-score}
The score $\ell(\bX)$ has an absolutely continuous distribution on $[0,1]$, with distribution function $H(\cdot)$ and bounded density $h(\cdot)$. For some $\lambda_+\in(\alpha,1)$, we have $H(\alpha)>0$, $\Ebb\{\ell(\bX)\mid\ell(\bX)\le\lambda_+\}>\alpha$, and $h(\cdot)$ is continuous and strictly positive on $[\alpha,\lambda_+]$.
\end{assumption}

Assumption~\ref{ass:oracle-score} ensures that the oracle selection problem is nontrivial. The condition $H(\alpha)>0$ guarantees a positive probability of scores below $\alpha$, whereas $\Ebb\{\ell(\bX)\mid\ell(\bX)\le\lambda_+\}>\alpha$ implies $\Ebb\{\ell(\bX)\}>\alpha$, excluding the possibility of assigning treatment to the entire population. For $\lambda\in[\alpha,1]$, define $q(\lambda)=\Ebb\{\ell(\bX)\mid\ell(\bX)\le\lambda\}$ as the population average null risk among individuals selected at the threshold $\lambda$. The following proposition establishes the existence and uniqueness of an interior oracle threshold $\lambda^\star\in(\alpha,\lambda_+)$ and a global contraction inequality.

\begin{proposition}[Oracle threshold and global contraction]
\label{prop:oracle-threshold}
Under Assumption~\ref{ass:oracle-score}, there exists a unique
$\lambda^\star\in(\alpha,\lambda_+)$ satisfying $q(\lambda^\star)=\alpha$. The function $q$ is nondecreasing and Lipschitz on $[\alpha,1]$, and there exists a constant $c_g>0$ such that
\begin{equation*}
 (\lambda-\lambda^\star)H(\lambda)\{\alpha-q(\lambda)\}\le-c_g(\lambda-\lambda^\star)^2,
 \quad\text{for}\quad\lambda\in[\alpha,1]
\end{equation*}
\end{proposition}

We formulate the smoothness conditions on the nuisance functions using H\"older classes, a standard framework in nonparametric estimation \citep{tsybakov2009introduction}.

\begin{definition}[H\"older class]
\label{def:holder-class}
Let $\mathcal X\subset\mathbb R^d$ be a compact set, and fix a bounded open convex neighborhood $\mathcal U\supset\mathcal X$. For $s>0$ and $0<L<\infty$, set $r=\lceil s\rceil-1$. The H\"older class $\mathcal H^s(\mathcal X,L)$ consists of the restrictions to $\mathcal X$ of functions $g:\mathcal U\to\mathbb R$ having continuous partial derivatives through order $r$ and satisfying
\begin{equation*}
 \max_{|\nu|\le r}\sup_{\bx\in\mathcal U}|D^\nu g(\bx)|\le L,
 \qquad
 \max_{|\nu|=r}\sup_{\substack{\bx,\bz\in\mathcal U;\bx\ne\bz}}
 \frac{|D^\nu g(\bx)-D^\nu g(\bz)|}{\|\bx-\bz\|^{s-r}}\le L.
\end{equation*}
Here $\nu\in\mathbb N_0^d$ is a multi-index, $|\nu|=\sum_{j=1}^d\nu_j$, and $D^\nu$ denotes the corresponding partial derivative, with $D^0g=g$.
\end{definition}

\begin{assumption}[H\"older model]
\label{ass:nonparametric-model}
The covariate vector $\bX$ has compact support $\mathcal X\subset\mathbb R^d$ and a density satisfying $0<c_X\le f_X(\bx)\le C_X<\infty$ for Lebesgue-almost every $\bx\in\mathcal X$. For some fixed $s>0$ and $0<L<\infty$, $b_0\in\mathcal H^s(\mathcal X,L)$, $F_1(y\mid\cdot)\in\mathcal H^s(\mathcal X,L)$ uniformly in $y\in\mathbb R$.
\end{assumption}

Assumption~\ref{ass:nonparametric-model} imposes smoothness restrictions on the nuisance functions while leaving the choice of estimator open. We consider nuisance estimators satisfying the following time-averaged mean-square bounds on their uniform estimation errors:
\begin{align}
\frac1T\sum_{t=1}^T\Ebb[(\delta_{b,t}^{\infty})^2]=O(r_{b,T}^2),\quad
\frac1T\sum_{t=1}^T\Ebb[(\delta_{F,t}^{\infty})^2]=O(r_{F,T}^2),\quad T\to\infty,
\label{eq:nuisance_rate}
\end{align}
where $\delta_{b,t}^{\infty}=\operatorname*{ess\,sup}_{\bx\sim P_X}|\widehat b_{0,t}(\bx)-b_0(\bx)|$ and $\delta_{F,t}^{\infty}=\operatorname*{ess\,sup}_{\bx\sim P_X}\sup_{|u|\le\delta_U}|\widehat F_{1,t}\{b_0(\bx)+u\mid\bx\}-F_1\{b_0(\bx)+u\mid\bx\}|$. Appendix~\ref{app:weighted-nuisance-proofs} verifies these bounds using stabilized inverse-propensity-weighted spline regression \citep{chen2015optimal} and local-linear regression, respectively. The control reference is estimated by conditional-mean regression, and the treated-outcome CDF by regression of $\ind\{Y\le y\}$ followed by truncation and monotonicization. Under Assumption~\ref{ass:nonparametric-model} on $[0,1]^d$, uniformly sub-Gaussian control errors, and the stated sampling conditions, the weighted spline construction satisfies the bounds with
\begin{equation}\label{eq:ws_rates}
r_{b,T}=r_{F,T}=\bigg(\frac{\log T}{T^{1-\chi}}\bigg)^{s/(2s+d)}\text{ for }\chi\in(0,1).    
\end{equation}

\begin{assumption}[Step sizes]
\label{ass:step-sizes}
The prespecified step-size sequence $(\eta_t)_{t\ge1}$ is positive and nonincreasing. There exist $\zeta\in(1/2,1)$ and constants $0<c_\eta\le C_\eta<\infty$ such that
\begin{equation}
c_\eta t^{-\zeta}\le\eta_t\le C_\eta t^{-\zeta}\quad\text{for all sufficiently large }t.
\label{eq:step-size-order}
\end{equation}
\end{assumption}

Assumption~\ref{ass:step-sizes} adopts a standard polynomial step-size schedule widely used in stochastic approximation \citep{robbins1951stochastic}. It satisfies the classical Robbins--Monro conditions: $\eta_t\to0$, $\sum_{t\ge1}\eta_t=\infty$, and $\sum_{t\ge1}\eta_t^2<\infty$, which balance continued adaptation with control of stochastic fluctuations. The following theorem establishes the main asymptotic FDR control guarantee for the proposed treatment-assignment procedure.

\begin{theorem}[Asymptotic individual treatment FDR control]
\label{thm:fdr-one-pass}
Suppose Assumptions~\ref{ass:local-density}--\ref{ass:step-sizes} hold and the nuisance estimators satisfy \eqref{eq:nuisance_rate} with $r_{T}=r_{b,T}+r_{F,T}$. Let $(\lambda_t)_{t\ge1}$ be generated by \eqref{eq:paits-cutoff-update}. Then the cutoffs converge to $\lambda^\star$ in mean square, i.e.,
\begin{equation*}
 \frac1T\sum_{t=1}^T\Ebb[(\lambda_t-\lambda^\star)^2]\longrightarrow0.
\end{equation*}
Moreover,
\begin{equation}
\max\bigg\{\bigg|\frac{R_T}{T}-H(\lambda^\star)\bigg|,|\FDP_T-\alpha|\bigg\}=O_{\Pbb}\bigg(r_{T}+T^{-\zeta/2}+\frac{\mathfrak H_\chi(T)}{T}\bigg),
\label{eq:one-pass-frequency-fdp-rates}
\end{equation}
where $R_T=\sum_{t=1}^T A_t$ is the count of all treatment assignments, including those made during exploration, and $\mathfrak H_\chi(T)=\sum_{t=1}^T t^{-\chi}$. Therefore, when $r_{T}\to0$, we have
$$
\FDR_T=\Ebb(\FDP_T)\longrightarrow\alpha.
$$
\end{theorem}

The convergence rate \eqref{eq:one-pass-frequency-fdp-rates} in Theorem~\ref{thm:fdr-one-pass} separates three error sources. The term $r_{T}$ captures estimation error arising from estimation of both the treated-outcome distribution and the control reference. The term $T^{-\zeta/2}$ reflects stochastic error in the adaptive cutoff update. Finally, $\mathfrak H_\chi(T)/T$ accounts for exploratory treatments, which enter the FDP despite being assigned independently of the plug-in recommendation. Exploration therefore improves nuisance estimation but contributes directly to selection error.

For weighted spline regression, using the convergence rates \eqref{eq:ws_rates} gives, for $0<\chi<1$,
$$
|\FDP_T-\alpha|=O_{\Pbb}\big\{(\log T)^{s/(2s+d)}T^{-(1-\chi)s/(2s+d)}+T^{-\zeta/2}+T^{-\chi}\big\}.
$$
The nuisance term decays at rate $T^{-(1-\chi)s/(2s+d)}$, which slows as
$\chi$ increases, whereas the exploration term decays at rate $T^{-\chi}$.
Equating the two exponents yields $\chi=s/(3s+d)$, which maximizes the polynomial rate of this bound. Thus, the exploration schedule $\varepsilon_t=\varepsilon_0t^{-s/(3s+d)}$ optimizes the polynomial order of this bound, provided that $\zeta/2\ge s/(3s+d)$. With this choice,
$$
|\FDP_T-\alpha|=O_{\Pbb}\big\{(\log T)^{s/(2s+d)}T^{-s/(3s+d)}\big\}.
$$

In addition, if the uniform nuisance estimation errors vanish in mean square, the following corollary establishes mean-square convergence of the adaptive cutoff.

\begin{corollary}[Mean-square convergence of the cutoff]
\label{cor:paits-cutoff-mean-square}
Under the assumptions of Theorem~\ref{thm:fdr-one-pass}, suppose in addition that $\Ebb[(\delta_{b,t}^{\infty})^2]+\Ebb[(\delta_{F,t}^{\infty})^2]\longrightarrow0$ as $t\to\infty$. Then $\lambda_t$ converges to $\lambda^\star$ in mean square: $\Ebb[(\lambda_t-\lambda^\star)^2]\longrightarrow0$.
\end{corollary}

Having established asymptotic FDR control, we next quantify the cost of learning and exploration through the gap-weighted assignment regret.

\begin{theorem}[Gap-weighted assignment regret]
\label{thm:one-pass-weighted-selection-regret}
Under the assumptions of Theorem~\ref{thm:fdr-one-pass}, let $\widehat{\Regret}_T$ and $\Regret_T$ denote the predictable and expected regrets defined in \eqref{eq:weighted-selection-regret}. There exists a constant $C<\infty$, independent of $T$, such that
\begin{align}
\widehat{\Regret}_T&\le C\sum_{t=1}^T\{(\delta_{b,t}^{\infty})^2+(\delta_{F,t}^{\infty})^2+|\lambda_t-\lambda^\star|^2+\varepsilon_t\}\quad\text{a.s.},
\label{eq:one-pass-weighted-selection-quadratic-transfer}\\
\Regret_T &\le CT\{r_{b,T}^2+r_{F,T}^2+\bar\eta_T+\bar\varepsilon_T+T^{-1}\}
\label{eq:one-pass-weighted-selection-quadratic-rate}
\end{align}
for all sufficiently large $T$, where $r_{b,T}$ and $r_{F,T}$ are defined in \eqref{eq:nuisance_rate}, $\bar\eta_T=\sum_{t=1}^T\eta_t/T$, and $\bar\varepsilon_T=\sum_{t=1}^T\varepsilon_t/T$. Moreover, $\widehat{\Regret}_T=O_{\Pbb}[T\{r_{b,T}^2+r_{F,T}^2+\bar\eta_T+\bar\varepsilon_T+T^{-1}\}]$. Consequently,
$$
\frac{\Regret_T}{T}\longrightarrow0,
\qquad
\frac{\widehat{\Regret}_T}{T}\xrightarrow{\Pbb}0.
$$
\end{theorem}

Theorem~\ref{thm:one-pass-weighted-selection-regret} shows that nuisance-estimation and cutoff errors contribute quadratically to regret, whereas exploration contributes linearly. For weighted spline regression, substituting the convergence rates \eqref{eq:ws_rates} gives, for $0<\chi<1$,
$$
\Regret_T=O\bigg\{T\bigg(\frac{\log T}{T^{1-\chi}}\bigg)^{2s/(2s+d)}+T^{1-\zeta}+T^{1-\chi}\bigg\}.
$$
Balancing the polynomial orders of the nuisance-estimation
and exploration contributions yields $\chi=(2s)/(4s+d)$. The cutoff contribution is then of smaller order, and 
$$\Regret_T=O\{(\log T)^{2s/(2s+d)}T^{(2s+d)/(4s+d)}\}.$$
The same rate holds for $\widehat{\Regret}_T$ in probability. This exploration choice optimizes the polynomial order of the regret bound and differs from $\chi=s/(3s+d)$, which optimizes the FDP bound. Section~\ref{subsec:minimax} compares this polynomial order with a minimax lower bound under the same uniform learning requirement.

\subsection{Minimax regret under uniform nuisance learning}
\label{subsec:minimax}

We next characterize the tradeoff between assignment regret and nuisance-estimation accuracy. We establish a lower bound on the regret incurred by procedures that attain a prescribed uniform nuisance-estimation accuracy while controlling asymptotic FDR. Under the uniform regularity conditions stated below, PAITS achieves the corresponding minimax rate up to logarithmic factors.

Let $\mathcal A$ denote a nonanticipating assignment procedure defined for all $t\ge1$, together with its nuisance estimators. At time $t$, the estimators are based on $\mathcal F_{t-1}$, and the assignment uses only $\mathcal F_{t-1}$, $\bX_t$, and independent randomization. The resulting feedback consists solely of $Y_t(A_t)$. For a model class $\mathcal P$, deterministic sequences $r_{b,T},r_{F,T}>0$ tending to zero, and a fixed constant $C_0>0$, let $\mathfrak A_T(\mathcal P;r_{b,T},r_{F,T},C_0)$ denote the class of such procedures satisfying asymptotic FDR control,
\begin{equation}
\limsup_{T\to\infty}\FDR_T(\mathcal A,P)\le\alpha
\qquad\text{for every }P\in\mathcal P,
\label{eq:minimax-admissible-fdr}
\end{equation}
and the uniform time-averaged mean-square error bounds
\begin{equation}
\begin{aligned}
\sup_{P\in\mathcal P}\frac1T\sum_{t=1}^T\Ebb_{P,\mathcal A}\big[(\delta_{b,t}^{\infty})^2\big]&\le C_0r_{b,T}^2,\\
\sup_{P\in\mathcal P}\frac1T\sum_{t=1}^T\Ebb_{P,\mathcal A}\big[(\delta_{F,t}^{\infty})^2\big]&\le C_0r_{F,T}^2.
\end{aligned}
\label{eq:minimax-uniform-nuisance-constraint}
\end{equation}
Here $\delta_{b,t}^{\infty}$ and $\delta_{F,t}^{\infty}$ are the uniform estimation errors defined after \eqref{eq:nuisance_rate}, with population quantities evaluated under $P$. The same procedure must satisfy these requirements throughout $\mathcal P$ and cannot depend on the unknown distribution.

Writing $\Regret_T(\mathcal A,P)$ for the expected gap-weighted
assignment regret in \eqref{eq:weighted-selection-regret}, define
the constrained minimax regret as
\begin{equation}
\mathfrak R_T(\mathcal P;r_{b,T},r_{F,T})
:=\inf_{\mathcal A\in\mathfrak A_T(\mathcal P;r_{b,T},r_{F,T},C_0)}
  \sup_{P\in\mathcal P}\Regret_T(\mathcal A,P).
\label{eq:minimax-regret-definition}
\end{equation}
We suppress the fixed constant $C_0$ in the notation. This criterion characterizes the tradeoff between assignment regret and nuisance learning: it is the smallest worst-case regret attainable at the prescribed estimation accuracy while maintaining asymptotic FDR control.

Let $\mathcal P^{\mathrm{np}}(s)$ denote the collection of all distributions satisfying Assumptions~\ref{ass:local-density}, \ref{ass:oracle-score}, and \ref{ass:nonparametric-model} on $\mathcal X=[0,1]^d$, with the same regularity and oracle-stability constants. Throughout this class, take $b_0(\bx)=\Ebb_P\{Y(0)\mid\bX=\bx\}$ and assume that the control errors are uniformly sub-Gaussian: for a fixed $1\le\sigma^2<\infty$,
\begin{equation}
\Ebb_P[\exp\{u(Y(0)-b_0(\bX))\}\mid\bX]
\le \exp(\sigma^2u^2/2)\qquad\text{a.s.},
\quad u\in\mathbb R,\quad P\in\mathcal P^{\mathrm{np}}(s).
\label{eq:minimax-control-subgaussian}
\end{equation}
The historical sample follows the randomized design in Section~\ref{subsec:causal-setting}, with expected control and treatment sample sizes $n_{0,\mathrm{hist}}=(1-\pi)n_{\mathrm{hist}}$ and $n_{1,\mathrm{hist}}=\pi n_{\mathrm{hist}}$, respectively.

\begin{theorem}[Regret lower bound under prescribed learning accuracy]
\label{thm:minimax-nonparametric-fixed}
For sufficiently small $r_{b,T}$ and $r_{F,T}$, there exist constants $c,C_{\mathrm{hist}}>0$, depending only on $C_0$ and the fixed model-class constants, such that
\begin{equation}
\begin{aligned}
\mathfrak R_T(\mathcal P^{\mathrm{np}}(s);r_{b,T},r_{F,T})\ge
c\min\{T,\{r_{b,T}^{-(2s+d)/s}-C_{\mathrm{hist}}n_{0,\mathrm{hist}}\big\}_++\{r_{F,T}^{-(2s+d)/s}-C_{\mathrm{hist}}n_{1,\mathrm{hist}}\big\}_+\},
\end{aligned}
\label{eq:minimax-nonparametric-fixed}
\end{equation}
where $z_+=\max\{z,0\}$. 
\end{theorem}

The theorem characterizes the tradeoff between assignment regret and nuisance estimation accuracy under partial feedback. In the worst case, learning the control reference requires control observations in regions where the oracle assigns treatment, whereas learning the treated-outcome distribution requires treatment observations in regions where the oracle assigns control. These informative assignments incur a nonvanishing regret per observation, so greater estimation accuracy entails a larger assignment cost. The two terms in \eqref{eq:minimax-nonparametric-fixed} separate the costs of learning the two nuisance functions, with the historical-sample corrections accounting for information available before online assignment begins. Because the lower bound also holds without the FDR restriction, it identifies a cost intrinsic to nuisance learning under partial feedback.

To compare this lower bound with the preceding PAITS regret guarantee, we use the uniform nuisance-estimation rate attained by the nonparametric estimators in Appendix~\ref{app:weighted-nuisance-proofs}. The following corollary establishes a matching polynomial regret rate, provided that the PAITS guarantees hold uniformly over the same model class.

\begin{corollary}[Minimax optimality under prescribed learning accuracy]
\label{cor:minimax-nonparametric-rate}
Suppose $n_{\mathrm{hist}}=O(1)$. Consider $\mathcal A_{\mathrm{PAITS}}$ using nonparametric nuisance estimators satisfying the uniform accuracy condition with $\varepsilon_t\asymp t^{-2s/(4s+d)}$, $\eta_t\asymp t^{-\zeta}$ and $\zeta\in(1/2,1)$. Define
\begin{equation}
 \rho_T=(\log T)^{s/(2s+d)}T^{-s/(4s+d)},\qquad T\ge2.
 \label{eq:minimax-common-radius}
\end{equation}
Suppose the conditions of Theorems~\ref{thm:fdr-one-pass} and \ref{thm:one-pass-weighted-selection-regret} hold with the same constants over $\mathcal P^{\mathrm{np}}(s)$, and
\begin{equation}
\sup_{P\in\mathcal P^{\mathrm{np}}(s)}\frac1T\sum_{t=1}^T
\Ebb_{P,\mathcal A_{\mathrm{PAITS}}}
[(\delta_{b,t}^{\infty})^2+(\delta_{F,t}^{\infty})^2]
\le C_{\mathrm{nuis}}\rho_T^2
\label{eq:minimax-paits-uniform-nuisance}
\end{equation}
for all $T\ge T_0$, where $C_{\mathrm{nuis}}>0$ and $T_0$ are independent of $P$ and $T$. For any fixed $C_0\ge C_{\mathrm{nuis}}$, PAITS attains the minimax regret rate $T^{(2s+d)/(4s+d)}$ up to logarithmic factors over $\mathcal P^{\mathrm{np}}(s)$, among admissible procedures with prescribed estimation accuracy $r_{b,T}=r_{F,T}=\rho_T$.
\end{corollary}

Accurate assignment and accurate estimation are distinct objectives. A procedure may achieve small assignment regret without reliably estimating the control reference or the treated-outcome probabilities that inform treatment selection. This distinction is particularly important under partial feedback, where favorable assignment performance may coexist with limited information about one treatment condition in parts of the covariate space. We therefore treat nuisance learning as an explicit objective alongside treatment assignment. The prescribed estimation requirement ensures that procedures are compared at a the same level of statistical accuracy, rather than solely by their agreement with the oracle rule. The resulting minimax criterion quantifies the assignment cost of acquiring sufficiently accurate information about both treatment conditions. Thus, the learning constraint is essential to the joint estimation-and-assignment objective considered here, although it is not required for asymptotic FDR control alone.

\section{Simulation Studies}
\label{sec:simulation-settings}

We first conduct a simulation study to evaluate the finite-sample performance of the proposed method in terms of FDR control and assignment regret. We generate independent units with $\bX\sim\operatorname{Unif}([0,1]^d)$ and $Y(a)=f_a(\bX)+\epsilon_a$ for $a\in\{0,1\}$, where $\epsilon_0,\epsilon_1\stackrel{\mathrm{iid}}{\sim}N(0,1)$ are independent of the covariates and assignment randomization. Historical and online samples share the same distribution. Set $b_0(\bx)=\E(Y(0)\mid\bx)=f_0(\bx)$, $f_1(\bx)=f_0(\bx)+\Delta(\bx)$, and hence $\ell(\bx)=\Phi\{-\Delta(\bx)\}$, where $\Phi(\cdot)$ is the standard normal distribution function and $\Delta(\cdot)$ is the signal strength. We consider the following three models:
\begin{enumerate}
    \item[M1] (linear): $d=1$, $f_0(x)=x$, and $\Delta(x)=6x-3$.

    \item[M2] (nonlinear): $d=1$, $f_0(x)=\sin(2\pi x)+0.2\cos(4\pi x)$, $\Delta(x)=6x-3+0.6\sin(2\pi x)$.

    \item[M3] (nonlinear with interaction): $d=2$, $f_0(\bx)=0.4\sin(2\pi x_1)+0.4\cos(2\pi x_2)+1.2(x_1-0.5)(x_2-0.5)$, $\Delta(\bx)=4x_1-2+0.5\sin(2\pi x_2)$.
\end{enumerate}

For the comparisons in Figures~\ref{fig:simulation-fdr} and \ref{fig:simulation-regret}, we implement PAITS with four weighted regression methods for estimating the control reference $b_0$: local-linear regression (LL), B-spline regression with a growing basis (B-spline), random forest regression (RF), and nearest-neighbor regression (kNN). For the control-reference estimator $\widehat b_{0,t}$, LL combines kernel and observation weights in local linear fits; B-spline uses weighted least squares over an expanding basis; RF averages weighted predictions across trees; and kNN averages outcomes within neighborhoods defined by a prescribed cumulative-weight threshold. Appendix~\ref{app:simulation-estimators} provides further details on these four estimators. In all cases, we estimate the treated-outcome conditional distribution $F_1$ directly via weighted B-spline regression, following Proposition~\ref{prop:spline-nuisance} in Appendix~\ref{app:weighted-nuisance-proofs}. Treatment selection uses the resulting score $\widehat\ell_t(\bx)=\widehat F_{1,t}\{\widehat b_{0,t}(\bx)\mid\bx\}$. Unless otherwise specified, we set $\alpha=0.1$ and historical sample size $n_{\mathrm{hist}}=2000$, and use $T=50{,}000$ online arrivals and $100$ independent replications per configuration. Historical assignments and assignments during exploration each have treatment probability $0.5$.

\begin{figure}[t]
    \centering
    \includegraphics[width=\textwidth]{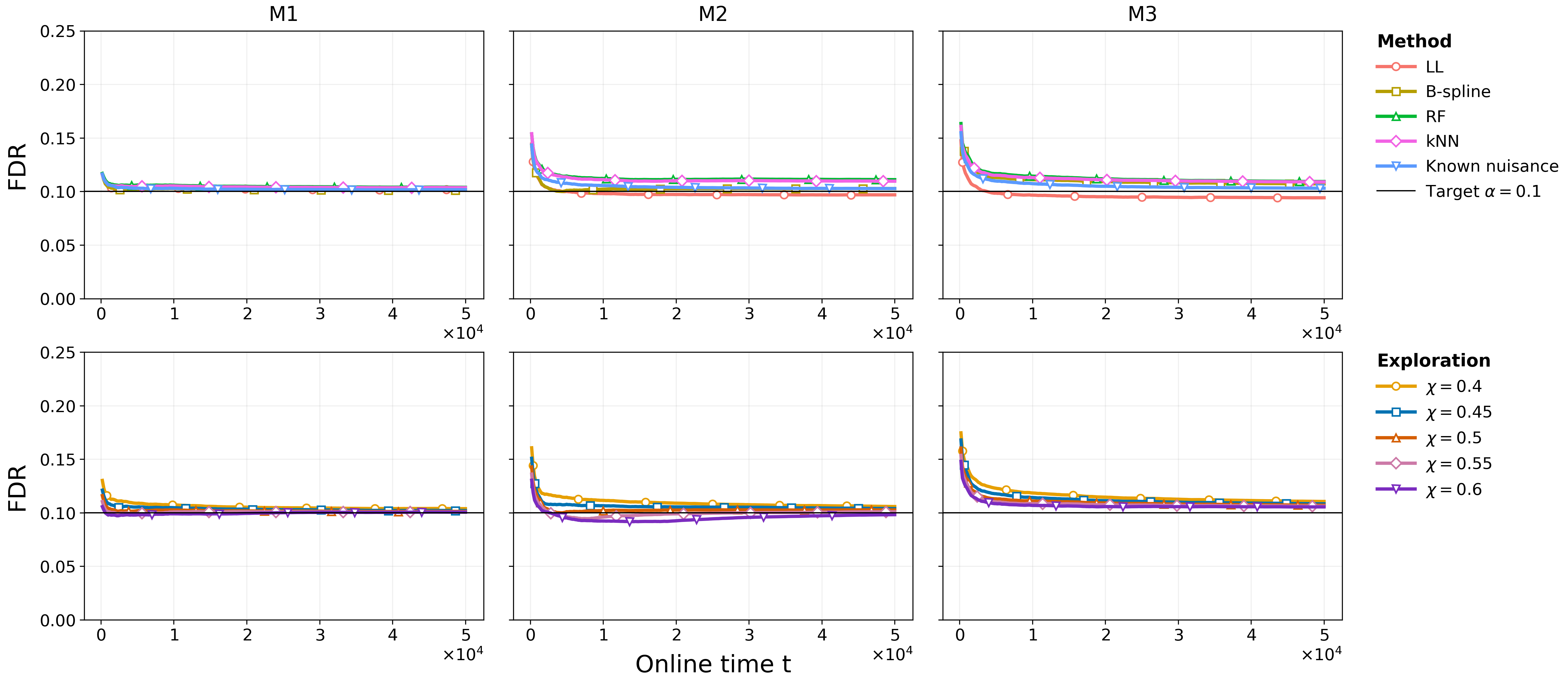}
    \caption{Empirical FDR over online time for Models M1--M3. The top row compares four control-reference estimators at $\chi=0.5$, and includes the known-nuisance setting. The bottom row varies $\chi$, using weighted B-spline estimation for both nuisance functions. The horizontal line denotes the nominal FDR level $\alpha=0.1$.}
    \label{fig:simulation-fdr}
\end{figure}

We first examine the FDR performance of PAITS. Figure~\ref{fig:simulation-fdr} plots the empirical FDR over online time for the three simulation models. The top row compares the four control-reference estimators at $\chi=0.5$, with weighted B-spline estimation of $F_1$ in each case. It also includes a known-nuisance setting in which both $b_0(\cdot)$ and $F_1(\cdot\mid\cdot)$ are set to their true data-generating values. The bottom row varies the exploration exponent $\chi\in\{0.4,0.45,0.5,0.55,0.6\}$, using weighted B-spline estimation for both nuisance functions. In M1, the empirical FDR approaches the nominal level rapidly for all four control-reference estimators. The differences are more pronounced in M2 and M3. In M2, LL and B-spline remain closest to the target, whereas RF and kNN exhibit larger upward deviations. In M3, LL settles below the nominal level, while B-spline, RF, and kNN remain moderately above it. A small upward deviation is also present in the known-nuisance setting, indicating that nuisance estimation does not fully account for the finite-sample discrepancy. This is consistent with the theoretical decomposition in Theorem~\ref{thm:fdr-one-pass}, where exploratory treatment assignments contribute directly to the FDP. In the bottom row, the FDR trajectories approach the nominal level across the values of $\chi$ considered, although M3 still shows upward deviations over the simulated horizon.

\begin{figure}[ht]
    \centering
    \includegraphics[width=\textwidth]{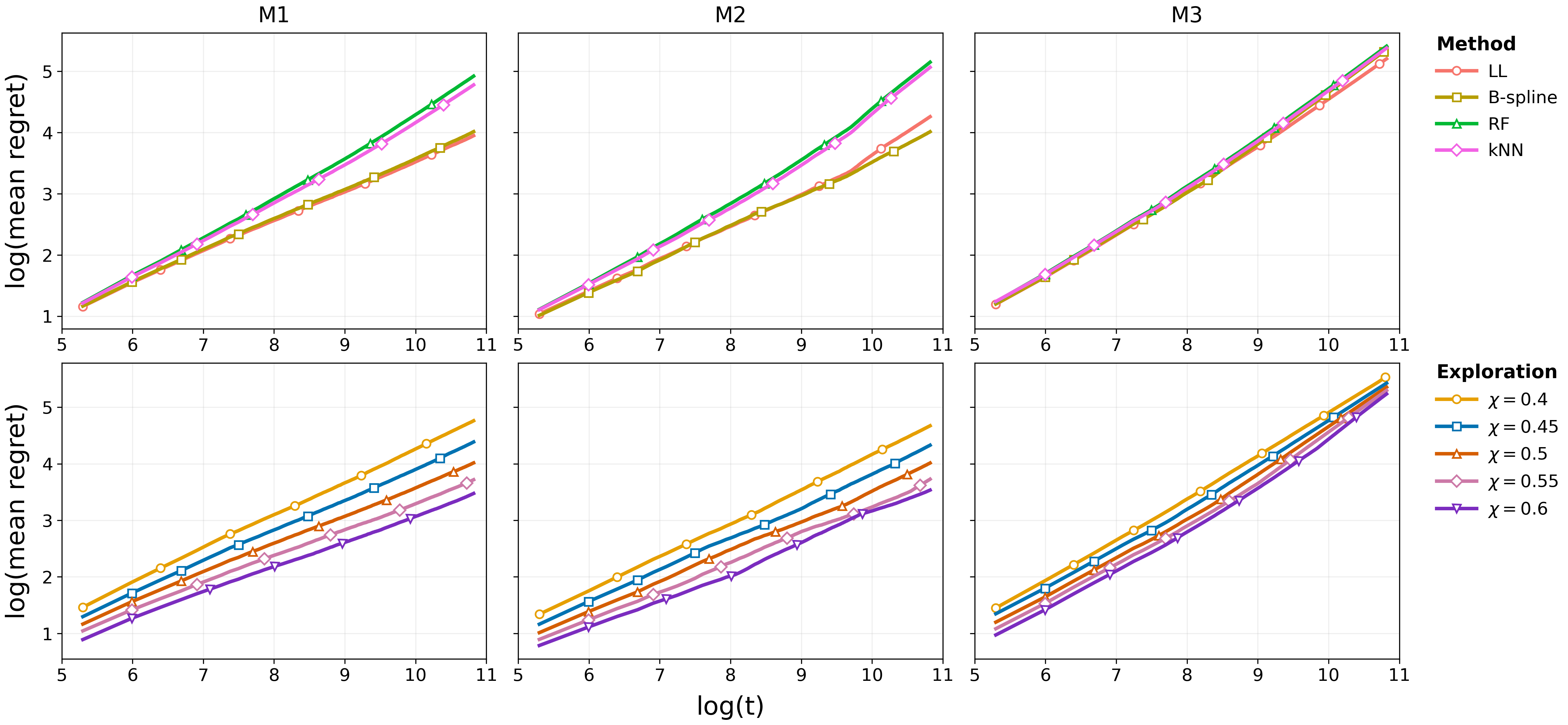}
    \caption{Log--log plots of the mean cumulative gap-weighted assignment regret for Models M1--M3. The top row compares four control-reference estimators at $\chi=0.5$. The bottom row varies $\chi$, using weighted B-spline estimation for both nuisance functions. }
    \label{fig:simulation-regret}
\end{figure}

We next examine assignment efficiency through the cumulative gap-weighted regret. Figure~\ref{fig:simulation-regret} shows the empirical trajectories of mean cumulative regret on a log--log scale. The approximately linear trajectories on the log--log scale are consistent with the polynomial dependence on $T$ in our theoretical regret bounds. Supplementary log--log linear fits and the corresponding slopes are reported in Appendix~\ref{app:simulation-loglog-fits}. In M1 and M2, LL and B-spline yield lower cumulative regret than RF and kNN, with the differences becoming more pronounced as online observations accumulate. In M3, the four methods exhibit similar regret trajectories, although LL yields somewhat lower regret toward the end of the simulated horizon. The bottom row of Figure~\ref{fig:simulation-regret} shows that the exploration effect depends on the model. Larger values of $\chi$ yield lower cumulative regret in all three models over the plotted range. The separation between the exploration settings is more pronounced in M1 and M2 than in M3 toward the end of the horizon. Together with Figure~\ref{fig:simulation-fdr}, these results illustrate the finite-sample tradeoff between FDR performance and assignment efficiency.

\begin{figure}[ht]
    \centering
    \includegraphics[width=\textwidth]{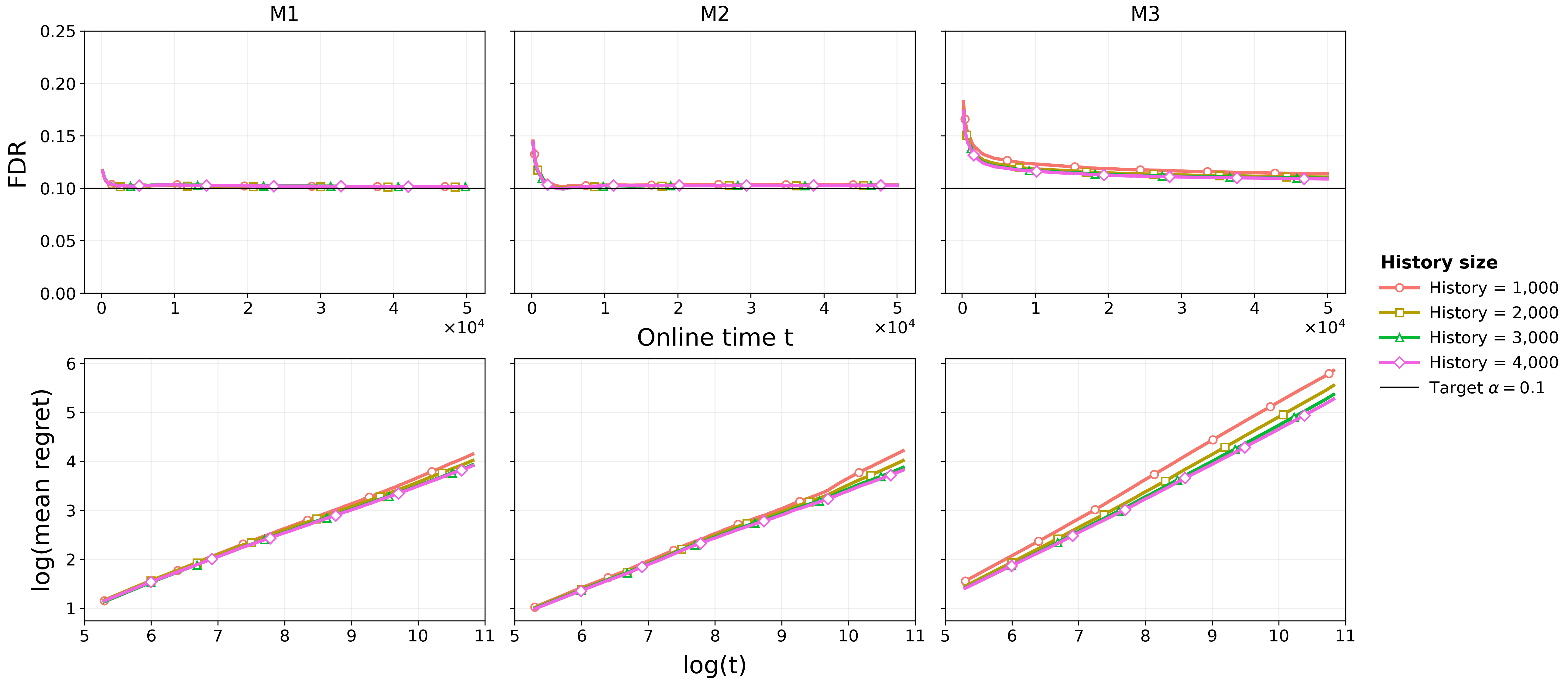}
    \caption{Sensitivity to the historical sample size for Models M1--M3. The historical sample sizes are $n_{\mathrm{hist}}\in\{1000,2000,3000,4000\}$. The top row reports the empirical FDR over online time, and the bottom row reports the log--log plots of the mean cumulative gap-weighted assignment regret.}
    \label{fig:history-sensitivity}
\end{figure}

We further examine the sensitivity of PAITS to the amount of historical information available at initialization and to the step-size schedule used in the adaptive cutoff update. Both analyses use weighted B-spline estimators for $b_0$ and $F_1$, with $\chi=0.5$ in M1 and M2 and $\chi=0.4$ in M3.

Figure~\ref{fig:history-sensitivity} examines the effect of the historical sample used to initialize PAITS. In M1 and M2, the empirical FDR is close to the nominal level across all four historical sample sizes, indicating relatively little sensitivity to initialization in these simpler settings. The effect is more pronounced in M3, where increasing $n_{\mathrm{hist}}$ reduces the finite-sample FDR deviation. This is consistent with the role of the historical sample in providing more accurate initial nuisance estimates before substantial online feedback has accumulated. Larger historical samples yield lower cumulative regret in all three models, with a more pronounced reduction in M3.

\begin{figure}[t]
    \centering
    \includegraphics[width=\textwidth]{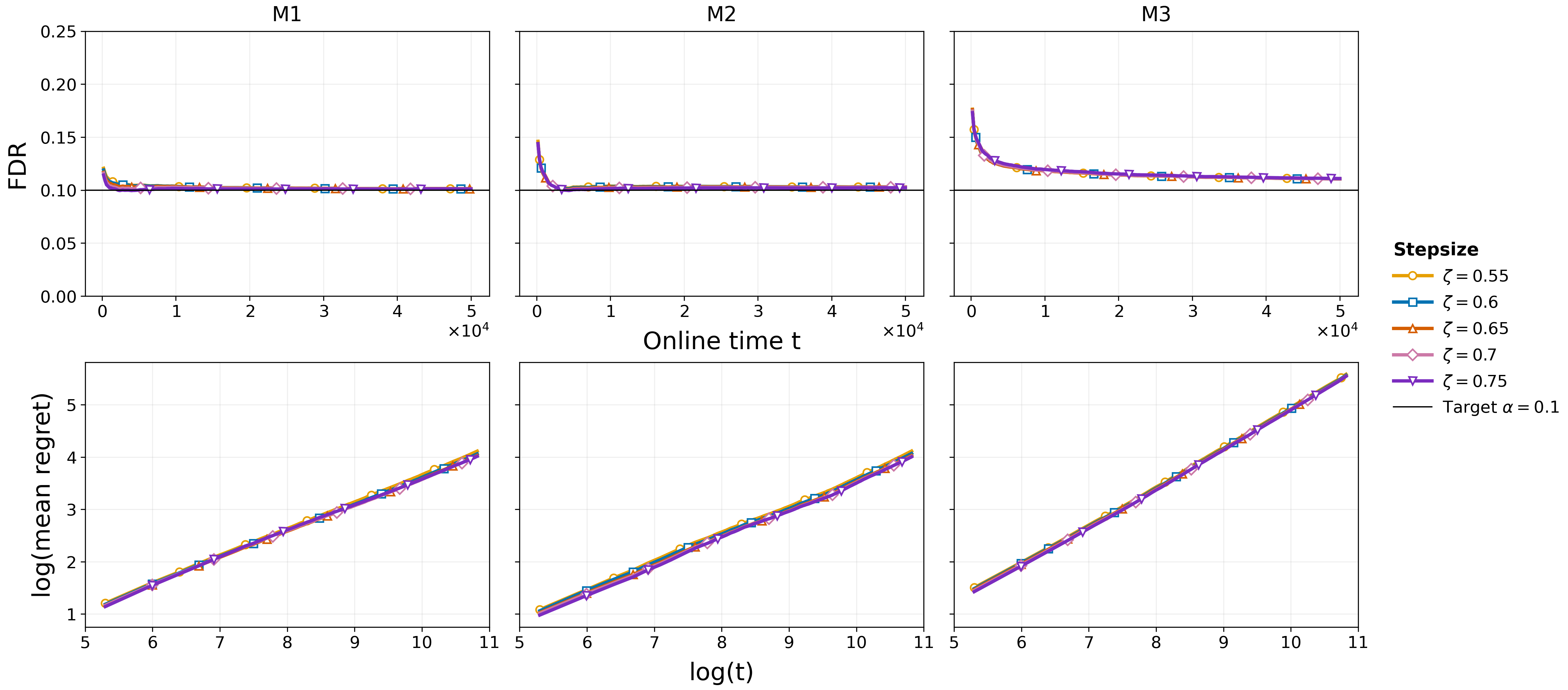}
    \caption{Sensitivity to the step-size exponent $\zeta$ for Models M1--M3. The adaptive cutoff uses a polynomial step-size schedule $\eta_t\asymp t^{-\zeta}$. The top row reports the empirical FDR over online time, and the bottom row reports the log--log plots of the mean cumulative gap-weighted assignment regret.}
    \label{fig:stepsize-sensitivity}
\end{figure}

Figure~\ref{fig:stepsize-sensitivity} investigates sensitivity to the step-size exponent $\zeta$ in Assumption~\ref{ass:step-sizes}. For the five values $\zeta\in\{0.55,0.60,0.65,0.70,0.75\}$, the FDR trajectories are very similar within each model. Thus, the FDR performance is relatively insensitive to the choice of $\zeta$ over the range considered. The regret results also exhibit limited sensitivity to the step-size exponent, with modest variation across the values of $\zeta$ considered. This suggests that PAITS is relatively robust to the choice of $\zeta$ within the admissible range.

\section{Real-data analysis: Jester recommendation}
\label{subsec:jester-real-data}

We further evaluate PAITS using the Jester recommendation dataset \citep{goldberg2001eigentaste}, which contains anonymous joke ratings from 73,421 users. We retain the 70,434 users with all four ratings available and treat this subset as the finite population for evaluation. Let $R_{ij}\in[-10,10]$ denote the rating given by user $i$ to joke $j$. We consider a recommendation problem that chooses between jokes 17 and 50, using the user's ratings for jokes 5 and 13 as pre-decision contextual information. Specifically, we define
\begin{equation*}
    \bX_i=\bigg(\frac{R_{i5}+10}{20},\frac{R_{i13}+10}{20}\bigg)^\top,
    \qquad
    Y_i(0)=R_{i17},
    \qquad
    Y_i(1)=R_{i50}.
\end{equation*}
The ratings are linearly rescaled from $[-10,10]$ to $[0,1]$ for numerical convenience. Then, $A_i=0$ corresponds to recommending joke 17 and $A_i=1$ to recommending joke 50. Hence, a recommendation of joke 50 is counted as a false discovery when its realized rating fails to exceed the conditional mean rating for joke 17 at the same covariate value.

Although both potential outcomes are available in the original data, we retain only the outcome under the assigned treatment to reproduce the partial-feedback setting. We initialize PAITS with $10,000$ randomly assigned historical observations and then run the procedure for $T=10{,}000$ online arrivals. The historical sample is fixed across repeated runs, while the online sample, arrival order, and exploration randomization are varied. For PAITS, we estimate the control reference $b_0(\bx)$ by weighted B-spline regression, and the outcome distribution $F_1(\cdot\mid\bx)$ by a weighted local empirical distribution function. For offline evaluation, we use the complete finite population to estimate the corresponding population quantities and thereby obtain the oracle policy defined in Lemma~\ref{lem:oracle-assignment-optimality}.  Sections~\ref{app:jester-online-nuisance}--\ref{app:jester-evaluation}
in Appendix~\ref{app:jester-details} provide further implementation and evaluation details. We set the target FDR level to $\alpha=0.10$ and perform five replay replications, each using a randomized ordering of the evaluation population. 

\begin{figure}[ht]
    \centering
    \includegraphics[width=0.98\textwidth]{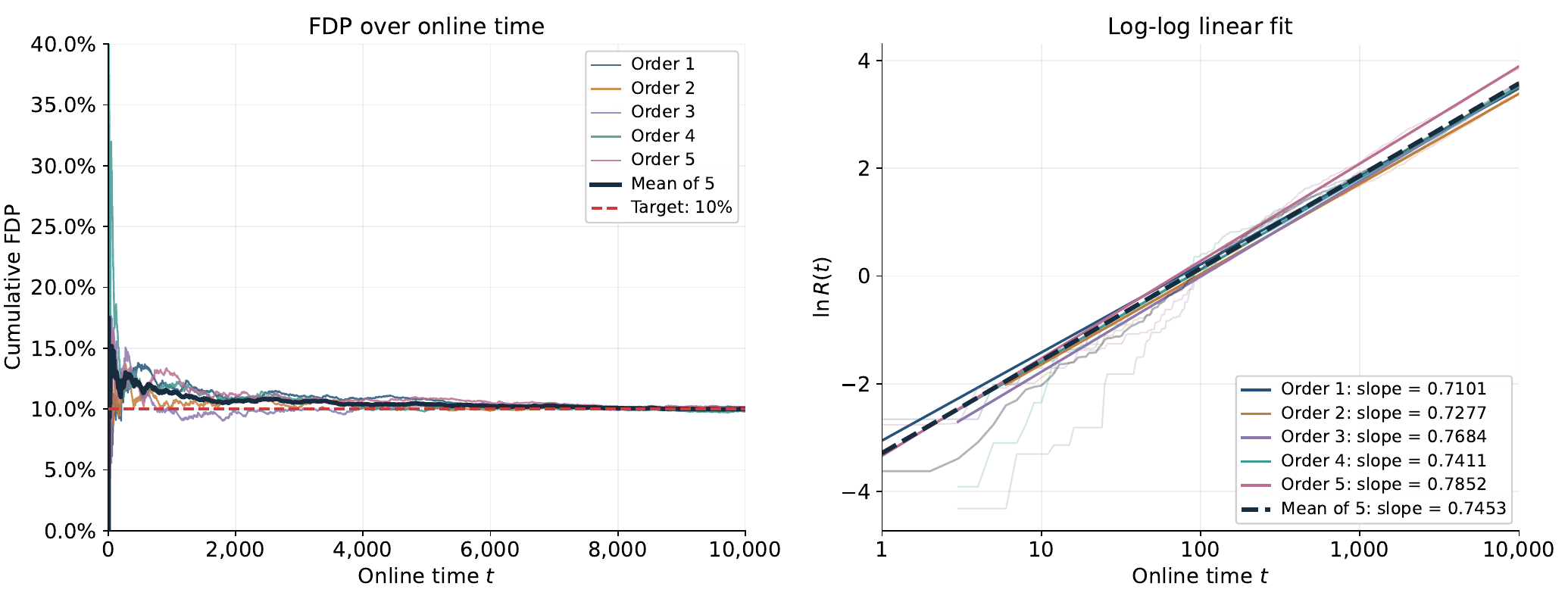}
    \caption{
    Performance of PAITS on the Jester data over five randomized replay orders. Left: cumulative false discovery proportion over online time, where the thick solid line represents the pointwise average over the five orders and the dashed line denotes the target level $\alpha=0.10$. Right: log--log trajectories of cumulative gap-weighted assignment regret relative to the finite-population oracle, together with the fitted log--log slopes.
    }
    \label{fig:jester-results}
\end{figure}

Figure~\ref{fig:jester-results} summarizes the resulting FDP and regret performance. The left panel shows the cumulative FDP over online time. The FDP exhibits noticeable fluctuations early in the sequence, when relatively few treatment assignments have been made, but becomes increasingly stable as feedback accumulates. By the end of the online horizon, all five trajectories, along with their pointwise average, are close to the target level of $10\%$.

The right panel of Figure~\ref{fig:jester-results} shows the log--log trajectories of the cumulative gap-weighted assignment regret relative to the finite-population oracle. The fitted slopes range from $0.710$ to $0.785$ across the five replay orders, with a slope of $0.745$ for the average regret trajectory. The slopes are consistently below one, indicating sublinear regret growth over the observed online horizon. These results demonstrate favorable performance of PAITS in the Jester replay evaluation and suggest its potential for online recommendation systems and related sequential decision-making applications.

\section{Conclusion}
\label{sec:con}
We develop PAITS, an online procedure for individualized treatment selection with asymptotic FDR control under partial outcome feedback. The procedure integrates adaptive nuisance estimation, randomized exploration, and feedback-driven threshold adaptation. Our theoretical results establish asymptotic FDR control over all treated individuals. Under the prescribed uniform nuisance-estimation requirements, PAITS attains the constrained minimax regret rate up to logarithmic factors.

Several directions are worth further investigation. One is to develop finite-sample FDR control at the nominal level while accounting for uncertainty in the estimated control reference and conditional null risk. Another is to extend the framework to delayed outcomes and changing outcome distributions. A further direction is to characterize optimal assignment regret under weaker, decision-focused learning requirements. This could clarify the gains from concentrating exploration where information is most relevant to treatment selection, while providing insight into how nuisance learning, treatment assignment, and false discovery control jointly shape the design of adaptive procedures.

\bibliographystyle{chicago}
\bibliography{ref}

\clearpage

\appendix

\section{Proofs of the Theoretical Results}

\subsection{Oracle Threshold}

We first consider the oracle threshold.

\begin{proof}[Proof of Lemma \ref{lem:oracle-assignment-optimality}]

By the definition of $\lambda^\star$, we have
$$
    \Ebb[\{\ell(\bX)-\alpha\}A^\star]=0,
$$
which indicates that the oracle assignment rule satisfies the marginal FDR constraint
$$\Ebb \bigg[\sum_{t=1}^TA_t^*\ind\{Y_t(1)\le b_0(X_t)\}\bigg]\le\alpha\Ebb \bigg[\sum_{t=1}^T A_t^* \bigg].$$

For any competing assignment $A$, the definition of $A^\star$ implies that $\{\ell(\bX)-\lambda^\star\}(A-A^\star)\geq0$ almost surely. We also have the following decomposition 
\begin{align*}
    \Ebb[\{\ell(\bX)-\alpha\}A]=\Ebb[\{\ell(\bX)-\lambda^\star\}(A-A^\star)]+(\lambda^\star-\alpha)\{\Ebb(A)-\Ebb(A^\star)\}.
\end{align*}

If $\Ebb(A)>\Ebb(A^\star)$, the right-hand side is strictly positive, contradicting the marginal FDR constraint $\Ebb[\{\ell(\bX)-\alpha\}A]\le0$. Hence no feasible assignment treats a larger proportion of the population than $A^\star$. If a feasible assignment satisfies $\Ebb(A)=\Ebb(A^\star)$, the first term in the preceding display must equal zero. Because its integrand is nonnegative and $\Pbb\{\ell(\bX)=\lambda^\star\}=0$, this implies $A=A^\star$ almost surely, establishing uniqueness. Finally, under the oracle assignment,
$$
   \operatorname{mFDR}_T(A^\star)= \Pbb\{Y(1)\leq b_0(\bX)|A^\star=1\}=\Ebb[\ell(\bX)|\ell(\bX)\leq\lambda^\star]=\alpha.
$$
Therefore, using the fixed-rule identity above,
$$
\FDR_T=\operatorname{mFDR}_T(A^\star)\Pbb(R_T>0)=\alpha\Pbb(R_T>0)\leq\alpha.
$$

Moreover, since $\lambda^\star\le1$, for any feasible assignment rule $A$,
\begin{align*}
\Ebb[\{1-\ell(\bX)\}(A^\star-A)]=(1-\lambda^\star)\{\Ebb(A^\star)-\Ebb(A)\}+\Ebb[\{\ell(\bX)-\lambda^\star\}(A-A^\star)]\ge0.
\end{align*}
Thus, $A^\star$ also maximizes the probability of a successful treatment assignment, $\Ebb[\{1-\ell(\bX)\}A]$, subject to the marginal FDR constraint.
\end{proof}

\begin{proof}[Proof of Proposition~\ref{prop:oracle-threshold}]
Note that we have $H(\lambda)\ge H(\alpha)>0$ for $\lambda\in[\alpha,1]$, and that
\begin{equation}\label{eq:def_qlamb}
    q(\lambda)=\frac{\int_0^\lambda u h(u)\,du}{H(\lambda)},
\end{equation}
is absolutely continuous on $[\alpha,1]$, with derivative
$$
q'(\lambda)=\frac{h(\lambda)}{H(\lambda)}\{\lambda-q(\lambda)\}=\frac{h(\lambda)}{H(\lambda)^2}\int_0^\lambda(\lambda-u)h(u)du
$$
almost everywhere. Thus $0\le q'(\lambda)\le \|h\|_\infty/H(\alpha)$ almost everywhere, so $q$ is nondecreasing and Lipschitz continuous. Absolute continuity of the score distribution and $H(\alpha)>0$ imply $q(\alpha)<\alpha$. Since $h(\cdot)$ is continuous and strictly positive on $[\alpha,\lambda_+]$, the derivative formula gives
$$
q'(\lambda)\ge\{\min_{u\in[\alpha,\lambda_+]}h(u)\}H(\alpha)\{\alpha-q(\alpha)\}=:c_{\mathrm{loc}}>0,\quad\text{for}\quad \lambda\in(\alpha,\lambda_+).
$$
Here we used $H(\lambda)\le1$ and restricted the integral to $[0,\alpha]$. Consequently, $q(\cdot)$ is strictly increasing on $[\alpha,\lambda_+]$. Together with $q(\alpha)<\alpha<q(\lambda_+)$, this establishes a unique $\lambda^\star\in(\alpha,\lambda_+)$ satisfying $q(\lambda^\star)=\alpha$. Monotonicity excludes additional roots above $\lambda_+$.

For $\lambda\in[\alpha,\lambda_+]$, the derivative bound yields
$$
(\lambda-\lambda^\star)\{q(\lambda)-\alpha\}\ge c_{\mathrm{loc}}(\lambda-\lambda^\star)^2.
$$
For $\lambda\in[\lambda_+,1]$, monotonicity instead gives
$$
q(\lambda)-\alpha\ge q(\lambda_+)-\alpha\ge\frac{q(\lambda_+)-\alpha}{1-\lambda^\star}(\lambda-\lambda^\star).
$$
Combining these inequalities with $H(\lambda)\ge H(\alpha)$ establishes
$$
(\lambda-\lambda^\star)H(\lambda)\{\alpha-q(\lambda)\}\le -c_g(\lambda-\lambda^\star)^2,
\quad\text{for}\quad\lambda\in[\alpha,1],
$$
where
$$
c_g=H(\alpha)\min\bigg\{c_{\mathrm{loc}},\frac{q(\lambda_+)-\alpha}{1-\lambda^\star}\bigg\}>0.
$$
\end{proof}

\subsection{Proofs of Theorem~\ref{thm:fdr-one-pass} and Corollary~\ref{cor:paits-cutoff-mean-square}}
\label{app:individual-fdr-proof}

Throughout this subsection, all assignments, estimators, and cutoffs are those defined in Section~\ref{sec:method}. Write $r_{T}=r_{b,T}+r_{F,T}$ and let $\overline h=\|h\|_\infty$. Theorem~\ref{thm:fdr-one-pass} uses the time-averaged nuisance bounds in \eqref{eq:nuisance_rate}, with $r_{b,T}+r_{F,T}\to0$.

The following lemma shows that randomized exploration supplies infinitely many observations from each treatment condition, while the proportion of exploratory assignments vanishes. It also quantifies the exploration contribution to the treatment-frequency and FDP bounds. Throughout, $C$ denotes a generic positive constant whose value may vary from line to line.

\begin{lemma}[Exploration counts]
\label{lem:epsilon-greedy-exposure}
Let $p_1=\pi_{\mathrm{exp}}$, $p_0=1-\pi_{\mathrm{exp}}$, and $s_T=\sum_{t=1}^T\varepsilon_t$. For $N_{\mathrm{exp},a}(T)=\sum_{t=1}^T B_t\ind\{Z_t=a\}$,
\begin{equation*}
 \frac{N_{\mathrm{exp},a}(T)}{p_as_T}\longrightarrow1
 \quad\text{a.s.},\qquad a\in\{0,1\}.
\end{equation*}
In particular, with $\bar\varepsilon_T=s_T/T$, the following equation holds
\begin{equation}
 \bar\varepsilon_T
 =O\{\mathfrak H_\chi(T)/T\}
 =\begin{cases}
 O(T^{-\chi}),&0<\chi<1,\\
 O\{\log(T\vee2)/T\},&\chi=1.
 \end{cases}
 \label{eq:exploration-sample-size-rate}
\end{equation}
\end{lemma}

\begin{proof}
Fix $a\in\{0,1\}$. Since the exploration randomization is independent of $\mathcal F_{t-1}$, we have
\begin{align*}
 \Pbb(B_t=1,Z_t=a\mid\mathcal F_{t-1})=\Pbb(B_t=1\mid\mathcal F_{t-1})\Pbb(Z_t=a\mid B_t=1,\mathcal F_{t-1})=\varepsilon_t p_a.
\end{align*}
Thus $B_t\ind\{Z_t=a\}$ is conditionally Bernoulli with parameter $p_a\varepsilon_t$. Define $D_{a,t}=B_t\ind\{Z_t=a\}-p_a\varepsilon_t$. These centered increments
are $\mathcal F_t$-measurable and satisfy
\begin{equation*}
 \Ebb(D_{a,t}\mid\mathcal F_{t-1})=0,\qquad
 \Ebb(D_{a,t}^2\mid\mathcal F_{t-1})
 =p_a\varepsilon_t(1-p_a\varepsilon_t)\le p_a\varepsilon_t,
\end{equation*}
indicating that $(D_{a,t}:t\ge1)$ is a sequence of martingale difference with respect to $(\mathcal F_T)_{T\ge0}$. In particular, $N_{\mathrm{exp},a}(T)-p_as_T=\sum_{t=1}^T D_{a,t}$ forms a martingale with respect to $(\mathcal F_T)_{T\ge0}$.

We first verify the growth and summability properties of the normalizing sequence $s_T$. Since $\varepsilon_t=\varepsilon_0t^{-\chi}$, integral comparison gives
\begin{equation*}
 \varepsilon_0\int_1^{T+1}u^{-\chi}\,du
 \le s_T\le\varepsilon_0\bigg(1+\int_1^T u^{-\chi}du\bigg).
\end{equation*}
Consequently, we have $s_T\to\infty$ and
\begin{equation}\label{eq:asymp_s_T}
 s_T\asymp
 \begin{cases}
 T^{1-\chi},&0<\chi<1,\\
 \log(T\vee2),&\chi=1.
 \end{cases}
\end{equation}
Moreover, $s_t-s_{t-1}=\varepsilon_t$ and $s_t\ge s_{t-1}>0$ for $t\ge2$, so
\begin{align*}
 \sum_{t=2}^T\frac{\varepsilon_t}{s_t^2}\le\sum_{t=2}^T\frac{s_t-s_{t-1}}{s_{t-1}s_t}=\sum_{t=2}^T\bigg(\frac1{s_{t-1}}-\frac1{s_t}\bigg)
 =\frac1{s_1}-\frac1{s_T}\le\frac1{s_1}.
\end{align*}
Including the finite term at $t=1$ yields $\sum_{t=1}^{\infty}\varepsilon_t/s_t^2\le2/s_1<\infty$.

Now define the normalized martingale $M_{a,T}=\sum_{t=1}^T D_{a,t}/s_t$, with $M_{a,0}=0$. For $u<t$, the martingale-difference property gives $\Ebb(D_{a,u}D_{a,t})=\Ebb\{D_{a,u}\Ebb(D_{a,t}\mid\mathcal F_{t-1})\}=0$. Hence
\begin{equation*}
 \sup_{T\ge1}\Ebb(M_{a,T}^2)
 =\sup_{T\ge1}\sum_{t=1}^T\frac{\Ebb(D_{a,t}^2)}{s_t^2}
 \le p_a\sum_{t=1}^{\infty}\frac{\varepsilon_t}{s_t^2}
 <\infty.
\end{equation*}
The $L^2$ martingale convergence theorem therefore implies $M_{a,T}\to M_{a,\infty}$ almost surely for a finite random variable $M_{a,\infty}$. Since $M_{a,0}=0$, summation by parts gives
\begin{align*}
\frac1{s_T}\sum_{t=1}^T D_{a,t}
&=\frac1{s_T}\sum_{t=1}^T s_t(M_{a,t}-M_{a,t-1})=\frac1{s_T}\bigg\{s_TM_{a,T}-\sum_{t=1}^{T-1}(s_{t+1}-s_t)M_{a,t}\bigg\}\\
&=M_{a,T}-\frac1{s_T}\sum_{t=1}^{T-1}(s_{t+1}-s_t)M_{a,t}.
\end{align*}
Fix a sample path on which $M_{a,t}\to M_{a,\infty}$. For any $\delta>0$, there exists a finite integer $n$ such that
$$
|M_{a,t}-M_{a,\infty}|\le\delta,\qquad t\ge n.
$$
Since $s_t$ is nondecreasing and $\sum_{t=1}^{T-1}(s_{t+1}-s_t)=s_T-s_1$, for $T>n$,
\begin{align*}
&\bigg|\frac1{s_T}\sum_{t=1}^{T-1}(s_{t+1}-s_t)M_{a,t}-M_{a,\infty}\bigg|=\bigg|\frac1{s_T}\sum_{t=1}^{T-1}(s_{t+1}-s_t)(M_{a,t}-M_{a,\infty})-\frac{s_1}{s_T}M_{a,\infty}\bigg|\\
\le&\frac1{s_T}\sum_{t=1}^{n-1}(s_{t+1}-s_t)|M_{a,t}-M_{a,\infty}|+\delta\frac{s_T-s_n}{s_T}+\frac{s_1}{s_T}|M_{a,\infty}|.
\end{align*}
The first term tends to zero because its numerator is a fixed finite sum and $s_T\to\infty$. The last term also vanishes, whereas the middle term is at most $\delta$. Thus,
$$
\limsup_{T\to\infty}\bigg|\frac1{s_T}\sum_{t=1}^{T-1}(s_{t+1}-s_t)M_{a,t}-M_{a,\infty}\bigg|\le\delta.
$$
Since $\delta>0$ is arbitrary, the weighted sum converges to $M_{a,\infty}$. Combining this conclusion with $M_{a,T}\to M_{a,\infty}$ gives
$$
\frac1{s_T}\sum_{t=1}^T D_{a,t}\longrightarrow0
$$
on every sample path where the martingale converges. Because this event has probability one, the convergence holds almost surely. Thus
\begin{equation*}
 \frac{N_{\mathrm{exp},a}(T)-p_as_T}{s_T}
 =\frac1{s_T}\sum_{t=1}^T D_{a,t}\longrightarrow0
 \quad\text{a.s.}
\end{equation*}
Since $p_a>0$, this proves $N_{\mathrm{exp},a}(T)/(p_as_T)\to1$ almost surely for each treatment condition for $a\in\{0,1\}$. Adding the two count limits also gives
\begin{equation*}
 \frac{\sum_{t=1}^T B_t}{s_T}
 =\frac{N_{\mathrm{exp},0}(T)+N_{\mathrm{exp},1}(T)}{s_T}
 \longrightarrow p_0+p_1=1\quad\text{a.s.}
\end{equation*}
Finally, dividing the bounds for $s_T$ in \eqref{eq:asymp_s_T} by $T$ proves \eqref{eq:exploration-sample-size-rate}. In particular,
\begin{equation*}
 \frac1T\sum_{t=1}^T B_t
 =\frac{s_T}{T}\{1+o(1)\}\longrightarrow0\quad\text{a.s.}
\end{equation*}
\end{proof}

For later use, recall the mean-square bounds in \eqref{eq:nuisance_rate}:
\begin{equation}
 \frac1T\sum_{t=1}^T\Ebb[(\delta_{b,t}^{\infty})^2]
 =O(r_{b,T}^2),\qquad
 \frac1T\sum_{t=1}^T\Ebb[(\delta_{F,t}^{\infty})^2]
 =O(r_{F,T}^2).
 \label{eq:nuisance-average-bounds}
\end{equation}
When polynomial upper bounds are available, the corresponding notation is
\begin{equation}
 r_{b,T}=O(T^{-\rho_b}),\qquad r_{F,T}=O(T^{-\rho_F}),
 \qquad \rho_b,\rho_F>0.
 \label{eq:nuisance-radius-orders}
\end{equation}
The following proposition shows how nuisance estimation errors affect the plug-in local null risk and the conditional mean of the update direction $S_t(\alpha-\widehat E_t)$. The following proposition supplies the perturbation bounds needed to apply the oracle contraction in Proposition~\ref{prop:oracle-threshold}.

\begin{proposition}[Plug-in risk and predictable drift]
\label{prop:score-rate}
Under Assumptions~\ref{ass:local-density} and \ref{ass:oracle-score}, let
\begin{align*}
 \delta_{\ell,t}^{\infty}
 &=\operatorname*{ess\,sup}_{\bx\sim P_X}
       |\widehat\ell_t(\bx)-\ell(\bx)|,\\
 \Gamma_t
 &=\sup_{\lambda\in[\alpha,1]}
   \Pbb\left\{[\widehat\ell_t(\bX_t)-\lambda]
                  [\ell(\bX_t)-\lambda]\le0
                 \mid\mathcal F_{t-1}\right\}.
\end{align*}
Then, there are positive constants $C_U,\delta_U,C<\infty$ such that
\begin{equation}
 \delta_{\ell,t}^{\infty}
 \le\delta_{F,t}^{\infty}
       +(C_U+\delta_U^{-1})\delta_{b,t}^{\infty},
 \qquad
 \Gamma_t\le2\overline h\delta_{\ell,t}^{\infty}.
 \label{eq:score-uniform-bound}
\end{equation}
Moreover, the predictable drift of the cutoff satisfies
\begin{equation}
 \big|\Ebb[S_t(\alpha-\widehat E_t)\mid\mathcal F_{t-1}]
       -(1-\varepsilon_t)H(\lambda_t)\{\alpha-q(\lambda_t)\}\big|
 \le\Gamma_t+C\delta_{b,t}^{\infty}.
 \label{eq:paits-drift-error}
\end{equation}
Consequently, under \eqref{eq:nuisance_rate},
\begin{align}
 \frac1T\sum_{t=1}^T\Ebb[(\delta_{\ell,t}^{\infty})^2+\Gamma_t^2]
 &=O(r_{b,T}^2+r_{F,T}^2),
 \label{eq:score-average-second-moment-rate}\\
 \frac1T\sum_{t=1}^T\Gamma_t
 &=O_{\Pbb}(r_{T}).
 \label{eq:score-crossing-average-rate}
\end{align}
\end{proposition}

\begin{proof}

We start by decomposing the difference between $\widehat\ell_t(\bx)$ and $\ell(\bx)$:
\begin{align*}
 |\widehat\ell_t(\bx)-\ell(\bx)|
 &\le
 |\widehat F_{1,t}\{\widehat b_{0,t}(\bx)\mid\bx\}
      -F_1\{\widehat b_{0,t}(\bx)\mid\bx\}|+
 |F_1\{\widehat b_{0,t}(\bx)\mid\bx\}-F_1\{b_0(\bx)\mid\bx\}|\\
 &\le \delta_{F,t}^{\infty} + |F_1\{\widehat b_{0,t}(\bx)\mid\bx\}-F_1\{b_0(\bx)\mid\bx\}|.
\end{align*}

If $|\widehat b_{0,t}(\bx)-b_0(\bx)|\le\delta_U$, then 
$$|F_1\{\widehat b_{0,t}(\bx)\mid\bx\}-F_1\{b_0(\bx)\mid\bx\}|\le
       C_U|\widehat b_{0,t}(\bx)-b_0(\bx)|.$$

Otherwise, as both $F_1\{\widehat b_{0,t}(\bx)\mid\bx\}$ and $F_1\{b_0(\bx)\mid\bx\}$ lie in $[0,1]$, their difference is at most $1\le\delta_U^{-1}|\widehat b_{0,t}(\bx)-b_0(\bx)|$. Taking essential suprema proves the first inequality in \eqref{eq:score-uniform-bound}. 

On the event $\{\{\widehat\ell_t(\bX_t)-\lambda\}\{\ell(\bX_t)-\lambda\}\le0\}$, we must have 
$$|\ell(\bX_t)-\lambda| \le |\widehat\ell_t(\bX_t)-\ell(\bX_t)| \le \delta_{\ell,t}^{\infty}.$$
As $\bX_t$ is independent of $\mathcal F_{t-1}$, the conditional distribution of $\bX_t$ given $\mathcal F_{t-1}$ is $P_X$. Integrating the bounded density $h$ over an interval of length at most $2\delta_{\ell,t}^{\infty}$ proves the second inequality in \eqref{eq:score-uniform-bound}.

Define the oracle selection map
\begin{equation}
 S_t^0(\lambda)=(1-B_t)\ind\{\ell(\bX_t)\le\lambda\},
 \qquad \lambda\in[\alpha,1].
 \label{eq:one-pass-oracle-selection-map}
\end{equation}
For this comparison, keep the current cutoff $\lambda_t$ fixed and replace the plug-in selection and null indicator with $S_t^0(\lambda_t)$ and $E_t$. The resulting oracle increment before projection is $\eta_tS_t^0(\lambda_t)(\alpha-E_t)$. Since $\lambda_t$ is $\mathcal F_{t-1}$-measurable and $\Ebb(E_t\mid\mathcal G_t)=\ell(\bX_t)$, the conditional mean of its update direction is
\begin{equation}    \label{eq:E_update}
\begin{aligned}
 \Ebb[S_t^0(\lambda_t)(\alpha-E_t)\mid\mathcal F_{t-1}]
 &=(1-\varepsilon_t)
   \Ebb[\ind\{\ell(\bX_t)\le\lambda_t\}
                  \{\alpha-\ell(\bX_t)\}\mid\mathcal F_{t-1}]\\
 &=(1-\varepsilon_t)\int_0^{\lambda_t}(\alpha-u)h(u)\,du\\
 &=(1-\varepsilon_t)H(\lambda_t)\{\alpha-q(\lambda_t)\}.
\end{aligned}
\end{equation}
Here $(1-\varepsilon_t)H(\lambda_t)$ is the conditional probability that oracle selection triggers an update, and $q(\lambda_t)$ defined by \eqref{eq:def_qlamb} is the average null risk among units selected by that rule. Multiplying the product in \eqref{eq:E_update} by $\eta_t$ gives the conditional mean of $\eta_tS_t^0(\lambda_t)(\alpha-E_t)$ before projection.

We further have
\begin{align}
 &\Ebb[S_t(\alpha-\widehat E_t)\mid\mathcal F_{t-1}]
       -\Ebb[S_t^0(\lambda_t)(\alpha-E_t)\mid\mathcal F_{t-1}]
       \notag\\
 &=\Ebb[\{S_t-S_t^0(\lambda_t)\}(\alpha-E_t)
                       \mid\mathcal F_{t-1}]
          -\Ebb[S_t(\widehat E_t-E_t)\mid\mathcal F_{t-1}].
 \label{eq:one-pass-drift-decomposition}
\end{align}
The first term on the right side of \eqref{eq:one-pass-drift-decomposition} has absolute value at most $\Gamma_t$. For the second term, as $S_t=1$ implies $A_t=1$ and $Y_t=Y_t(1)$, we have
\begin{equation*}
 \Ebb[S_t(\widehat E_t-E_t)\mid\mathcal G_t]
 =(1-\varepsilon_t)G_t
   [F_1\{\widehat b_{0,t}(\bX_t)\mid\bX_t\}-\ell(\bX_t)].
\end{equation*}
To bound the difference between $F_1\{\widehat b_{0,t}(\bX_t)\mid\bX_t\}$ and $F_1\{ b_{0}(\bX_t)\mid\bX_t\} = \ell(\bX_t)$, we consider two cases.
\begin{itemize}
    \item If $|\widehat b_{0,t}(\bX_t)-b_0(\bX_t)|\le\delta_U$, Assumption~\ref{ass:local-density} gives
\begin{equation*}
 |F_1\{\widehat b_{0,t}(\bX_t)\mid\bX_t\}-\ell(\bX_t)|
 \le C_U|\widehat b_{0,t}(\bX_t)-b_0(\bX_t)|.
\end{equation*}
\item If $|\widehat b_{0,t}(\bX_t)-b_0(\bX_t)|>\delta_U$, boundedness of distribution functions instead gives
\begin{equation*}
 |F_1\{\widehat b_{0,t}(\bX_t)\mid\bX_t\}-\ell(\bX_t)|
 \le1\le\delta_U^{-1}|\widehat b_{0,t}(\bX_t)-b_0(\bX_t)|.
\end{equation*}
\end{itemize}

Combining the two cases above yields
\begin{equation*}
 |F_1\{\widehat b_{0,t}(\bX_t)\mid\bX_t\}-\ell(\bX_t)|
 \le\max\{C_U,\delta_U^{-1}\}\delta_{b,t}^{\infty}.
\end{equation*}
Since $0\le(1-\varepsilon_t)G_t\le1$ and $\delta_{b,t}^{\infty}$ is $\mathcal F_{t-1}$-measurable, we have
\begin{align*}
 \big|\Ebb[S_t(\widehat E_t-E_t)\mid\mathcal F_{t-1}]\big|\le\Ebb\big[\big|\Ebb[S_t(\widehat E_t-E_t)\mid\mathcal G_t]\big|\big|\mathcal F_{t-1}\big]\leq\max\{C_U,\delta_U^{-1}\}\delta_{b,t}^{\infty},
\end{align*}
while the first term in \eqref{eq:one-pass-drift-decomposition} is bounded by $\Gamma_t$. Applying the triangle inequality proves \eqref{eq:paits-drift-error}.

Next, the two inequalities in \eqref{eq:score-uniform-bound} and $(u+v)^2\le2u^2+2v^2$ give
\begin{align*}
 (\delta_{\ell,t}^{\infty})^2+\Gamma_t^2\le(1+4\overline h^2)(\delta_{\ell,t}^{\infty})^2\le2(1+4\overline h^2)
   \big\{(\delta_{F,t}^{\infty})^2
     +(C_U+\delta_U^{-1})^2(\delta_{b,t}^{\infty})^2\big\}.
\end{align*}
Taking expectations and averaging over $t=1,\ldots,T$ yields
\begin{align*}
 \frac1T\sum_{t=1}^T
   \Ebb[(\delta_{\ell,t}^{\infty})^2+\Gamma_t^2]
\le\frac CT\sum_{t=1}^T
   \big\{\Ebb[(\delta_{F,t}^{\infty})^2]
+\Ebb[(\delta_{b,t}^{\infty})^2]\big\}=O(r_{F,T}^2+r_{b,T}^2),
\end{align*}
where the last equality follows from \eqref{eq:nuisance_rate}, and the constant $C$ depends on $\overline{h},C_U,\delta_U$. This establishes \eqref{eq:score-average-second-moment-rate}.

Finally, we apply the Cauchy--Schwarz inequality twice to obtain
\begin{align*}
 \Ebb\bigg[\frac1T\sum_{t=1}^T\Gamma_t\bigg]
 \le\frac1T\sum_{t=1}^T\{\Ebb(\Gamma_t^2)\}^{1/2}\le\bigg\{\frac1T\sum_{t=1}^T\Ebb(\Gamma_t^2)\bigg\}^{1/2}= O\big((r_{b,T}^2+r_{F,T}^2)^{1/2} \big)
 = O(r_{T}).
\end{align*}
Since $\Gamma_t\ge0$, the Markov's inequality implies, for every $M>0$,
\begin{equation*}
 \Pbb\bigg\{\frac1T\sum_{t=1}^T\Gamma_t
                    >Mr_{T}\bigg\}
 \le\frac{\Ebb[T^{-1}\sum_{t=1}^T\Gamma_t]}
           {Mr_{T}}
 = O(M^{-1}).
\end{equation*}
Taking $M$ arbitrarily large proves
\eqref{eq:score-crossing-average-rate}.
\end{proof}

The following proposition establishes upper bounds for the difference between $\lambda_t$ and $\lambda^\star$, which are then used to compare the realized assignments with the oracle assignments in the proof of Theorem~\ref{thm:fdr-one-pass}.

\begin{proposition}[Time-averaged bounds for the cutoff]
\label{prop:paits-cutoff-moments}
Under the assumptions of Theorem~\ref{thm:fdr-one-pass},
\begin{equation}
 \begin{aligned}
 \frac1T\sum_{t=1}^T\Ebb[(\lambda_t-\lambda^\star)^2]
 &=O(r_{b,T}^2+r_{F,T}^2+T^{-\zeta}),\\
 \frac1T\sum_{t=1}^T|\lambda_t-\lambda^\star|
 &=O_{\Pbb}(r_{T}+T^{-\zeta/2}).
 \end{aligned}
 \label{eq:one-pass-threshold-average-rates}
\end{equation}
\end{proposition}

\begin{proof}
The update rule \eqref{eq:paits-cutoff-update} can be written as $\lambda_{t+1}=\Pi_{[\alpha,1]}[\lambda_t+\eta_t S_t(\alpha-\widehat E_t)]$. Since $\Pi_{[\alpha,1]}(\lambda^\star)=\lambda^\star$ and $|S_t(\alpha-\widehat E_t)|\le1$, we always have
\begin{align*}
 (\lambda_{t+1}-\lambda^\star)^2\le(\lambda_t-\lambda^\star)^2+2\eta_t(\lambda_t-\lambda^\star)S_t(\alpha-\widehat E_t)+\eta_t^2.
\end{align*}
Let $c=(1-\varepsilon_0)c_g>0$. For the second term of the right-hand side, Proposition~\ref{prop:oracle-threshold} and \eqref{eq:paits-drift-error} in Proposition~\ref{prop:score-rate} imply
\begin{align*}
 (\lambda_t-\lambda^\star)\Ebb[S_t(\alpha-\widehat E_t)\mid\mathcal F_{t-1}]\le-c(\lambda_t-\lambda^\star)^2+|\lambda_t-\lambda^\star|(\Gamma_t+C\delta_{b,t}^{\infty}).
\end{align*}
Applying the inequality $2uv\le cu^2+c^{-1}v^2$ and \eqref{eq:score-uniform-bound} taking expectations can yield
\begin{align}
 \Ebb[(\lambda_{t+1}-\lambda^\star)^2]\le(1-c\eta_t)\Ebb[(\lambda_t-\lambda^\star)^2]
 +\eta_t^2 +C'\eta_t\Ebb[(\delta_{b,t}^{\infty})^2
+(\delta_{F,t}^{\infty})^2],
 \label{eq:one-pass-mean-square-recursion}
\end{align}
for some finite $C'>0$.

Choose a fixed $n$ such that \eqref{eq:step-size-order} and $0<c\eta_t<1$ hold for every $t\ge n$, and write $K_{k,t}=\prod_{j=k+1}^{t-1}(1-c\eta_j)$, with empty products equal to one. Iterating \eqref{eq:one-pass-mean-square-recursion} gives
\begin{align}
 \Ebb[(\lambda_t-\lambda^\star)^2]\le\Ebb[(\lambda_n-\lambda^\star)^2]
          \prod_{j=n}^{t-1}(1-c\eta_j)+\sum_{k=n}^{t-1}\eta_k K_{k,t}
    \big\{C'\Ebb[(\delta_{b,k}^{\infty})^2+(\delta_{F,k}^{\infty})^2]+\eta_k\big\}.
 \label{eq:one-pass-mean-square-kernel}
\end{align}
By Assumption \ref{ass:step-sizes}, we claim that
\begin{equation*}
 \sup_{k\ge n}\eta_k\sum_{t=k+1}^{\infty}K_{k,t}<\infty,
 \qquad
 \sum_{t=n+1}^{\infty}\prod_{j=n}^{t-1}(1-c\eta_j)<\infty.
\end{equation*}
To verify the first inequality, let $p=1-\zeta\in(0,1)$ and $a=k+1$. Integral comparison and $1-x\le e^{-x}$ give $K_{k,t}\le\exp\{-c_1(t^p-a^p)\}$ for a constant $c_1>0$. Hence
\begin{align*}
 \sum_{t=k+1}^{\infty}K_{k,t}\le1+\int_a^{\infty}e^{-c_1(x^p-a^p)}dx=1+\frac1p\int_0^{\infty}e^{-c_1u}(a^p+u)^{\zeta/p}du\le C a^\zeta.
\end{align*}
Since $a=k+1$ and $\eta_k\le C_\eta k^{-\zeta}$, this estimate gives
\begin{equation*}
 \eta_k\sum_{t=k+1}^{\infty}K_{k,t}
 \le CC_\eta k^{-\zeta}(k+1)^\zeta
 =CC_\eta\left(1+\frac1k\right)^\zeta
 \le CC_\eta 2^\zeta,\qquad k\ge n.
\end{equation*}
The upper bound is independent of $k$, proving the first assertion.

For the second assertion, keep $n$ fixed. The lower step-size bound and integral comparison imply, for $t\ge n+1$,
\begin{equation*}
 \sum_{j=n}^{t-1}\eta_j
 \ge c_\eta\sum_{j=n}^{t-1}j^{-\zeta}
 \ge c_\eta\int_n^t x^{-\zeta}dx
 =\frac{c_\eta}{p}(t^p-n^p).
\end{equation*}
Using $1-x\le e^{-x}$ again, we obtain
\begin{align*}
 \prod_{j=n}^{t-1}(1-c\eta_j)\le\exp\bigg(-c\sum_{j=n}^{t-1}\eta_j\bigg)\le\exp\bigg\{-\frac{cc_\eta}{p}(t^p-n^p)\bigg\} =C_n e^{-c_2t^p},
\end{align*}
where $c_2=cc_\eta/p>0$ and $C_n=e^{c_2n^p}$ is finite because $n$ is fixed. Since $x\mapsto e^{-c_2x^p}$ is decreasing,
\begin{align*}
 \sum_{t=n+1}^{\infty}\prod_{j=n}^{t-1}(1-c\eta_j)\le C_n\sum_{t=n+1}^{\infty}e^{-c_2t^p}\le C_n\int_n^{\infty}e^{-c_2x^p}dx=\frac{C_n}{p}\int_{n^p}^{\infty} e^{-c_2u}u^{1/p-1}du<\infty.
\end{align*}
The last equality uses $u=x^p$; the integral is finite because $p>0$ and the exponential factor dominates the fixed power of $u$. This proves the second assertion.

Summing \eqref{eq:one-pass-mean-square-kernel} over $t\le T$ gives
\begin{equation}
 \sum_{t=1}^T\Ebb[(\lambda_t-\lambda^\star)^2]\le C_1+C_2 \sum_{t=1}^T\Ebb[(\delta_{b,t}^{\infty})^2+(\delta_{F,t}^{\infty})^2]+C_3 \sum_{t=1}^T\eta_t,
 \label{eq:one-pass-cutoff-cumulative-second-moment}
\end{equation}
for some finite $C_1,C_2,C_3>0$. We also have
\begin{equation}
 \bar\eta_T:=\frac1T\sum_{t=1}^T\eta_t=O(T^{-\zeta}).
 \label{eq:step-size-calendar-average-endpoint}
\end{equation}
Combining \eqref{eq:nuisance_rate} and \eqref{eq:step-size-calendar-average-endpoint}, the first equation in \eqref{eq:one-pass-threshold-average-rates} is proved. And
\begin{equation*}
 \Ebb\bigg[\frac1T\sum_{t=1}^T|\lambda_t-\lambda^\star|\bigg]\le\bigg\{\frac1T\sum_{t=1}^T\Ebb[(\lambda_t-\lambda^\star)^2]\bigg\}^{1/2},
\end{equation*}
therefore Markov's inequality gives the second equation of \eqref{eq:one-pass-threshold-average-rates}.
\end{proof}

Then, we are in a position to show the main FDR theorem.

\begin{proof}[Proof of Theorem~\ref{thm:fdr-one-pass}]
By the first conclusion in \eqref{eq:one-pass-threshold-average-rates}, together with $r_{b,T}+r_{F,T}\to0$ and $\zeta>0$, we have the time-averaged mean-square convergence in Theorem~\ref{thm:fdr-one-pass}, i.e.,
$$ \frac1T\sum_{t=1}^T\Ebb[(\lambda_t-\lambda^\star)^2]\longrightarrow0.$$

Next, we continue to prove \eqref{eq:one-pass-frequency-fdp-rates} and the convergence of FDR. We first compare the plug-in selection rule $S_t$ with $S_t^0(\lambda^\star)=(1-B_t)A_t^\star$:
\begin{equation*}
 |S_t-S_t^0(\lambda^\star)|
 \le |S_t-S_t^0(\lambda_t)|
      +|S_t^0(\lambda_t)-S_t^0(\lambda^\star)|.
\end{equation*}
When $B_t=0$, the first difference $|S_t-S_t^0(\lambda_t)|$ can be nonzero only if $\widehat\ell_t(\bX_t)$ and $\ell(\bX_t)$ lie on different sides of $\lambda_t$, which implies $(\widehat\ell_t(\bX_t)-\lambda_t)(\ell(\bX_t)-\lambda_t)\le0$. The second difference $|S_t^0(\lambda_t)-S_t^0(\lambda^\star)|$ can be nonzero only if $\ell(\bX_t)$ lies between $\lambda_t$ and $\lambda^\star$, indicating $|\ell(\bX_t)-\lambda^\star|\le|\lambda_t-\lambda^\star|$. Both differences vanish when $B_t=1$. Consequently,
\begin{align}
 |S_t-S_t^0(\lambda^\star)|
 &\le(1-B_t)\ind\{|\ell(\bX_t)-\lambda^\star|
                         \le|\lambda_t-\lambda^\star|\}\notag\\
 &\quad+(1-B_t)\ind\big\{
      (\widehat\ell_t(\bX_t)-\lambda_t)
      (\ell(\bX_t)-\lambda_t)\le0\big\}.
 \label{eq:one-pass-selector-decomposition}
\end{align}
Conditional on $\mathcal F_{t-1}$, the cutoff $\lambda_t$ and the estimated function $\widehat\ell_t$ are fixed, and $\bX_t$ has distribution $P_X$. Independence of the exploration randomization gives
\begin{align*}
 &\Ebb[(1-B_t)
   \ind\{|\ell(\bX_t)-\lambda^\star|
                   \le|\lambda_t-\lambda^\star|\}
                   \mid\mathcal F_{t-1}]=(1-\varepsilon_t)
   \Pbb\{|\ell(\bX_t)-\lambda^\star|
                   \le|\lambda_t-\lambda^\star|
                   \mid\mathcal F_{t-1}\}\\
 &\le2\overline h|\lambda_t-\lambda^\star|.
\end{align*}
The last inequality integrates the density $h\le\overline h$ over an interval of length at most $2|\lambda_t-\lambda^\star|$. For the crossing term, by the definition of $\Gamma_t$ in Proposition \ref{prop:score-rate}, we have
\begin{align*}
 &\Ebb[(1-B_t)\ind\big\{
   (\widehat\ell_t(\bX_t)-\lambda_t)
   (\ell(\bX_t)-\lambda_t)\le0\big\}
   \mid\mathcal F_{t-1}]\\
 &=(1-\varepsilon_t)
   \Pbb\big\{(\widehat\ell_t(\bX_t)-\lambda_t)
             (\ell(\bX_t)-\lambda_t)\le0
             \mid\mathcal F_{t-1}\big\}\le\Gamma_t,
\end{align*}
as $\lambda_t\in[\alpha,1]$ is measurable with respect to $\mathcal F_{t-1}$ and $\Gamma_t$ takes the supremum over that interval. Taking conditional expectations in \eqref{eq:one-pass-selector-decomposition} therefore gives
\begin{equation}
 \Ebb\Big[|S_t-S_t^0(\lambda^\star)|\Big|\mathcal F_{t-1}\Big]
 \le2\overline h|\lambda_t-\lambda^\star|+\Gamma_t.
 \label{eq:one-pass-selector-conditional-bound}
\end{equation}

We next derive the upper bounds of the averages on the right-hand side of \eqref{eq:one-pass-selector-conditional-bound}. The first conclusion in \eqref{eq:one-pass-threshold-average-rates} implies
\begin{align*}
 \frac1T\sum_{t=1}^T\Ebb|\lambda_t-\lambda^\star|
 \le\bigg\{\frac1T\sum_{t=1}^T
                   \Ebb[(\lambda_t-\lambda^\star)^2]\bigg\}^{1/2}= O(r_{b,T}^2+r_{F,T}^2+T^{-\zeta})^{1/2}= O(r_{T}+T^{-\zeta/2}).
\end{align*}
Similarly, the second-moment bound \eqref{eq:score-average-second-moment-rate} in Proposition~\ref{prop:score-rate} gives
\begin{equation*}
 \frac1T\sum_{t=1}^T\Ebb(\Gamma_t)
 \le\bigg\{\frac1T\sum_{t=1}^T\Ebb(\Gamma_t^2)\bigg\}^{1/2}
 = O\big((r_{b,T}^2+r_{F,T}^2)^{1/2}\big)
 = O(r_{T}).
\end{equation*}
Taking expectations in \eqref{eq:one-pass-selector-conditional-bound} and summing over $t$ now yields
\begin{align*}
 \Ebb\bigg[\frac1T\sum_{t=1}^T
                   |S_t-S_t^0(\lambda^\star)|\bigg]
 \le\frac{2\overline h}{T}\sum_{t=1}^T
                   \Ebb|\lambda_t-\lambda^\star|
       +\frac1T\sum_{t=1}^T\Ebb(\Gamma_t)= O(r_{T}+T^{-\zeta/2}).
\end{align*}
Since the average inside the expectation is nonnegative, for every $M>0$ and all sufficiently large $T$, Markov's inequality gives
\begin{align*}
 &\Pbb\bigg\{\frac1T\sum_{t=1}^T|S_t-S_t^0(\lambda^\star)|
          >M(r_{T}+T^{-\zeta/2})\bigg\}\le
 \frac{\Ebb[T^{-1}\sum_{t=1}^T|S_t-S_t^0(\lambda^\star)|]}
      {M(r_{T}+T^{-\zeta/2})}
 \le\frac CM.
\end{align*}
Therefore,
\begin{equation*}
 \frac1T\sum_{t=1}^T|S_t-S_t^0(\lambda^\star)|
 =O_{\Pbb}(r_{T}+T^{-\zeta/2}).
\end{equation*}

Now we turn to exploratory treatments; recall that
\begin{equation*}
 A_t=S_t+B_tZ_t,\qquad R_T=N_T+Q_T.
\end{equation*}
Since $A_t^\star=S_t^0(\lambda^\star)+B_tA_t^\star$, we have
\begin{equation*}
 A_t-A_t^\star
 =S_t-S_t^0(\lambda^\star)+B_t(Z_t-A_t^\star),
\end{equation*}
where both $Z_t$ and $A_t^\star$ take values in $\{0,1\}$,  implying $|Z_t-A_t^\star|\le1$. The triangle inequality therefore gives
\begin{equation*}
 |A_t-A_t^\star|
 \le|S_t-S_t^0(\lambda^\star)|+B_t|Z_t-A_t^\star|
 \le|S_t-S_t^0(\lambda^\star)|+B_t.
\end{equation*}
By Lemma~\ref{lem:epsilon-greedy-exposure}, we have
\begin{equation*}
 \Ebb\bigg[\frac1T\sum_{t=1}^T B_t\bigg]
 =\frac1T\sum_{t=1}^T\varepsilon_t
 =O\{\mathfrak H_\chi(T)/T\},
\end{equation*}
and consequently we have
\begin{equation}\label{eq:bound_e_ave_A}
\begin{aligned}
 \Ebb\bigg[\frac1T\sum_{t=1}^T|A_t-A_t^\star|\bigg]
 &\le\Ebb\bigg[\frac1T\sum_{t=1}^T
                   |S_t-S_t^0(\lambda^\star)|
       +\frac1T\sum_{t=1}^T\varepsilon_t \bigg]= O\bigg(r_{T}+T^{-\zeta/2}
                       +\frac{\mathfrak H_\chi(T)}T\bigg).
\end{aligned}
\end{equation}

We now compute the oracle treatment and null counts. Each pair $(A_t^\star,A_t^\star E_t)$ is a function of $(\bX_t,Y_t(1))$ and $\lambda^\star$. Moreover,
\begin{equation*}
 \Ebb(A_t^\star)
 =\Pbb\{\ell(\bX_t)\le\lambda^\star\}
 =H(\lambda^\star)\ge H(\alpha)>0.
\end{equation*}
Using the facts that $\Ebb(E_t\mid\bX_t)=\ell(\bX_t)$ and the identity $q(\lambda^\star)=\alpha$ from Proposition~\ref{prop:oracle-threshold}, we obtain
\begin{align*}
 \Ebb(A_t^\star E_t)
 &=\Ebb[\ind\{\ell(\bX_t)\le\lambda^\star\}
                         \Ebb(E_t\mid\bX_t)]=\Ebb[\ind\{\ell(\bX_t)\le\lambda^\star\}
                         \ell(\bX_t)]\\
 &=\int_0^{\lambda^\star}u h(u)du
 =H(\lambda^\star)q(\lambda^\star)
 =\alpha H(\lambda^\star).
\end{align*}
Since $A_t^\star$ is an indicator and the observations are independent across $t$,
\begin{align*}
 \Ebb\bigg[\bigg\{\frac1T\sum_{t=1}^T A_t^\star
-H(\lambda^\star)\bigg\}^2\bigg]=\frac1{T^2}\sum_{t=1}^T\Var(A_t^\star)\le\frac1{4T}.
\end{align*}
Similarly, we have
\begin{equation*}
 \Ebb\bigg[\bigg\{\frac1T\sum_{t=1}^T A_t^\star E_t
                         -\alpha H(\lambda^\star)\bigg\}^2\bigg]
 \le\frac1{4T}.
\end{equation*}
By the Cauchy--Schwarz inequality, we further have
\begin{equation*}
 \Ebb\bigg|\frac1T\sum_{t=1}^T A_t^\star-H(\lambda^\star)\bigg|
 \le\frac1{2\sqrt T},
\end{equation*}
and
\begin{equation*}
 \Ebb\bigg|\frac1T\sum_{t=1}^T A_t^\star E_t
                         -\alpha H(\lambda^\star)\bigg|
 \le\frac1{2\sqrt T}.
\end{equation*}
Recall that $R_T=\sum_{t=1}^T A_t$ and $V_T=\sum_{t=1}^T A_tE_t$. We have the following decomposition
\begin{equation*}
 \bigg|\frac{R_T}T-H(\lambda^\star)\bigg|
 \le\frac1T\sum_{t=1}^T|A_t-A_t^\star|
       +\bigg|\frac1T\sum_{t=1}^T A_t^\star-H(\lambda^\star)\bigg|.
\end{equation*}
Furthermore, since $0\le E_t\le1$, we have
\begin{align*}
 \bigg|\frac{V_T}T-\alpha H(\lambda^\star)\bigg|
 &\le\frac1T\sum_{t=1}^T|A_t-A_t^\star|E_t
       +\bigg|\frac1T\sum_{t=1}^T A_t^\star E_t
                            -\alpha H(\lambda^\star)\bigg|\\
 &\le\frac1T\sum_{t=1}^T|A_t-A_t^\star|
       +\bigg|\frac1T\sum_{t=1}^T A_t^\star E_t
                            -\alpha H(\lambda^\star)\bigg|.
\end{align*}
Combining the above inequalities and \eqref{eq:bound_e_ave_A} yields
\begin{equation}\label{eq:bound42counts}
    \begin{aligned}
 &\Ebb\bigg|\frac{R_T}T-H(\lambda^\star)\bigg|
  +\Ebb\bigg|\frac{V_T}T-\alpha H(\lambda^\star)\bigg|\\
 &\le2\Ebb\bigg[\frac1T\sum_{t=1}^T|A_t-A_t^\star|\bigg]+T^{-1/2}= O\bigg(r_{T}+T^{-\zeta/2}+\frac{\mathfrak H_\chi(T)}T\bigg).
\end{aligned}
\end{equation}
Here $T^{-1/2}\le T^{-\zeta/2}$ because $\zeta<1$. As the maximum of the two absolute deviations is bounded by their sum, by Markov's inequality, we have
\begin{equation*}
 \max\bigg\{\bigg|\frac{R_T}T-H(\lambda^\star)\bigg|,
              \bigg|\frac{V_T}T-\alpha H(\lambda^\star)\bigg|\bigg\}
 =O_{\Pbb}\bigg(r_{T}+T^{-\zeta/2}
                       +\frac{\mathfrak H_\chi(T)}T\bigg).
\end{equation*}

It remains to justify the bound of $|\FDR_T-\alpha|$. Since
$r_{T}\to0$, $\zeta>0$, and $\mathfrak H_\chi(T)/T\to0$, we also have
\begin{align*}
 \Pbb\bigg\{\frac{R_T}T<\frac{H(\lambda^\star)}2\bigg\}
 \le\Pbb\bigg\{\bigg|\frac{R_T}T-H(\lambda^\star)\bigg|
 >\frac{H(\lambda^\star)}2\bigg\}\le\frac{2}{H(\lambda^\star)}
          \Ebb\bigg|\frac{R_T}T-H(\lambda^\star)\bigg|
 \longrightarrow0.
\end{align*}
On the event $R_T/T\ge H(\lambda^\star)/2$, the integer $R_T$ is positive, so $R_T\vee1=R_T$ and $\FDP_T=V_T/R_T$. Hence
\begin{align*}
 |\FDP_T-\alpha|=\frac{|V_T/T-\alpha R_T/T|}{R_T/T}\le\frac{2}{H(\lambda^\star)}
   \bigg\{\bigg|\frac{V_T}T-\alpha H(\lambda^\star)\bigg|
        +\alpha\bigg|\frac{R_T}T-H(\lambda^\star)\bigg|\bigg\}.
\end{align*}
Therefore, for any $x>0$, we have 
\begin{align*}
 &\Pbb\{|\FDP_T-\alpha|>x\}\\
 &\le\Pbb\bigg\{\frac{R_T}T<\frac{H(\lambda^\star)}2\bigg\}+\Pbb\bigg\{
       \left|\frac{V_T}T-\alpha H(\lambda^\star)\right|
       +\alpha\left|\frac{R_T}T-H(\lambda^\star)\right|
                       >\frac{H(\lambda^\star)x}{2}\bigg\}\\
 &\le\Pbb\bigg\{\frac{R_T}T<\frac{H(\lambda^\star)}2\bigg\}+\frac{2}{H(\lambda^\star)x}
        \bigg\{\Ebb\bigg|\frac{V_T}T-\alpha H(\lambda^\star)\bigg|
          +\alpha\Ebb\bigg|\frac{R_T}T-H(\lambda^\star)\bigg|\bigg\},
\end{align*}
Taking $x=M\{r_{T}+T^{-\zeta/2}+\mathfrak H_\chi(T)/T\}$ and using \eqref{eq:bound42counts} yields
\begin{align*}
 \Pbb\bigg\{|\FDP_T-\alpha|
       >M\bigg(r_{T}+T^{-\zeta/2}
                     +\frac{\mathfrak H_\chi(T)}T\bigg)\bigg\}\le\Pbb\bigg\{\frac{R_T}T<\frac{H(\lambda^\star)}2\bigg\}+\frac CM.
\end{align*}
Taking $M$ and $T$ sufficiently large establishes the claimed FDP rate, and subsequently, \eqref{eq:one-pass-frequency-fdp-rates} in Theorem~\ref{thm:fdr-one-pass}.

Finally, since the rate on the right-hand side of \eqref{eq:one-pass-frequency-fdp-rates} tends to zero, we have $\FDP_T\to\alpha$ in probability. Because $0\le\FDP_T\le1$ and $0<\alpha<1$, we have $|\FDP_T-\alpha|\le1$. For every $u\in(0,1)$, we have
\begin{align*}
 \Ebb|\FDP_T-\alpha|
 &=\Ebb\big[|\FDP_T-\alpha|
                  \ind\{|\FDP_T-\alpha|\le u\}\big]+\Ebb\big[|\FDP_T-\alpha|
                  \ind\{|\FDP_T-\alpha|>u\}\big]\\
 &\le u+\Pbb\{|\FDP_T-\alpha|>u\}.
\end{align*}
Letting $T\to\infty$ gives $\limsup_{T\to\infty}\Ebb|\FDP_T-\alpha|\le u$. Since $u$ can be arbitrarily small, this expectation converges to zero. Consequently,
\begin{equation*}
 |\FDR_T-\alpha|
 =|\Ebb(\FDP_T)-\alpha|
 \le\Ebb|\FDP_T-\alpha|
 \longrightarrow0.
\end{equation*}
\end{proof}

\begin{proof}[Proof of Corollary~\ref{cor:paits-cutoff-mean-square}]
Let $c=(1-\varepsilon_0)c_g>0$, $C'$, $n$, and $K_{k,t}$ be as in the proof of Proposition~\ref{prop:paits-cutoff-moments}. Thus \eqref{eq:one-pass-mean-square-recursion} holds for every $t\ge1$, and both \eqref{eq:step-size-order} and $0<c\eta_t<1$ hold for every $t\ge n$.

We first collect three facts used below. The initial cutoff satisfies $\lambda_1\in[\alpha,1]$, and each update in \eqref{eq:paits-cutoff-update} either leaves the cutoff unchanged or projects it onto $[\alpha,1]$, so $\lambda_t\in[\alpha,1]$ for every $t$. Since $\lambda^\star\in(\alpha,\lambda_+)\subset[\alpha,1]$ by Proposition~\ref{prop:oracle-threshold}, we have $\Ebb[(\lambda_n-\lambda^\star)^2]\le(1-\alpha)^2$. Next, under the assumptions of Theorem~\ref{thm:fdr-one-pass}, the averages in \eqref{eq:nuisance_rate} are finite for all sufficiently large $T$. Their summands are nonnegative, so every summand is finite. Taking such a $T\ge k$ gives
\begin{equation*}
 C'\Ebb[(\delta_{b,k}^{\infty})^2+(\delta_{F,k}^{\infty})^2]+\eta_k<\infty,
 \qquad k\ge1.
\end{equation*}
Finally, \eqref{eq:step-size-order} gives $c_\eta t^{-\zeta}\le\eta_t\le C_\eta t^{-\zeta}$ for $t\ge n$. The upper bound yields $\eta_t\to0$. Since $\zeta<1$ implies $j^{-\zeta}\ge j^{-1}$, yielding $\sum_{j=k+1}^{t-1}\eta_j\ge c_\eta\sum_{j=k+1}^{t-1}j^{-1}\to\infty$ as $t\to\infty$ for every fixed $k\ge n-1$. Because $0<c\eta_j<1$ and $1-c\eta_j\le\exp(-c\eta_j)$ for $j\ge n$, it follows that, when $t\to\infty$,
\begin{equation*}
 0<K_{k,t}\le\exp\Big(-c\sum_{j=k+1}^{t-1}\eta_j\Big)\longrightarrow0
 \qquad k\ge n-1 .
\end{equation*}

To control the total weight of the perturbation terms in \eqref{eq:one-pass-mean-square-kernel}, fix $t>m\ge n$ and write each summand as a difference of two consecutive products:
\begin{align*}
 c\eta_kK_{k,t}
 =\{1-(1-c\eta_k)\}\prod_{j=k+1}^{t-1}(1-c\eta_j)=\prod_{j=k+1}^{t-1}(1-c\eta_j)-\prod_{j=k}^{t-1}(1-c\eta_j)
 =K_{k,t}-K_{k-1,t}.
\end{align*}
Hence $c\sum_{k=m}^{t-1}\eta_kK_{k,t}=K_{t-1,t}-K_{m-1,t}=1-K_{m-1,t}$, where $K_{m-1,t}\in(0,1)$ because every factor $1-c\eta_j$ with $j\ge m\ge n$ lies in $(0,1)$. Therefore,
\begin{equation*}
 0\le\sum_{k=m}^{t-1}\eta_kK_{k,t}\le\frac1c,\qquad t>m\ge n.
\end{equation*}
And the total weight is bounded uniformly in $t$ and $m$.

Now we consider the right-hand side of \eqref{eq:one-pass-mean-square-kernel}. The first term
$$\Ebb[(\lambda_n-\lambda^\star)^2]K_{n-1,t}\le(1-\alpha)^2K_{n-1,t}\to0.$$ 
For the second term, fix $m\ge n$ and split the sum in \eqref{eq:one-pass-mean-square-kernel} at $k=m$.  The summands with $n\le k<m$ are finitely many, their number does not depend on $t$, and each is the finite factor $\eta_k\{C'\Ebb[(\delta_{b,k}^{\infty})^2+(\delta_{F,k}^{\infty})^2]+\eta_k\}$
multiplied by $K_{k,t}\to0$. Hence
$$\sum_{k=n}^{m-1}\eta_k K_{k,t}
    \big\{C'\Ebb[(\delta_{b,k}^{\infty})^2
                    +(\delta_{F,k}^{\infty})^2]+\eta_k\big\}\to0.$$
For the remaining summands $\sum_{k=m}^{t-1}\eta_k K_{k,t}\big\{C'\Ebb[(\delta_{b,k}^{\infty})^2+(\delta_{F,k}^{\infty})^2]+\eta_k\big\}$, the additional condition of Corollary~\ref{cor:paits-cutoff-mean-square} and $\eta_k\to0$ imply that
$$C'\Ebb[(\delta_{b,k}^{\infty})^2+(\delta_{F,k}^{\infty})^2]+\eta_k\to0.$$
Bounding each factor by its supremum over $k\ge m$ and using the total-weight bound above, we obtain
\begin{equation*}
 \limsup_{t\to\infty}\Ebb[(\lambda_t-\lambda^\star)^2]
 \le\frac1c\sup_{k\ge m}
 \big\{C'\Ebb[(\delta_{b,k}^{\infty})^2+(\delta_{F,k}^{\infty})^2]
       +\eta_k\big\}.
\end{equation*}
The left-hand side does not depend on $m$, whereas the tail suprema on the right decrease to zero as $m\to\infty$. Letting $m\to\infty$ proves $\Ebb[(\lambda_t-\lambda^\star)^2]\to0$.
\end{proof}

\subsection{Regret bounds}

We then consider the regret bounds of the proposed method.

\begin{proof}[Proof of Theorem~\ref{thm:one-pass-weighted-selection-regret}]
Recall that $G_t=\ind\{\widehat\ell_t(\bX_t)\le\lambda_t\}$, $S_t=(1-B_t)G_t$, and $A_t=S_t+B_tZ_t$. Since the plug-in and exploration assignments occur on disjoint events, we have
\begin{equation}
 \begin{aligned}
 |\ell(\bX_t)-\lambda^\star||A_t-A_t^\star|
 &=(1-B_t)|\ell(\bX_t)-\lambda^\star||G_t-A_t^\star|+B_t|\ell(\bX_t)-\lambda^\star||Z_t-A_t^\star|.
 \end{aligned}
 \label{eq:one-pass-weighted-selection-decomposition}
\end{equation}
We bound the conditional expectations of the two terms on the right-hand side of \eqref{eq:one-pass-weighted-selection-decomposition}. Throughout this proof, $C$ denotes a finite constant independent of $T$.

Let $\delta_{\ell,t}^{\infty}$ be defined as in Proposition~\ref{prop:score-rate}. The fitted score $\widehat\ell_t$, the cutoff $\lambda_t$, and $\delta_{\ell,t}^{\infty}$ are $\mathcal F_{t-1}$-measurable. For the first term, we consider two cases.
\begin{itemize}
\item If $G_t=1$ and $A_t^\star=0$, then,
almost surely,
\begin{align*}
 0<\ell(\bX_t)-\lambda^\star
 &\le\ell(\bX_t)-\widehat\ell_t(\bX_t)
+\lambda_t-\lambda^\star\le\delta_{\ell,t}^{\infty}+|\lambda_t-\lambda^\star|.
\end{align*}
\item If $G_t=0$ and $A_t^\star=1$, then
\begin{align*}
 0\le\lambda^\star-\ell(\bX_t)
 &\le\lambda^\star-\lambda_t
       +\widehat\ell_t(\bX_t)-\ell(\bX_t)\le\delta_{\ell,t}^{\infty}+|\lambda_t-\lambda^\star|.
\end{align*}
\end{itemize}
Combining the two cases gives
\begin{equation}
 |G_t-A_t^\star|
 \le\ind\big\{|\ell(\bX_t)-\lambda^\star|
       \le\delta_{\ell,t}^{\infty}+|\lambda_t-\lambda^\star|\big\}.
 \label{eq:one-pass-weighted-selection-boundary}
\end{equation}
Since $\bX_t$ is independent of $\mathcal F_{t-1}$, its conditional distribution given $\mathcal F_{t-1}$ remains $P_X$. Therefore, the conditional distribution of $\ell(\bX_t)$ coincides with the population distribution of $\ell(\bX)$. By Assumption~\ref{ass:oracle-score}, its distribution is absolutely continuous on $[0,1]$ with density $h$, and the density is uniformly bounded:
$$
    \overline h:=\|h\|_\infty
    =\sup_{u\in[0,1]}h(u)<\infty.
$$
Equation~\eqref{eq:one-pass-weighted-selection-boundary} implies that $G_t\neq A_t^\star$ only when the true score $\ell(\bX_t)$ is sufficiently close to the oracle threshold $\lambda^\star$. More precisely, $G_t\neq A_t^\star$ only if
$$
|\ell(\bX_t)-\lambda^\star|
\le
\delta_{\ell,t}^{\infty}
+
|\lambda_t-\lambda^\star|.
$$
Therefore, combining
\eqref{eq:one-pass-weighted-selection-boundary}, we obtain
\begin{equation}
\begin{aligned}
&\Ebb[
|\ell(\bX_t)-\lambda^\star||G_t-A_t^\star|
\mid\mathcal F_{t-1}]\le
\int_0^1 |u-\lambda^\star|
\ind\big\{|u-\lambda^\star|
\le\delta_{\ell,t}^{\infty}+|\lambda_t-\lambda^\star|\big\}
h(u)du\\
&\quad\le
2\overline h
\int_0^{\delta_{\ell,t}^{\infty}+|\lambda_t-\lambda^\star|}
vdv=
\overline h
\big\{\delta_{\ell,t}^{\infty}
+|\lambda_t-\lambda^\star|\big\}^2\le
2\overline h
\big\{(\delta_{\ell,t}^{\infty})^2
+|\lambda_t-\lambda^\star|^2\big\}.
\end{aligned}
\label{eq:one-pass-weighted-selection-plugin-bound}
\end{equation}
For the second term in the right-hand side of \eqref{eq:one-pass-weighted-selection-decomposition}, since $\ell(\bX_t),\lambda^\star\in[0,1]$ and $Z_t,A_t^\star\in\{0,1\}$, we have
\begin{equation}
 \begin{aligned}
 &\Ebb[
   B_t|\ell(\bX_t)-\lambda^\star||Z_t-A_t^\star|
   \mid\mathcal F_{t-1}]\le\Ebb[B_t\mid\mathcal F_{t-1}]=\varepsilon_t.
 \end{aligned}
 \label{eq:one-pass-weighted-selection-exploration-bound}
\end{equation}
Taking conditional expectations in \eqref{eq:one-pass-weighted-selection-decomposition}, using $1-B_t\le1$, and applying \eqref{eq:one-pass-weighted-selection-plugin-bound} and \eqref{eq:one-pass-weighted-selection-exploration-bound} yields
\begin{equation}
 \begin{aligned}
 &\Ebb[
   |\ell(\bX_t)-\lambda^\star||A_t-A_t^\star|
   \mid\mathcal F_{t-1}]\le2\overline h(\delta_{\ell,t}^{\infty})^2
       +2\overline h|\lambda_t-\lambda^\star|^2+\varepsilon_t
       \quad\text{a.s.}
 \end{aligned}
 \label{eq:one-pass-weighted-selection-one-step-quadratic}
\end{equation}
The first inequality in \eqref{eq:score-uniform-bound} and $(u+v)^2\le2u^2+2v^2$ give
\begin{equation}
 (\delta_{\ell,t}^{\infty})^2
 \le2(\delta_{F,t}^{\infty})^2
      +2(C_U+\delta_U^{-1})^2(\delta_{b,t}^{\infty})^2.
 \label{eq:one-pass-score-squared-bound}
\end{equation}
Substituting \eqref{eq:one-pass-score-squared-bound} into \eqref{eq:one-pass-weighted-selection-one-step-quadratic} and summing over $t\le T$ proves \eqref{eq:one-pass-weighted-selection-quadratic-transfer}. Next, taking expectations in \eqref{eq:one-pass-weighted-selection-quadratic-transfer} and applying \eqref{eq:one-pass-cutoff-cumulative-second-moment} in Proposition~\ref{prop:paits-cutoff-moments} gives
\begin{equation}
 \begin{aligned}
 \Regret_T=\Ebb[\widehat{\Regret}_T]
 &\le C\sum_{t=1}^T
   \Big\{\Ebb[(\delta_{b,t}^{\infty})^2]
          +\Ebb[(\delta_{F,t}^{\infty})^2]
          +\Ebb[|\lambda_t-\lambda^\star|^2]+\varepsilon_t\Big\}\\
 &=O\bigg(1+\sum_{t=1}^T
   \Big\{\Ebb[(\delta_{b,t}^{\infty})^2]
          +\Ebb[(\delta_{F,t}^{\infty})^2]
          +\eta_t+\varepsilon_t\Big\}\bigg)\\
 &=O(T\{r_{b,T}^2+r_{F,T}^2+\bar\eta_T+\bar\varepsilon_T+T^{-1}\}).
 \end{aligned}
 \label{eq:one-pass-weighted-selection-expectation-bound}
\end{equation}
The last equation uses \eqref{eq:nuisance_rate} and the definitions of $\bar\eta_T$ and $\bar\varepsilon_T$ in \eqref{eq:step-size-calendar-average-endpoint} and \eqref{eq:exploration-sample-size-rate}. Hence \eqref{eq:one-pass-weighted-selection-quadratic-rate} holds for all sufficiently large $T$. Since $\widehat{\Regret}_T\ge0$, Markov's inequality and \eqref{eq:one-pass-weighted-selection-quadratic-rate} give, for every $M>0$ and all sufficiently large $T$,
\begin{equation}
 \begin{aligned}
 &\Pbb(
  \widehat{\Regret}_T>
  MT\{r_{b,T}^2+r_{F,T}^2+\bar\eta_T+\bar\varepsilon_T+T^{-1}\}
 )\le
 \frac{\Ebb[\widehat{\Regret}_T]}
 {MT\{r_{b,T}^2+r_{F,T}^2+\bar\eta_T+\bar\varepsilon_T+T^{-1}\}}
 \le\frac{C}{M}.
 \end{aligned}
 \label{eq:one-pass-weighted-selection-probability-bound}
\end{equation}
Equation~\eqref{eq:one-pass-weighted-selection-probability-bound} implies $\widehat{\Regret}_T=O_{\Pbb}[T\{r_{b,T}^2+r_{F,T}^2+\bar\eta_T+\bar\varepsilon_T+T^{-1}\}]$. Finally, as $r_{b,T}+r_{F,T}\to0$, $\bar\eta_T\to0$ and $\bar\varepsilon_T\to0$ implied by \eqref{eq:exploration-sample-size-rate} and \eqref{eq:step-size-calendar-average-endpoint}. Dividing \eqref{eq:one-pass-weighted-selection-quadratic-rate} by $T$ therefore gives $\Regret_T/T\to0$. For every fixed $u>0$, Markov's inequality and \eqref{eq:weighted-selection-regret} further yield
\begin{equation*}
 \Pbb\bigg(\frac{\widehat{\Regret}_T}{T}>u\bigg)
 \le\frac{\Ebb[\widehat{\Regret}_T]}{uT}
 =\frac{\Regret_T}{uT}\longrightarrow0,
\end{equation*}
which proves $\widehat{\Regret}_T/T\xrightarrow{\Pbb}0$.
\end{proof}

\subsection{Minimax bounds}
\label{app:minimax-proofs}

The proof relates nuisance-estimation accuracy to the number of assignments that disagree with the oracle. We construct models with the same score and oracle rule but different nuisance functions. An accurate estimator must distinguish these models, which requires observations under the treatment condition opposite to the oracle assignment in certain regions. We first bound regret by the number of these observations, then give two lemmas that quantify the information they provide, and finally construct the required models.

Specifically, we consider two disjoint interior cubes $\mathcal X_-,\mathcal X_+\subset(0,1)^d$ of positive volume such that, for a constant $g_0>0$,
\begin{equation}
 \ell(\bx)\le\lambda^\star-g_0\quad(\bx\in\mathcal X_-),
 \qquad
 \ell(\bx)\ge\lambda^\star+g_0\quad(\bx\in\mathcal X_+).
 \label{eq:minimax-interior-gap}
\end{equation}
By \eqref{eq:minimax-interior-gap}, the oracle assigns treatment on $\mathcal X_-$ and control on $\mathcal X_+$. In the models constructed in the following, learning $b_0$ on $\mathcal X_-$ requires control observations, whereas learning $F_1$ on $\mathcal X_+$ requires treatment observations. We count these observations by
\begin{equation}
 N_{0,-}(T)=\sum_{t=1}^T\ind\{\bX_t\in\mathcal X_-,A_t=0\},
 \qquad
 N_{1,+}(T)=\sum_{t=1}^T\ind\{\bX_t\in\mathcal X_+,A_t=1\}.
 \label{eq:minimax-informative-counts}
\end{equation}
Each observation counted in \eqref{eq:minimax-informative-counts} disagrees with the oracle and incurs regret at least $g_0$. Thus, \eqref{eq:minimax-interior-gap} gives
\begin{equation}
|\ell(\bX_t)-\lambda^\star|\ind\{A_t\ne A_t^\star\}\ge g_0\big[\ind\{\bX_t\in\mathcal X_-,A_t=0\}+\ind\{\bX_t\in\mathcal X_+,A_t=1\}\big].
 \label{eq:minimax-one-period-regret}
\end{equation}
Summing \eqref{eq:minimax-one-period-regret} over time, taking expectations, and using \eqref{eq:minimax-informative-counts} yields
\begin{equation}
 \Regret_T(\mathcal A,P)
 \ge g_0\Ebb_{P,\mathcal A}\{N_{0,-}(T)+N_{1,+}(T)\},
 \label{eq:minimax-count-to-regret}
\end{equation}
where $\Ebb_{P,\mathcal A}$ denotes the expectation under the joint distributions of the historical data and the online observations through time $T$

The first lemma identifies the information available under adaptive assignment. Conditional on the same observed history and current covariates, the procedure uses the same assignment rule under two candidate models. Thus, the new information for distinguishing them comes only from the observed outcome. The identity below will let us bound this information in terms of the counts in \eqref{eq:minimax-informative-counts}.

\begin{lemma}[Information in the observed outcomes]
\label{lem:minimax-adaptive-KL}
Let $P$ and $Q$ be two distributions of $(\bX,Y(0),Y(1))$ with the same covariate marginal $P_X$. Let $\mathcal{A}$ be a fixed nonanticipating assignment procedure used under both distributions. Let $\mathbb P_P^{\mathcal A,T}$ and $\mathbb P_Q^{\mathcal A,T}$ denote the joint distributions of the historical data and the online observations through time $T$. Then
\begin{equation}
 \KL(\mathbb P_P^{\mathcal A,T},\mathbb P_Q^{\mathcal A,T})=\KL(P_{\mathrm{hist}},Q_{\mathrm{hist}})+\sum_{t=1}^T\Ebb_{P,\mathcal A}
 \big[\KL\{P_{Y(A_t)\mid\bX_t,A_t},
 Q_{Y(A_t)\mid\bX_t,A_t}\}\big].
 \label{eq:minimax-adaptive-KL}
\end{equation}
\end{lemma}

\begin{proof}[Proof of Lemma~\ref{lem:minimax-adaptive-KL}]
Let $\mathsf H_{t-1}$ be the collection of historical samples and all information retained by the procedure before time $t$, so that $\mathcal F_{t-1}=\sigma(\mathsf H_{t-1})$. For a realized pre-arrival history $\mathsf H_{t-1}=\mathcal H$, let $K_t^{\mathcal A}(a\mid \mathcal H,\bx)$ be the probability that the procedure assigns treatment condition $a$ after observing covariates $\bx$. 

The arriving unit is independent of the preceding history, and the assignment uses no current outcome information. Thus, under $P$, the conditional distribution of the next observed triple factors as
\begin{equation}
 \begin{aligned}
 P\{\bX_t\in d\bx,A_t=a,Y_t\in dy\mid\mathsf H_{t-1}=\mathcal H\}=P_X(d\bx)K_t^{\mathcal A}(a\mid \mathcal H,\bx)P_{Y(a)\mid\bX=\bx}(dy).
 \end{aligned}
 \label{eq:minimax-observation-factorization}
\end{equation}
The corresponding factorization under $Q$ replaces only the final factor in \eqref{eq:minimax-observation-factorization}, while the covariate distribution and conditional assignment probabilities are unchanged. Applying the conditional KL chain rule gives
\begin{equation}
 \begin{aligned}
 &\KL\{P_{(\bX_t,A_t,Y_t)\mid\mathsf H_{t-1}=\mathcal H},
        Q_{(\bX_t,A_t,Y_t)\mid\mathsf H_{t-1}=\mathcal H}\}\\
 =&\int\sum_{a=0}^1 K_t^{\mathcal A}(a\mid \mathcal H,\bx)
       \KL\{P_{Y(a)\mid\bX=\bx},Q_{Y(a)\mid\bX=\bx}\}\,P_X(d\bx).
 \end{aligned}
 \label{eq:minimax-conditional-KL}
\end{equation}
For each $\mathcal H$, the right-hand side of
\eqref{eq:minimax-conditional-KL} averages the outcome KL divergence over
the arriving covariates and the treatment condition assigned by the
procedure. Taking expectation under $P$ yields
\begin{equation}
\Ebb_{P,\mathcal A}
   \big[\KL\{P_{(\bX_t,A_t,Y_t)\mid\mathcal F_{t-1}},
               Q_{(\bX_t,A_t,Y_t)\mid\mathcal F_{t-1}}\}\big]=\Ebb_{P,\mathcal A}
   \big[\KL\{P_{Y(A_t)\mid\bX_t,A_t},
               Q_{Y(A_t)\mid\bX_t,A_t}\}\big].
 \label{eq:minimax-averaged-conditional-KL}
\end{equation}
Finally, the chain rule of KL divergence gives
\begin{equation}
 \begin{aligned}
 \KL(\mathbb P_P^{\mathcal A,T},\mathbb P_Q^{\mathcal A,T})
=\KL(P_{\mathrm{hist}},Q_{\mathrm{hist}})+\sum_{t=1}^T\Ebb_{P,\mathcal A}
   \big[\KL\{P_{(\bX_t,A_t,Y_t)\mid\mathcal F_{t-1}},
               Q_{(\bX_t,A_t,Y_t)\mid\mathcal F_{t-1}}\}\big].
 \end{aligned}
 \label{eq:minimax-KL-chain}
\end{equation}
Substituting \eqref{eq:minimax-averaged-conditional-KL} into
\eqref{eq:minimax-KL-chain} proves \eqref{eq:minimax-adaptive-KL}.
\end{proof}

The next lemma constructs models $\{P_\omega\}$, indexed by $\omega=(\omega^b,\omega^F)\in\{-1,1\}^{M_b}\times\{-1,1\}^{M_F}$, to which \eqref{eq:minimax-adaptive-KL} will be applied. All $P_\omega$ share the covariate distribution and the score $\ell_0$, and also $\lambda^\star$, $A_t^\star$, $\mathcal X_\pm$, and $g_0$. Changing $\omega_j^b$ changes only the conditional law of $Y(0)$ given $\bX=\bx$, and only for $\bx\in Q_{b,j}\subset\mathcal X_-$. Changing $\omega_j^F$ changes only that of $Y(1)$, and only for $\bx\in Q_{F,j}\subset\mathcal X_+$. Here $Q_{b,j}\subset\mathcal X_-$, $j=1,\ldots,M_b$, and $Q_{F,j}\subset\mathcal X_+$, $j=1,\ldots,M_F$, are pairwise disjoint covariate cubes supporting local perturbations of $b_0$ and $F_1$, respectively. By \eqref{eq:minimax-adaptive-KL}, the online information distinguishing the two neighbors therefore comes only from rounds with $\{\bX_t\in Q_{b,j},A_t=0\}$ or $\{\bX_t\in Q_{F,j},A_t=1\}$, respectively. The lemma quantifies the separation of the nuisance functions and the information supplied by online and historical observations. All constants below depend only on $s$, $d$, $\alpha$, and the model-class constants, and not on $r_b$, $r_F$, $T$, or $n_{\mathrm{hist}}$.

\begin{lemma}[Construction of nonparametric alternatives with a common oracle]
\label{lem:minimax-product-hypercube}
Fix $s>0$ and $L_0>0$, and consider the model class $\mathcal P^{\mathrm{np}}(s)$ with fixed nuisance-regularity constants $L\ge1$, $C_U,\delta_U>0$, $0<c_X\le1\le C_X<\infty$, and $1\le\sigma^2<\infty$. The common oracle-regularity constants are chosen from the fixed score constructed below.

There exist a score $\ell_0$, disjoint cubes $\mathcal X_-,\mathcal X_+\subset(0,1)^d$, and constants $g_0,\kappa_b,\kappa_F>0$ such that, for every sufficiently small $r_b,r_F>0$, one can construct a finite class
$$
 \big\{P_\omega:\omega=(\omega^b,\omega^F)\in\{-1,1\}^{M_b}\times\{-1,1\}^{M_F}\big\}\subseteq\mathcal P^{\mathrm{np}}(s)
$$
whose members share the same oracle assignment but differ in their control references and treated-outcome distributions. The location scales $h_b,h_F$, perturbation amplitudes $\Delta_b,\Delta_F$, and numbers of perturbations $M_b,M_F$ satisfy
\begin{equation}
 \begin{aligned}
 M_b&\asymp h_b^{-d},
 &\Delta_b&=\kappa_bh_b^s=L_0r_b,\\
 M_F&\asymp h_F^{-d},
 &\Delta_F&=\kappa_Fh_F^s=L_0r_F.
 \end{aligned}
 \label{eq:minimax-hard-scaling}
\end{equation}
The constructed class has the following properties.

\begin{enumerate}[(i)]
\item \textbf{Common oracle and separated regions.}
Under every $P_\omega$, $\bX$ is uniform on $[0,1]^d$ and $F_{1,\omega}\{b_{0,\omega}(\bx)\mid\bx\}=\ell_0(\bx)$. Consequently, the oracle threshold $\lambda^\star$ and assignment $A^\star(\bx)=\ind\{\ell_0(\bx)\le\lambda^\star\}$ are the same to all members of the class. Moreover,
$$
 \ell_0(\bx)\le\lambda^\star-g_0
 \quad(\bx\in\mathcal X_-),
 \qquad
 \ell_0(\bx)\ge\lambda^\star+g_0
 \quad(\bx\in\mathcal X_+).
$$

\item \textbf{Separated nuisance functions.}
There are pairwise disjoint cubes $Q_{b,j}\subset\mathcal X_-$, $j=1,\ldots,M_b$, of side length $h_b$, and pairwise disjoint cubes $Q_{F,j}\subset\mathcal X_+$, $j=1,\ldots,M_F$, of side length $h_F$. Each control-reference cube $Q_{b,j}$ contains a concentric inner cube $Q_{b,j}^{\circ}$ with half its side length, so that $|Q_{b,j}^{\circ}|=2^{-d}|Q_{b,j}|$. It supports a smooth function $\psi_{b,j}$ taking values in $[0,1]$ and equal to one on $Q_{b,j}^{\circ}$. The treatment-distribution cubes are constructed analogously, with $|Q_{F,j}^{\circ}|=2^{-d}|Q_{F,j}|$ and smooth functions $\psi_{F,j}$ satisfying the same properties. The control distribution is
\begin{equation}
 \begin{gathered}
 Y(0)\mid\bX=\bx\sim N\{b_{0,\omega}(\bx),1\},\\
 b_{0,\omega}(\bx)
 =\Delta_b\sum_{j=1}^{M_b}\omega_j^b\psi_{b,j}(\bx).
 \end{gathered}
 \label{eq:minimax-control-construction}
\end{equation}
For a common baseline treatment CDF $F_{1,\mathrm{base}}$ and a fixed $u_0\in(0,\delta_U)$,
\begin{equation}
 \begin{aligned}
 b_{0,\omega}(\bx)
 &=\omega_j^b\Delta_b,
 &&\bx\in Q_{b,j}^{\circ},\\
 F_{1,\omega}(u_0\mid\bx)-F_{1,\mathrm{base}}(u_0\mid\bx)
 &=\omega_j^F\Delta_F,
 &&\bx\in Q_{F,j}^{\circ}.
 \end{aligned}
 \label{eq:minimax-hard-separation}
\end{equation}
In particular, $b_{0,\omega}=0$ on $\mathcal X_+$.

\item \textbf{Localized information.}
For $j=1,\ldots,M_b$, let $\omega^{b,j}$ be obtained from $\omega$ by reversing only $\omega_j^b$. For $j=1,\ldots,M_F$, define $\omega^{F,j}$ by reversing only $\omega_j^F$. Reversing a $b$-coordinate changes only the control distribution on $Q_{b,j}$; reversing an $F$-coordinate changes only the treatment distribution on $Q_{F,j}$. For a constant $C<\infty$,
\begin{equation}
 \begin{aligned}
 \KL\{P_{\omega,Y(0)\mid\bx},
          P_{\omega^{b,j},Y(0)\mid\bx}\}
 &\le C\Delta_b^2\ind\{\bx\in Q_{b,j}\},\\
 \KL\{P_{\omega,Y(1)\mid\bx},
          P_{\omega^{F,j},Y(1)\mid\bx}\}
 &\le C\Delta_F^2\ind\{\bx\in Q_{F,j}\}.
 \end{aligned}
 \label{eq:minimax-local-information}
\end{equation}

\item \textbf{Historical information.}
Under the randomized historical design in Section~\ref{subsec:causal-setting},
\begin{equation}
 \begin{aligned}
 \sum_{j=1}^{M_b}
 \KL(P_{\omega,\mathrm{hist}},
       P_{\omega^{b,j},\mathrm{hist}})
 &\le C\Delta_b^2 n_{0,\mathrm{hist}},\\
 \sum_{j=1}^{M_F}
 \KL(P_{\omega,\mathrm{hist}},
       P_{\omega^{F,j},\mathrm{hist}})
 &\le C\Delta_F^2 n_{1,\mathrm{hist}},
 \end{aligned}
 \label{eq:minimax-history-information}
\end{equation}
where $n_{0,\mathrm{hist}}$ and $n_{1,\mathrm{hist}}$ are the expected numbers of historical control and treatment observations, respectively.
\end{enumerate}
\end{lemma}

\begin{proof}[Proof of Lemma~\ref{lem:minimax-product-hypercube}]
We construct the class by perturbing the control reference in an oracle-treated region and the treated-outcome distribution in an oracle-control region, respectively, without changing the local null risk. The construction uses a flat treated CDF near the control reference in the first region and redistributes probability mass above the reference in the second.

\noindent\textit{Step 1: fix the common oracle.}
Let $\mathcal U$ be the bounded open convex neighborhood in Definition~\ref{def:holder-class}. Choose a fixed $\beta>0$ sufficiently small that
\begin{equation}
 \beta\sup_{\bx\in\mathcal U}|x_1-1/4|
 \le\tfrac12\min\{\alpha,1-\alpha\},
 \label{eq:minimax-score-fixed-range}
\end{equation}
and define $\bX\sim\operatorname{Unif}([0,1]^d)$, and $\ell_0(\bx)=\alpha+\beta(x_1-1/4)$ for $\bx\in\mathcal U$. Then $0<\ell_0<1$ on $\mathcal U$, and $\ell_0(\bX)\sim\operatorname{Unif}[\alpha-\beta/4,\alpha+3\beta/4]$. Writing $H_0$ for its distribution function, direct calculation gives
$$
H_0(\alpha)=\tfrac14,\qquad
 \lambda^\star=\alpha+\tfrac{\beta}{4},\qquad
 \lambda_+=\alpha+\tfrac{\beta}{2},$$
$$
 \Ebb\{\ell_0(\bX)\mid\ell_0(\bX)\le\lambda^\star\}=\alpha,
 \qquad
 \Ebb\{\ell_0(\bX)\mid\ell_0(\bX)\le\lambda_+\}
 =\alpha+\tfrac{\beta}{8}.$$
The score density equals $1/\beta$ on $[\alpha,\lambda_+]$, so Assumption~\ref{ass:oracle-score} holds. The oracle assigns treatment exactly when $x_1\le1/2$. Fix the two cubes
\begin{equation*}
 \mathcal X_-=[1/8,3/8]^d,\qquad
 \mathcal X_+=[5/8,7/8]\times[1/8,3/8]^{d-1}.
\end{equation*}
They do not depend on $\beta$, and
\begin{equation*}
 \min\big\{
 \inf_{\bx\in\mathcal X_-}\{\lambda^\star-\ell_0(\bx)\},
 \inf_{\bx\in\mathcal X_+}\{\ell_0(\bx)-\lambda^\star\}
 \big\}=\frac{\beta}{8}=:g_0>0.
\end{equation*}
The choice of $\beta$ will also satisfy the smoothness bounds in Step~4; it is fixed before the target radii are chosen.

\noindent\textit{Step 2: construct the local alternatives.}
Fix a smooth function $\psi$ taking values in $[0,1]$, supported in $(-1/2,1/2)^d$, and equal to one on $[-1/4,1/4]^d$. Choose fixed constants $\kappa_b,\kappa_F>0$ and set
\begin{equation*}
 \begin{aligned}
 h_b&=(L_0r_b/\kappa_b)^{1/s},
 &\Delta_b&=\kappa_bh_b^s=L_0r_b,\\
 h_F&=(L_0r_F/\kappa_F)^{1/s},
 &\Delta_F&=\kappa_Fh_F^s=L_0r_F.
 \end{aligned}
\end{equation*}
For sufficiently small radii, pack cubes of side $h_b$ in $\mathcal X_-$ with centers spaced $2h_b$ apart in each coordinate. Likewise, pack cubes of side $h_F$ in $\mathcal X_+$ with centers spaced $2h_F$ apart. Their numbers satisfy $M_b\asymp h_b^{-d}$ and $M_F\asymp h_F^{-d}$, where $\asymp$ denotes bounds above and below by fixed positive multiples. Let $\bz_{b,j}$ and $\bz_{F,j}$ denote the centers of $Q_{b,j}$ and $Q_{F,j}$, respectively, and define
\begin{equation*}
 \begin{aligned}
 \psi_{b,j}(\bx)&=\psi\{(\bx-\bz_{b,j})/h_b\},
 &Q_{b,j}^{\circ}&=\bz_{b,j}+h_b[-1/4,1/4]^d,\\
 \psi_{F,j}(\bx)&=\psi\{(\bx-\bz_{F,j})/h_F\},
 &Q_{F,j}^{\circ}&=\bz_{F,j}+h_F[-1/4,1/4]^d.
 \end{aligned}
\end{equation*}
Each bump equals one on its corresponding inner cube, and $|Q_{b,j}^{\circ}|=2^{-d}|Q_{b,j}|$ and $|Q_{F,j}^{\circ}|=2^{-d}|Q_{F,j}|$. Define the control distribution by \eqref{eq:minimax-control-construction}. Disjoint supports imply
\begin{equation}
 |b_{0,\omega}|\le\Delta_b,\qquad
 b_{0,\omega}(\bx)=\omega_j^b\Delta_b
 \ \ (\bx\in Q_{b,j}^{\circ}),\qquad
 b_{0,\omega}=0\ \ \text{outside }\mathcal X_-.
 \label{eq:minimax-control-local-values}
\end{equation}

To construct the treatment distributions, fix $u_0\in(0,\delta_U)$ and smooth probability densities $p_-,p_+,p_1,p_2$ with compact supports contained, respectively, in
$$
 (-\infty,-3\delta_U),\quad
 (3\delta_U,\infty),\quad
 (0,u_0),\quad
 (u_0,\delta_U).
$$
Also fix a smooth $\rho:\mathcal U\to[0,1]$ with bounded derivatives, equal to zero near $\mathcal X_-$ and one near $\mathcal X_+$. These functions can be chosen independently of $\beta$. For a fixed $0<\theta\le1$, define
\begin{equation*}
 \begin{aligned}
 p_{1,\mathrm{base}}(y\mid\bx)
 &=\ell_0(\bx)p_-(y)+\{1-\ell_0(\bx)\}[
 \{1-\theta\rho(\bx)\}p_+(y)
 +\frac{\theta\rho(\bx)}2\{p_1(y)+p_2(y)\}],
 \end{aligned}
\end{equation*}
and
\begin{equation}
 p_{1,\omega}(y\mid\bx)
 =p_{1,\mathrm{base}}(y\mid\bx)
 +\Delta_F\sum_{j=1}^{M_F}
 \omega_j^F\psi_{F,j}(\bx)\{p_1(y)-p_2(y)\}.
 \label{eq:minimax-treatment-perturbation}
\end{equation}
Thus each $F$-coordinate transfers probability mass between two positive intervals, preserving the CDF at zero while changing it at $u_0$. Let
\begin{equation*}
 m_+:=\frac12\inf_{\bx\in\mathcal X_+}\{1-\ell_0(\bx)\}>0.
\end{equation*}
After fixing $\theta,\kappa_b,\kappa_F$ as in Step~4, take the radii sufficiently small that
\begin{equation*}
 h_b,h_F\le1,\qquad
 \Delta_b\le\delta_U/2,\qquad
 \Delta_F\le\theta m_+/2\le\theta/4.
\end{equation*}
On $\mathcal X_+$, the baseline coefficients of $p_1$ and $p_2$ are at least $\theta m_+$. Since at most one bump is nonzero at any point, their perturbed coefficients remain at least $\theta m_+/2$. Hence
\begin{equation*}
 p_{1,\omega}(y\mid\bx)
 \ge\frac{\theta m_+}{2}\{p_1(y)+p_2(y)\},
 \qquad \bx\in\mathcal X_+.
\end{equation*}
Outside the perturbation cubes, the density is unchanged. Moreover, $\int(p_1-p_2)=0$, so the perturbed density integrates to one. It is therefore a probability density for every $\bx\in\mathcal U$.

Take the potential outcomes to be conditionally independent given $\bX$. With $\varphi$ denoting the standard normal density, the joint density is
\begin{equation*}
 p_\omega(\bx,y_0,y_1)
 =\begin{cases}
 \varphi\{y_0-b_{0,\omega}(\bx)\}
 p_{1,\omega}(y_1\mid\bx),&\bx\in[0,1]^d,\\
 0,&\text{otherwise}.
 \end{cases}
\end{equation*}
This defines a probability distribution $P_\omega$ for each sign vector and hence the required finite class. Distinct sign vectors give distinct distributions by the separation verified next.

\noindent\textit{Step 3: verify separation and preservation of the oracle.} Let $P_\iota(y)=\int_{-\infty}^y p_\iota(v)\,dv$ for $\iota\in\{-,+,1,2\}$. Integrating \eqref{eq:minimax-treatment-perturbation} gives
$$
 F_{1,\omega}(y\mid\bx)-F_{1,\mathrm{base}}(y\mid\bx)
 =\Delta_F\sum_j\omega_j^F\psi_{F,j}(\bx)
 \{P_1(y)-P_2(y)\}.
$$
Because $P_1(u_0)=1$ and $P_2(u_0)=0$,
\begin{equation*}
 F_{1,\omega}(u_0\mid\bx)-F_{1,\mathrm{base}}(u_0\mid\bx)
 =\Delta_F\sum_j\omega_j^F\psi_{F,j}(\bx).
\end{equation*}
Together with \eqref{eq:minimax-control-local-values}, this proves \eqref{eq:minimax-hard-separation}.

For $\bx\in\mathcal X_-$, we have $\rho(\bx)=0$ and no $F$-perturbation. Thus the treated density is $\ell_0(\bx)p_-+\{1-\ell_0(\bx)\}p_+$, and
\begin{equation}
 F_{1,\omega}(y\mid\bx)=\ell_0(\bx),
 \qquad |y|\le2\delta_U,\quad \bx\in\mathcal X_-.
 \label{eq:minimax-treatment-cdf-flat-region}
\end{equation}
Since $|b_{0,\omega}(\bx)|\le\delta_U/2$, evaluating this CDF at the perturbed control reference leaves its value unchanged. For $\bx\notin\mathcal X_-$, the control reference is zero. Only $p_-$ places mass below zero, and its coefficient is $\ell_0(\bx)$; hence $F_{1,\omega}(0\mid\bx)=\ell_0(\bx)$. Consequently,
\begin{equation}
 F_{1,\omega}\{b_{0,\omega}(\bx)\mid\bx\}
 =\ell_0(\bx),\qquad \bx\in[0,1]^d.
 \label{eq:minimax-common-score-identity}
\end{equation}
The oracle and its positive gaps are therefore the same over all members of the constructed class.

\noindent\textit{Step 4: verify the fixed regularity bounds.}
Put $r=\lceil s\rceil-1$ and $\gamma=s-r\in(0,1]$, and write $g_{F,\omega}=\Delta_F\sum_{j=1}^{M_F}\omega_j^F\psi_{F,j}$ for the treatment-CDF perturbation. Disjoint supports and rescaling give separate bounds for the control reference and the treatment-CDF perturbation:
\begin{equation}
 \begin{aligned}
 \|D^\nu b_{0,\omega}\|_{\infty,\mathcal U}
 &\le C_\nu\Delta_bh_b^{-|\nu|}
 =C_\nu\kappa_bh_b^{s-|\nu|},\\
 \|D^\nu g_{F,\omega}\|_{\infty,\mathcal U}
 &\le C_\nu\Delta_Fh_F^{-|\nu|}
 =C_\nu\kappa_Fh_F^{s-|\nu|},
 \end{aligned}
 \label{eq:minimax-perturbation-derivatives}
\end{equation}
where $C_\nu=\|D^\nu\psi\|_\infty$. For $|\nu|\le r$, these bounds are at most $C_\nu\kappa_b$ and $C_\nu\kappa_F$, respectively. For $|\nu|=r$, combining the order-$r$ bound with the order-$(r+1)$ bound and the mean-value theorem yields
\begin{equation}
 \begin{aligned}
 |D^\nu b_{0,\omega}(\bx)-D^\nu b_{0,\omega}(\bz)|
 &\le C\kappa_bh_b^\gamma
       \min\{1,\|\bx-\bz\|/h_b\} \le C\kappa_b\|\bx-\bz\|^\gamma,\\
 |D^\nu g_{F,\omega}(\bx)-D^\nu g_{F,\omega}(\bz)|
 &\le C\kappa_Fh_F^\gamma
       \min\{1,\|\bx-\bz\|/h_F\} \le C\kappa_F\|\bx-\bz\|^\gamma,
 \end{aligned}
 \label{eq:minimax-perturbation-holder-increments}
\end{equation}
for all $\bx,\bz\in\mathcal U$. These bounds are uniform over the signs and scales and apply also to points in different cubes.

The baseline CDF can be written as
\begin{equation*}
 \begin{aligned}
 F_{1,\mathrm{base}}(y\mid\bx)
 &=\alpha P_-(y)+(1-\alpha)P_+(y)+\beta(x_1-1/4)\{P_-(y)-P_+(y)\}\\
 &\quad+\theta\{1-\ell_0(\bx)\}\rho(\bx)
 \bigg\{\frac{P_1(y)+P_2(y)}2-P_+(y)\bigg\}.
 \end{aligned}
\end{equation*}
The first line is constant in $\bx$, and the CDF factors are bounded uniformly in $y$. Smoothness of $\rho$, boundedness of $\mathcal U$, and \eqref{eq:minimax-score-fixed-range} therefore give
\begin{equation*}
 \begin{aligned}
 &\max_{1\le|\nu|\le r}\sup_y
       \|D^\nu F_{1,\mathrm{base}}(y\mid\cdot)\|_{\infty,\mathcal U}+\max_{|\nu|=r}\sup_y\sup_{\substack{\bx,\bz\in\mathcal U\\\bx\ne\bz}}
 \frac{|D^\nu F_{1,\mathrm{base}}(y\mid\bx)
       -D^\nu F_{1,\mathrm{base}}(y\mid\bz)|}
      {\|\bx-\bz\|^\gamma}\\
 \le& C_{\mathrm{reg}}(\beta+\theta).
 \end{aligned}
\end{equation*}
The first term is omitted when $r=0$. Choose a common constant $C_{\mathrm{reg}}$ controlling this bound, the bump derivative bounds through order $r+1$, and the bump H\"older bounds.

We may consequently choose these positive constants, before choosing the target radii, so that
\eqref{eq:minimax-score-fixed-range} holds and
$$
 0<\theta\le1,\quad
 C_{\mathrm{reg}}(\beta+\theta)\le L/2,\quad
 C_{\mathrm{reg}}\kappa_F\le L/2,\quad
 C_{\mathrm{reg}}\kappa_b\le L,\quad
 \theta\max_{i=1,2}\|p_i\|_\infty\le C_U/2.
$$
Since $F_{1,\omega}(y\mid\cdot)=F_{1,\mathrm{base}}(y\mid\cdot)+g_{F,\omega}(\cdot)\{P_1(y)-P_2(y)\}$ and $ |P_1-P_2|\le1$, the preceding bounds control its derivatives and H\"older seminorm by $L$. Its zeroth-order bound follows from $0\le F_{1,\omega}(y\mid\bx)\le1\le L$ for $y\in\mathbb R$ and $\bx\in\mathcal U$. Likewise, \eqref{eq:minimax-perturbation-derivatives} and \eqref{eq:minimax-perturbation-holder-increments} imply $b_{0,\omega}\in\mathcal H^s([0,1]^d,L)$. Thus $F_{1,\omega}(y\mid\cdot)$ belongs to the same H\"older class uniformly in $y$, and both functions have the required extensions to $\mathcal U$.

For the local-density condition, if $\bx\in\mathcal X_-$ and $|u|\le\delta_U$, then $|u+b_{0,\omega}(\bx)|\le3\delta_U/2$, where the treated density vanishes. Otherwise, $b_{0,\omega}(\bx)=0$, and only $p_1,p_2$ can contribute on $[-\delta_U,\delta_U]$. Their supports are disjoint, and their coefficients are at most $\theta/2+\Delta_F\le3\theta/4$. Hence
\begin{equation*}
 f_{U\mid X}(u\mid\bx)
 \le\frac{3\theta}{4}\max_{i=1,2}\|p_i\|_\infty
 \le\frac{3C_U}{8},
 \qquad |u|\le\delta_U.
\end{equation*}
Finally, the covariate density is one, satisfying the specified bounds, and the control distribution obeys
\begin{equation*}
 \begin{gathered}
 \Ebb_{P_\omega}\{Y(0)\mid\bX\}=b_{0,\omega}(\bX),\\
 \Ebb_{P_\omega}
[e^{v\{Y(0)-b_{0,\omega}(\bX)\}}\mid\bX]
 =e^{v^2/2}\le e^{\sigma^2v^2/2},
 \qquad v\in\mathbb R.
 \end{gathered}
\end{equation*}
Together with Step~1 and \eqref{eq:minimax-common-score-identity}, these facts verify membership in $\mathcal P^{\mathrm{np}}(s)$ with the stated fixed nuisance bounds and the common oracle constants.

\noindent\textit{Step 5: bound the information in observed outcomes.}
The control law depends only on $\omega^b$, whereas the treatment law depends only on $\omega^F$. For neighboring control alternatives, the Gaussian KL formula gives
\begin{equation*}
 \KL\{P_{\omega,Y(0)\mid\bx},
       P_{\omega^{b,j},Y(0)\mid\bx}\}
 =2\Delta_b^2\psi_{b,j}(\bx)^2.
\end{equation*}
For neighboring treatment alternatives,
\begin{equation*}
 p_{1,\omega}(y\mid\bx)-p_{1,\omega^{F,j}}(y\mid\bx)
 =2\omega_j^F\Delta_F\psi_{F,j}(\bx)\{p_1(y)-p_2(y)\}.
\end{equation*}
On $Q_{F,j}$, both densities dominate $(\theta m_+/2)(p_1+p_2)$. Bounding KL divergence by $\chi^2$-divergence yields
\begin{equation*}
 \begin{aligned}
 &\KL\{P_{\omega,Y(1)\mid\bx},
          P_{\omega^{F,j},Y(1)\mid\bx}\}\le
 \frac{8\Delta_F^2\psi_{F,j}(\bx)^2}{\theta m_+}
 \int_{\mathbb R}\frac{(p_1-p_2)^2}{p_1+p_2}\,dy
 \le\frac{16}{\theta m_+}
       \Delta_F^2\psi_{F,j}(\bx)^2.
 \end{aligned}
\end{equation*}
The fraction is defined as zero where $p_1+p_2=0$. Outside $Q_{F,j}$ the densities coincide. Since $\psi_{b,j}^2\le\ind\{\bx\in Q_{b,j}\}$ and $\psi_{F,j}^2\le\ind\{\bx\in Q_{F,j}\}$, these bounds establish \eqref{eq:minimax-local-information} with $C_{\mathrm{KL}}=\max\{2,16/(\theta m_+)\}$.

For the historical sample, the covariate and assignment distributions are identical under neighboring models. Only control outcomes distinguish models differing in a $b$-coordinate, whereas only treatment outcomes distinguish models differing in an $F$-coordinate. Independence of historical units therefore gives
\begin{equation*}
 \begin{aligned}
 &\KL(P_{\omega,\mathrm{hist}},
        P_{\omega^{b,j},\mathrm{hist}})=n_{0,\mathrm{hist}}
 \int_{[0,1]^d}
 \KL\{P_{\omega,Y(0)\mid\bx},
       P_{\omega^{b,j},Y(0)\mid\bx}\}d\bx\le C_{\mathrm{KL}}\Delta_b^2
       n_{0,\mathrm{hist}}|Q_{b,j}|,\\
 &\KL(P_{\omega,\mathrm{hist}},
        P_{\omega^{F,j},\mathrm{hist}})=n_{1,\mathrm{hist}}
 \int_{[0,1]^d}
 \KL\{P_{\omega,Y(1)\mid\bx},
       P_{\omega^{F,j},Y(1)\mid\bx}\}d\bx\le C_{\mathrm{KL}}\Delta_F^2
       n_{1,\mathrm{hist}}|Q_{F,j}|.
 \end{aligned}
\end{equation*}
Summing the first bound over $j=1,\ldots,M_b$ and the second over $j=1,\ldots,M_F$, and using $\sum_j|Q_{b,j}|\le1$ and $\sum_j|Q_{F,j}|\le1$, proves \eqref{eq:minimax-history-information}. This completes the construction.
\end{proof}

We now prove Theorem~\ref{thm:minimax-nonparametric-fixed}. The argument consists of three parts: the learning constraint gives accurate tests of the local perturbation signs; accurate tests require a positive amount of KL information; and \eqref{eq:minimax-local-information} and \eqref{eq:minimax-count-to-regret} convert this information into regret.

\begin{proof}[Proof of Theorem~\ref{thm:minimax-nonparametric-fixed}]
If $\mathfrak A_T(\mathcal P^{\mathrm{np}}(s);r_{b,T},r_{F,T},C_0)$ is empty, the infimum in \eqref{eq:minimax-regret-definition} equals $+\infty$ and there is nothing to prove. Otherwise, fix $\mathcal A$ in this class and a constant $L_0\ge4\sqrt{C_0}$. For $r_{b,T}$ and $r_{F,T}$ sufficiently small, let $\{P_\omega\}\subset\mathcal P^{\mathrm{np}}(s)$ be the family of Lemma~\ref{lem:minimax-product-hypercube} with $r_b=r_{b,T}$ and $r_F=r_{F,T}$. Write $\Ebb_{\omega,\mathcal A}$ and $\Pbb_{\omega,\mathcal A}$ for expectation and probability under $P_\omega$ and $\mathcal A$. Throughout, $k\in\{b,F\}$, with $a=0$ when $k=b$ and $a=1$ when $k=F$. All constants below depend only on $C_0$ and the model-class constants.

Let $\tau$ be uniform on $\{1,\ldots,T\}$ and independent of the data and of the randomization of $\mathcal A$. Since $\mathcal A$ satisfies \eqref{eq:minimax-uniform-nuisance-constraint} under every $P_\omega$, 
\begin{equation}
 \Ebb_{\omega,\mathcal A}[(\delta_{k,\tau}^{\infty})^2]
 =\frac1T\sum_{t=1}^T\Ebb_{\omega,\mathcal A}[(\delta_{k,t}^{\infty})^2]
 \le C_0r_{k,T}^2 .
 \label{eq:minimax-random-time-risk}
\end{equation}
Define the cell-average sign estimators
\begin{equation}
 \widehat\omega_j^b
 =\operatorname{sign}\bigg\{\frac1{|Q_{b,j}^{\circ}|}
   \int_{Q_{b,j}^{\circ}}\widehat b_{0,\tau}\bigg\},
 \qquad
 \widehat\omega_j^F
 =\operatorname{sign}\bigg[\frac1{|Q_{F,j}^{\circ}|}
   \int_{Q_{F,j}^{\circ}}\{\widehat F_{1,\tau}(u_0\mid\cdot)
       -F_{1,\mathrm{base}}(u_0\mid\cdot)\}\bigg],
 \label{eq:minimax-sign-tests}
\end{equation}
with a fixed convention at zero. By \eqref{eq:minimax-hard-separation}, $b_{0,\omega}=\omega_j^b\Delta_b$ on $Q_{b,j}^\circ$ and $F_{1,\omega}(u_0\mid\cdot)-F_{1,\mathrm{base}}(u_0\mid\cdot)=\omega_j^F\Delta_F$ on $Q_{F,j}^\circ$. Moreover, $b_{0,\omega}=0$ on $\mathcal X_+$ and $0<u_0<\delta_U$, so $u_0$ lies in the band that defines $\delta_{F,\tau}^\infty$. Since $P_X$ is uniform and $|Q_{k,j}^\circ|>0$, averages over $Q_{k,j}^\circ$ are controlled by $P_X$-essential suprema. Hence
\begin{equation}
 \bigg|\frac1{|Q_{b,j}^{\circ}|}\int_{Q_{b,j}^{\circ}}\widehat b_{0,\tau}
   -\omega_j^b\Delta_b\bigg|\le\delta_{b,\tau}^{\infty},
 \qquad
 \bigg|\frac1{|Q_{F,j}^{\circ}|}\int_{Q_{F,j}^{\circ}}
   \{\widehat F_{1,\tau}(u_0\mid\cdot)-F_{1,\mathrm{base}}(u_0\mid\cdot)\}
   -\omega_j^F\Delta_F\bigg|\le\delta_{F,\tau}^{\infty}.
 \label{eq:minimax-cell-average-errors}
\end{equation}
A wrong sign places the corresponding average at a distance of at least $\Delta_k$ from $\omega_j^k\Delta_k$. Therefore \eqref{eq:minimax-cell-average-errors}, Markov's inequality, \eqref{eq:minimax-random-time-risk}, and $\Delta_k=L_0r_{k,T}$ give
\begin{equation}
 \Pbb_{\omega,\mathcal A}(\widehat\omega_j^k\ne\omega_j^k)
 \le\Pbb_{\omega,\mathcal A}(\delta_{k,\tau}^{\infty}\ge\Delta_k)
 \le\frac{C_0r_{k,T}^2}{\Delta_k^2}=\frac{C_0}{L_0^2}\le\frac1{16}.
 \label{eq:minimax-sign-error-probability}
\end{equation}

Let $\mathbb Q_{\omega,\tau}$ be the joint law of $\tau$ and $\mathsf H_{\tau-1}$, in the notation of Lemma~\ref{lem:minimax-adaptive-KL}. The estimators in \eqref{eq:minimax-sign-tests} are functions of $(\tau,\mathsf H_{\tau-1})$. For the neighboring pair $(\omega,\omega^{k,j})$, the event $\{\widehat\omega_j^k=\omega_j^k\}$ is an error under $\omega^{k,j}$. Thus \eqref{eq:minimax-sign-error-probability} yields
\begin{equation*}
 \frac18
 \ge\mathbb Q_{\omega,\tau}(\widehat\omega_j^k\ne\omega_j^k)
   +\mathbb Q_{\omega^{k,j},\tau}(\widehat\omega_j^k=\omega_j^k)
 \ge1-\operatorname{TV}(\mathbb Q_{\omega,\tau},\mathbb Q_{\omega^{k,j},\tau}),
\end{equation*}
and Pinsker's inequality gives
\begin{equation}
 \KL(\mathbb Q_{\omega,\tau},\mathbb Q_{\omega^{k,j},\tau})
 \ge2\operatorname{TV}(\mathbb Q_{\omega,\tau},\mathbb Q_{\omega^{k,j},\tau})^2
 \ge2(7/8)^2\ge1 .
 \label{eq:minimax-testing-information}
\end{equation}
Define
\begin{equation}
 N_{b,j}(T)=\sum_{v=1}^T\ind\{\bX_v\in Q_{b,j},A_v=0\},
 \qquad
 N_{F,j}(T)=\sum_{v=1}^T\ind\{\bX_v\in Q_{F,j},A_v=1\}.
 \label{eq:minimax-cell-observation-counts}
\end{equation}
The law of $\tau$ is the same under both models. Applying the chain rule and then \eqref{eq:minimax-adaptive-KL} with horizon $t-1$, and interchanging the two sums, we obtain
\begin{equation*}
\begin{aligned}
     &\KL(\mathbb Q_{\omega,\tau},\mathbb Q_{\omega^{k,j},\tau})
 =\frac1T\sum_{t=1}^T\KL(\mathbb P_{P_\omega}^{\mathcal A,t-1},
     \mathbb P_{P_{\omega^{k,j}}}^{\mathcal A,t-1}) \\
 =&\KL(P_{\omega,\mathrm{hist}},P_{\omega^{k,j},\mathrm{hist}})
 +\sum_{v=1}^{T-1}\frac{T-v}T\,\Ebb_{\omega,\mathcal A}
   \big[\KL\{P_{\omega,Y(A_v)\mid\bX_v,A_v},
             P_{\omega^{k,j},Y(A_v)\mid\bX_v,A_v}\}\big].
\end{aligned}
\end{equation*}
By Lemma~\ref{lem:minimax-product-hypercube}(iii), reversing $\omega_j^k$ changes only the law of $Y(a)$ given $\bX\in Q_{k,j}$. Hence, by \eqref{eq:minimax-local-information}, the $v$th expectation is at most $C\Delta_k^2\,\Pbb_{\omega,\mathcal A}(\bX_v\in Q_{k,j},A_v=a)$. Using $(T-v)/T\le1$,
\begin{equation}
 \KL(\mathbb Q_{\omega,\tau},\mathbb Q_{\omega^{k,j},\tau})
 \le\KL(P_{\omega,\mathrm{hist}},P_{\omega^{k,j},\mathrm{hist}})
 +C\Delta_k^2\,\Ebb_{\omega,\mathcal A}N_{k,j}(T).
 \label{eq:minimax-testing-count-upper}
\end{equation}
Combining \eqref{eq:minimax-testing-information} with \eqref{eq:minimax-testing-count-upper}, summing over $j$, and applying \eqref{eq:minimax-history-information}, we obtain
\begin{equation*}
 M_b\le C\Delta_b^2\Big\{n_{0,\mathrm{hist}}
   +\Ebb_{\omega,\mathcal A}\sum_{j=1}^{M_b}N_{b,j}(T)\Big\},
\end{equation*}
\begin{equation*}
 M_F\le C\Delta_F^2\Big\{n_{1,\mathrm{hist}}
   +\Ebb_{\omega,\mathcal A}\sum_{j=1}^{M_F}N_{F,j}(T)\Big\},
\end{equation*}
that is,
\begin{equation}
\begin{split}
     &\Ebb_{\omega,\mathcal A}\sum_{j=1}^{M_b}N_{b,j}(T)
 \ge cM_b\Delta_b^{-2}-n_{0,\mathrm{hist}}, \\
 &\Ebb_{\omega,\mathcal A}\sum_{j=1}^{M_F}N_{F,j}(T)
 \ge cM_F\Delta_F^{-2}-n_{1,\mathrm{hist}},
 \qquad c=1/C.
\end{split}
 \label{eq:minimax-required-cell-counts}
\end{equation}
By \eqref{eq:minimax-hard-scaling},
\begin{equation*}
 M_b\Delta_b^{-2}\asymp h_b^{-d}\Delta_b^{-2}
 =\kappa_b^{d/s}\Delta_b^{-(2s+d)/s}
 =\kappa_b^{d/s}(L_0r_{b,T})^{-(2s+d)/s}
 \asymp r_{b,T}^{-(2s+d)/s},
\end{equation*}
\begin{equation*}
 M_F\Delta_F^{-2}\asymp h_F^{-d}\Delta_F^{-2}
 =\kappa_F^{d/s}\Delta_F^{-(2s+d)/s}
 =\kappa_F^{d/s}(L_0r_{F,T})^{-(2s+d)/s}
 \asymp r_{F,T}^{-(2s+d)/s}.
\end{equation*}
The $b$-cells are disjoint subsets of $\mathcal X_-$, and the $F$-cells are disjoint subsets of $\mathcal X_+$. By \eqref{eq:minimax-informative-counts} and \eqref{eq:minimax-cell-observation-counts},
\begin{equation}
 \sum_{j=1}^{M_b}N_{b,j}(T)\le N_{0,-}(T),
 \qquad
 \sum_{j=1}^{M_F}N_{F,j}(T)\le N_{1,+}(T).
 \label{eq:minimax-cell-region-counts}
\end{equation}
Combining \eqref{eq:minimax-required-cell-counts}--\eqref{eq:minimax-cell-region-counts} with the nonnegativity of the counts, and choosing $c$ and $C_{\mathrm{hist}}$ common to $k=b$ and $k=F$, gives
\begin{equation}
 \Ebb_{\omega,\mathcal A}N_{0,-}(T)
 \ge c\{r_{b,T}^{-(2s+d)/s}-C_{\mathrm{hist}}n_{0,\mathrm{hist}}\}_+,
 \qquad
 \Ebb_{\omega,\mathcal A}N_{1,+}(T)
 \ge c\{r_{F,T}^{-(2s+d)/s}-C_{\mathrm{hist}}n_{1,\mathrm{hist}}\}_+ .
 \label{eq:minimax-required-region-counts}
\end{equation}
Both bounds hold under the same $P_\omega$. Substituting them into \eqref{eq:minimax-count-to-regret} and absorbing $g_0$ into $c$, we have
\begin{equation}
 \Regret_T(\mathcal A,P_\omega)
 \ge c\big[\{r_{b,T}^{-(2s+d)/s}-C_{\mathrm{hist}}n_{0,\mathrm{hist}}\}_+
   +\{r_{F,T}^{-(2s+d)/s}-C_{\mathrm{hist}}n_{1,\mathrm{hist}}\}_+\big].
 \label{eq:minimax-procedure-regret-lower}
\end{equation}
Since $P_\omega\in\mathcal P^{\mathrm{np}}(s)$ and $\mathcal A$ was an arbitrary admissible procedure, taking the supremum over $\mathcal P^{\mathrm{np}}(s)$ and then the infimum over $\mathfrak A_T(\mathcal P^{\mathrm{np}}(s);r_{b,T},r_{F,T},C_0)$ bounds \eqref{eq:minimax-regret-definition} from below by the right-hand side of \eqref{eq:minimax-procedure-regret-lower}. Taking the minimum with $T$ weakens this bound, which proves \eqref{eq:minimax-nonparametric-fixed}. Finally, $N_{0,-}(T)+N_{1,+}(T)\le T$ on every sample path. Hence, if the right-hand sides of \eqref{eq:minimax-required-region-counts} sum to more than $T$, the admissible class is empty, which is the case treated at the beginning of the proof.
\end{proof}

We then provide the proof of Corollary~\ref{cor:minimax-nonparametric-rate}.

\begin{proof}[Proof of Corollary~\ref{cor:minimax-nonparametric-rate}]
By the definition of $\rho_T$ in \eqref{eq:minimax-common-radius},
\begin{equation}
 \rho_T^{-(2s+d)/s}
 =\frac{T^{(2s+d)/(4s+d)}}{\log T}\longrightarrow\infty,
 \qquad \rho_T^{-(2s+d)/s}=o(T).
 \label{eq:minimax-radius-inverse-order}
\end{equation}
Since $n_{\mathrm{hist}}=O(1)$, both historical corrections in \eqref{eq:minimax-nonparametric-fixed} are bounded. Equation \eqref{eq:minimax-radius-inverse-order} shows that these corrections are eventually smaller than a fixed fraction of $\rho_T^{-(2s+d)/s}$, while the minimum with $T$ is eventually inactive. Consequently, for all sufficiently large $T$,
\begin{equation}
 \mathfrak R_T(\mathcal P^{\mathrm{np}}(s);\rho_T,\rho_T)
 \ge c\frac{T^{(2s+d)/(4s+d)}}{\log T}.
 \label{eq:minimax-corollary-lower}
\end{equation}

The bound \eqref{eq:minimax-paits-uniform-nuisance} implies both inequalities in \eqref{eq:minimax-uniform-nuisance-constraint} when $C_0\ge C_{\mathrm{nuis}}$. Theorem~\ref{thm:fdr-one-pass} supplies asymptotic FDR control for each model in the class. Hence $\mathcal A_{\mathrm{PAITS}}$ belongs to the admissible class for all sufficiently large $T$. By \eqref{eq:minimax-regret-definition},
\begin{equation}
 \mathfrak R_T(\mathcal P^{\mathrm{np}}(s);\rho_T,\rho_T)
 \le\sup_{P\in\mathcal P^{\mathrm{np}}(s)}
       \Regret_T(\mathcal A_{\mathrm{PAITS}},P).
 \label{eq:minimax-paits-comparison}
\end{equation}
The constants in \eqref{eq:one-pass-weighted-selection-quadratic-rate} are the same across all models by the corollary's assumptions. Combining \eqref{eq:minimax-paits-uniform-nuisance}, \eqref{eq:step-size-calendar-average-endpoint}, and the exploration schedule gives
\begin{equation}
 \begin{aligned}
 \sup_{P\in\mathcal P^{\mathrm{np}}(s)}
 \Regret_T(\mathcal A_{\mathrm{PAITS}},P)
 &\le CT\big\{\rho_T^2+T^{-\zeta}
                    +T^{-2s/(4s+d)}+T^{-1}\big\}\\
 &\le C(\log T)^{2s/(2s+d)}T^{(2s+d)/(4s+d)}.
 \end{aligned}
 \label{eq:minimax-paits-upper}
\end{equation}
Indeed, \eqref{eq:minimax-common-radius} gives $T\rho_T^2=(\log T)^{2s/(2s+d)}T^{(2s+d)/(4s+d)}$; the exploration term has the same polynomial order without the logarithmic factor, and $1-\zeta<1/2<(2s+d)/(4s+d)$ makes the step-size term smaller. Combining \eqref{eq:minimax-corollary-lower}, \eqref{eq:minimax-paits-comparison}, and \eqref{eq:minimax-paits-upper} therefore yields
\begin{equation}
 \begin{aligned}
 c\frac{T^{(2s+d)/(4s+d)}}{\log T}
 &\le\mathfrak R_T(\mathcal P^{\mathrm{np}}(s);\rho_T,\rho_T)\le\sup_{P\in\mathcal P^{\mathrm{np}}(s)}
       \Regret_T(\mathcal A_{\mathrm{PAITS}},P)\\
 &\le C(\log T)^{2s/(2s+d)}T^{(2s+d)/(4s+d)}.
 \end{aligned}
 \label{eq:minimax-nonparametric-comparison}
\end{equation}
Equation~\eqref{eq:minimax-nonparametric-comparison} proves that PAITS attains the constrained minimax polynomial rate, with the possible difference confined to logarithmic factors.
\end{proof}

\section{Nuisance Estimators with Uniform Rates}
\label{app:weighted-nuisance-proofs}

We establish the uniform nuisance-estimation bounds in \eqref{eq:nuisance_rate} and \eqref{eq:minimax-paits-uniform-nuisance} for stabilized inverse-propensity-weighted spline and local-linear regression. Section~\ref{app:nuisance-estimators} states the required assumptions, defines the proposed estimators, and presents their convergence rates. Section~\ref{app:nuisance-proofs} proves the convergence results using the technical lemmas in Section~\ref{app:nuisance-tools}.

\subsection{Assumptions, estimators, and convergence rates}
\label{app:nuisance-estimators}

We first impose regularity conditions on the data-generating mechanism.

\begin{assumption}[Sampling and nuisance regularity]
	\label{ass:weighted-nuisance-model}
	The sampling and treatment assignment mechanisms are those in Sections~\ref{subsec:causal-setting} and \ref{subsec:epsilon-greedy-design}. The historical and online units share the same distribution, and $n_{\mathrm{hist}}$ is fixed. On $\mathcal X=[0,1]^d$, Assumption~\ref{ass:nonparametric-model} holds, $b_0(\bx)=\Ebb\{Y(0)\mid\bX=\bx\}$, and the control errors satisfy \eqref{eq:minimax-control-subgaussian}. The exploration probabilities satisfy $\varepsilon_t\asymp t^{-\chi}$ for some $\chi\in(0,1]$.
\end{assumption}

The H\"older condition on $F_1(y\mid\cdot)$ is uniform in $y\in\Rbb$. No additional moment condition on $Y(1)$ is needed. Throughout this appendix, $C$ and $c$ denote positive constants depending only on the regularity constants, the assignment and exploration mechanisms, the estimator specifications, and $n_{\mathrm{hist}}$. They may vary between displays and can be chosen uniformly over $\mathcal P^{\mathrm{np}}(s)$. Probabilities and expectations include the historical sample, online observations, and assignment randomization. Although we do not use $\chi=1$ in the main paper, we derive the estimation results for $\chi=1$ for completeness.

Because the online procedure uses adaptive weighting, we introduce the stabilized weights below. For $a\in\{0,1\}$, both methods use all observations in $\cD_{a,t}=\cD_a\cup\{(\bX_v,Y_v):v<t,\ A_v=a\}$. Let $p_1=\pi_{\mathrm{exp}}$, $p_0=1-\pi_{\mathrm{exp}}$, and $G_v(\bx)=\ind\{\widehat\ell_v(\bx)\le\lambda_v\}$. The assignment probabilities and regression weights are defined as
\begin{equation}
	\pi_{a,v}(\bx)=(1-\varepsilon_v)\ind\{G_v(\bx)=a\}+\varepsilon_vp_a,
	\qquad
	w_{a,v}=\frac{\varepsilon_v}{\pi_{a,v}(\bX_v)},
	\qquad
	w_{a,i}^{\mathrm{hist}}=\frac1{\pi_a^{\mathrm{hist}}},
	\label{eq:weighted-regression-weights}
\end{equation}
where $\pi_1^{\mathrm{hist}}=\pi$ and $\pi_0^{\mathrm{hist}}=1-\pi$. Each online weight uses the propensity score at the time when the observation was collected. Define $m_t=1+\sum_{v<t}\varepsilon_v$ and $\nu_t=n_{\mathrm{hist}}+\sum_{v<t}\varepsilon_v$.

We then provide details about the estimators. For spline regression, we use a tensor-product B-spline basis of fixed degree $q\ge\lceil s\rceil$ in each coordinate, with equally spaced knots and repeated boundary knots. The basis functions are nonnegative and sum to one. For local-linear regression, let $s_{\mathrm{lin}}=\min\{s,2\}$ and let $K$ be a bounded, nonnegative, Lipschitz kernel on $\Rbb^d$, supported on the unit ball and bounded below by a positive constant near zero. We define the augmented covariate by $\bp(\bu)=(1,\bu^\top)^\top$ and the unit vector by $\mathbf e_0=(1,0,\ldots,0)^\top$.

To simplify the proof, we partition the observation times according to $m_t$. Let
\begin{equation}
	M_k=2^k,\qquad L_k=\log(eM_k),\qquad
	\mathcal T_k=\{t\ge1:M_k\le m_t<2M_k\},\qquad
	k(t)=\lfloor\log_2m_t\rfloor,
	\label{eq:generic-estimation-epochs}
\end{equation}
for $k\ge0$. Write $\bar s=s$ for spline regression and $\bar s=s_{\mathrm{lin}}$ for local-linear regression. Choose deterministic values $h^{(k)}\asymp(L_k/M_k)^{1/(2\bar s+d)}$ satisfying $0<h^{(k)}\le h_0<1$ for a positive constant $h_0$, and set $h_t=h^{(k(t))}$. For spline regression, each $h^{(k)}$ is the reciprocal of an integer. We denote the corresponding tensor-product basis by $\mathbf B^{(k)}$ and set $\mathbf B_t=\mathbf B^{(k(t))}$. Thus the mesh size or bandwidth is constant for $t\in\mathcal T_k$, and each estimator uses all observations available before time $t$. The inequalities $M_{k(t)}\le m_t<2M_{k(t)}$ in \eqref{eq:generic-estimation-epochs} imply the orders of $h_t$ stated in Propositions~\ref{prop:spline-nuisance} and~\ref{prop:local-linear-nuisance}.

We use the weighted $\ell_2$ loss to estimate the coefficients. We first present the Gram matrices and eigenvalue properties. For spline regression, let $k=k(t)$ and define
$$
\begin{aligned}
	\bG_{a,t}&=\sum_{i\in\cD_{a,t}}w_i\mathbf B^{(k)}(\bX_i)\mathbf B^{(k)}(\bX_i)^\top,\\
	\bR_{b,t}&=\sum_{i\in\cD_{0,t}}w_i\mathbf B^{(k)}(\bX_i)Y_i,\\
	\bR_{F,t}(y)&=\sum_{i\in\cD_{1,t}}w_i\mathbf B^{(k)}(\bX_i)\ind\{Y_i\le y\}.
\end{aligned}
$$
For local-linear regression, write $h=h^{(k(t))}$, $\mathbf v_{\bx,h}(\bz)=\bp\{(\bz-\bx)/h\}$, and $K_{\bx,h}(\bz)=K\{(\bz-\bx)/h\}$, and define
$$
\bG_{a,t}(\bx)=\sum_{i\in\cD_{a,t}}w_iK_{\bx,h}(\bX_i)\mathbf v_{\bx,h}(\bX_i)\mathbf v_{\bx,h}(\bX_i)^\top,
$$
with $\bR_{b,t}(\bx)$ and $\bR_{F,t}(\bx,y)$ defined by multiplying $w_iK_{\bx,h}(\bX_i)\mathbf v_{\bx,h}(\bX_i)$ by $Y_i$ and $\ind\{Y_i\le y\}$, respectively, and summing over the corresponding samples. Suppose that
\begin{equation}
	c_{\mathrm{fit}}\nu_th^d\bI\preceq \bG\preceq C_{\mathrm{fit}}\nu_th^d\bI,
	\label{eq:weighted-regression-safeguard}
\end{equation}
where $c_{\mathrm{fit}}$ and $C_{\mathrm{fit}}$ are the constants in the lower and upper eigenvalue bounds in \eqref{eq:spline-approximation-design} and \eqref{eq:local-linear-population-design}. The factor $\nu_t h^d$ is the order of the expected total weight of observations within the support of a basis or kernel function. Condition~\eqref{eq:weighted-regression-safeguard} requires the corresponding Gram matrices to have eigenvalues of this order, ensuring invertibility and uniformly bounded condition numbers. 

In both regressions, the coefficient vector is $\bG^{-1}\bR$ when $\nu_t>0$. To be precise, the control reference is estimated by weighted regression of $Y$ on $\bX$ using $\cD_{0,t}$. For each $y\in\Rbb$, the conditional distribution is estimated by weighted regression of $\ind\{Y\le y\}$ on $\bX$ using $\cD_{1,t}$. For spline regression, define
\begin{equation}
	\begin{aligned}
\widehat{\boldsymbol\theta}_{b,t}&\in\arg\min_{\boldsymbol\theta}\sum_{i\in\cD_{0,t}}w_i\{Y_i-\mathbf B_t(\bX_i)^\top\boldsymbol\theta\}^2, \\
		\widehat{\boldsymbol\theta}_{F,t}(y)&\in\arg\min_{\boldsymbol\theta}\sum_{i\in\cD_{1,t}}w_i[\ind\{Y_i\le y\}-\mathbf B_t(\bX_i)^\top\boldsymbol\theta]^2,
	\end{aligned}
	\label{eq:weighted-spline-fit}
\end{equation}
and set $\widetilde b_{0,t}(\bx)=\mathbf B_t(\bx)^\top\widehat{\boldsymbol\theta}_{b,t}$ and $\widetilde F_{1,t}(y\mid\bx)=\mathbf B_t(\bx)^\top\widehat{\boldsymbol\theta}_{F,t}(y)$. For local-linear regression, define
\begin{equation}
	\begin{aligned}
		\widehat{\boldsymbol\theta}_{b,t}(\bx)
		&\in\arg\min_{\boldsymbol\theta}\sum_{i\in\cD_{0,t}}w_iK\Big(\frac{\bX_i-\bx}{h_t}\Big)\Big\{Y_i-\bp\Big(\frac{\bX_i-\bx}{h_t}\Big)^\top\boldsymbol\theta\Big\}^2,\\
		\widehat{\boldsymbol\theta}_{F,t}(\bx,y)&\in\arg\min_{\boldsymbol\theta}\sum_{i\in\cD_{1,t}}w_iK\Big(\frac{\bX_i-\bx}{h_t}\Big)\bigg[\ind\{Y_i\le y\}-\bp\Big(\frac{\bX_i-\bx}{h_t}\Big)^\top\boldsymbol\theta\bigg]^2,
	\end{aligned}
	\label{eq:weighted-local-linear-fit}
\end{equation}
and use the fitted intercepts $\widetilde b_{0,t}(\bx)=\mathbf e_0^\top\widehat{\boldsymbol\theta}_{b,t}(\bx)$ and $\widetilde F_{1,t}(y\mid\bx)=\mathbf e_0^\top\widehat{\boldsymbol\theta}_{F,t}(\bx,y)$. Use each fit only when $\nu_t>0$, and it is well-defined by \eqref{eq:weighted-regression-safeguard}. Otherwise, use $\widetilde b_{0,t}(\bx)=0$ for the control fit and $\widetilde F_{1,t}(y\mid\bx)=\ind\{y\ge0\}$ for the treatment fit. For local-linear regression, this condition is checked separately at each query point. Finally, when necessary, project the control estimate onto its known range and monotonize the conditional distribution estimate:
\begin{equation}
	\widehat b_{0,t}(\bx)=\Pi_{[-L,L]}\{\widetilde b_{0,t}(\bx)\},
	\qquad
	\widehat F_{1,t}(y\mid\bx)=\sup_{z\le y}\Pi_{[0,1]}\{\widetilde F_{1,t}(z\mid\bx)\}.
	\label{eq:weighted-nonparametric-nuisance}
\end{equation}
Lemma~\ref{lem:weighted-cdf-postprocessing} shows that $\widehat F_{1,t}(\cdot\mid\bx)$ is a distribution function and that postprocessing does not increase the uniform estimation error.

Then, we state the convergence rate of the proposed methods. The errors $\delta_{b,t}^{\infty}$ and $\delta_{F,t}^{\infty}$ are defined after \eqref{eq:nuisance_rate}. We bound the error of the conditional distribution estimator uniformly over $(\bx,y)\in[0,1]^d\times\Rbb$ without requiring knowledge of $b_0$. For $\chi\in(0,1]$, define
$$
\psi_T(\chi)=\frac{\log(T\vee2)}{T^{1-\chi}}\quad(0<\chi<1),
\qquad
\psi_T(1)=\frac{\log\log(T\vee e^e)}{\log(T\vee e^e)}.
$$

\begin{proposition}[Weighted spline regression]
	\label{prop:spline-nuisance}
	Under Assumption~\ref{ass:weighted-nuisance-model}, consider the spline estimators defined above, with $h_t\asymp\{\log(em_t)/m_t\}^{1/(2s+d)}$ chosen according to \eqref{eq:generic-estimation-epochs}. For every $P\in\mathcal P^{\mathrm{np}}(s)$, $\delta_{b,t}^{\infty}\to0$ and $\delta_{F,t}^{\infty}\to0$ almost surely. Moreover, for all $T\ge1$,
	\begin{equation}
		\sup_{P\in\mathcal P^{\mathrm{np}}(s)}\frac1T\sum_{t=1}^T
		\Ebb_P\big[(\delta_{b,t}^{\infty})^2+(\delta_{F,t}^{\infty})^2\big]
		\le C\,\psi_T(\chi)^{2s/(2s+d)} .
		\label{eq:spline-average-rates}
	\end{equation}
\end{proposition}

\begin{proposition}[Weighted local-linear regression]
	\label{prop:local-linear-nuisance}
	Under Assumption~\ref{ass:weighted-nuisance-model}, consider the local-linear estimators defined above, with $h_t\asymp\{\log(em_t)/m_t\}^{1/(2s_{\mathrm{lin}}+d)}$ chosen according to \eqref{eq:generic-estimation-epochs}. For every $P\in\mathcal P^{\mathrm{np}}(s)$, both estimation errors converge to zero almost surely. In addition, \eqref{eq:spline-average-rates} holds with $s$ replaced by $s_{\mathrm{lin}}=\min\{s,2\}$.
\end{proposition}

For $s\le2$, the two methods attain the same rates. For $s>2$, spline regression exploits the additional smoothness through its polynomial degree, whereas the local-linear approximation error remains $O(h^2)$.

\subsection{Proofs of the convergence results}
\label{app:nuisance-proofs}

We first establish a uniform error bound on each interval $\mathcal T_k$, then derive almost-sure consistency and time-averaged mean-square rates. For either estimator, define
$$
W_k=\sup_{t\in\mathcal T_k}\max\Big\{\|\widehat b_{0,t}-b_0\|_\infty,
\sup_{\bx\in[0,1]^d,\,y\in\Rbb}
|\widehat F_{1,t}(y\mid\bx)-F_1(y\mid\bx)|\Big\}.
$$
Here $\|\cdot\|_\infty$ denotes the supremum over $[0,1]^d$. Writing $M=M_k$, $h=h^{(k)}$, and $\bar s=s$ for spline regression or $\bar s=s_{\mathrm{lin}}$ for local-linear regression, we show that
\begin{equation}
	\sup_{P\in\mathcal P^{\mathrm{np}}(s)}\Pbb_P\Big[W_k>
	C\Big\{h^{\bar s}+\sqrt{\frac{L_k+u}{Mh^d}}+\frac{L_k+u}{Mh^d}\Big\}\Big]
	\le Ce^{-u},\qquad u\ge0 .
	\label{eq:weighted-regression-uniform-tail}
\end{equation}
The constants do not depend on $k$, $u$, or $P$.

\subsubsection*{Weighted spline regression}

Fix $k$, and write $M=M_k$, $h=h^{(k)}$, and $\mathbf B=\mathbf B^{(k)}$. Lemma~\ref{lem:weighted-spline-design} provides the approximation and population Gram matrix bounds in \eqref{eq:spline-approximation-design}, uniformly over the nuisance functions. Let
$$
\begin{aligned}
	\bU_{b,t}&=\sum_{i\in\cD_{0,t}}w_i\mathbf B^{(k)}(\bX_i)\{Y_i(0)-b_0(\bX_i)\},\\
	\br_F(y)&=\Ebb\{\mathbf B^{(k)}(\bX)F_1(y\mid\bX)\}.
\end{aligned}
$$
Apply \eqref{eq:weighted-epoch-scalar-tail} to each coordinate of $\bU_{b,t}$ and each potentially nonzero entry of the Gram matrices. For each coordinate of $\bR_{F,t}(y)$, apply \eqref{eq:weighted-cdf-threshold-tail} with $\vartheta=h^s$. There are $O(J_k)$ quantities in these bounds and $\log J_k=O(L_k)$. Consequently, a union bound gives, with probability at least $1-Ce^{-u}$,
\begin{equation}
	\begin{aligned}
		\sup_{t\in\mathcal T_k}\Big\{\max_{a\in\{0,1\}}
		\|\bG_{a,t}-\nu_t\bGamma_k\|_\infty+\|\bU_{b,t}\|_\infty\Big\}
		&\le C\{\sqrt{Mh^d(L_k+u)}+L_k+u\},\\
		\sup_{t\in\mathcal T_k,\,y\in\Rbb}\|\bR_{F,t}(y)-\nu_t\br_F(y)\|_\infty
		&\le C\{\sqrt{Mh^d(L_k+u)}+L_k+u+Mh^{d+s}\}.
	\end{aligned}
	\label{eq:spline-epoch-moments}
\end{equation}
Here $\|\cdot\|_\infty$ denotes the maximum absolute row sum for matrices and the maximum absolute coordinate for vectors. Each row of the Gram matrix has a bounded number of nonzero entries, and $Mh^d/L_k\to\infty$. Combining \eqref{eq:spline-epoch-moments} with \eqref{eq:spline-approximation-design} shows that, on the same event, every $\bG_{a,t}$ with $t\in\mathcal T_k$ satisfies \eqref{eq:weighted-regression-safeguard} for $0\le u\le c_1Mh^d$, for all sufficiently large $k$ and sufficiently small $c_1>0$. By the inverse bound in Lemma~\ref{lem:weighted-spline-design} and $\nu_t\asymp M$, we also have
$$
\|\bG_{a,t}^{-1}\|_\infty\le \frac{C}{Mh^d},
\qquad \|\bG_{a,t}\mathbf1\|_\infty\le CMh^d.
$$

We then derive the uniform estimation errors. Let $b_{0,h}=\mathbf B^\top\boldsymbol\beta_b$ satisfy
\eqref{eq:spline-approximation-design}. The estimating equations associated with \eqref{eq:weighted-spline-fit} give
$$
\widehat{\boldsymbol\theta}_{b,t}-\boldsymbol\beta_b=\bG_{0,t}^{-1}\sum_{i\in\cD_{0,t}}w_i\mathbf B(\bX_i)\{b_0(\bX_i)-b_{0,h}(\bX_i)\}+\bG_{0,t}^{-1}\bU_{b,t}.
$$
By the nonnegativity of the basis functions, \eqref{eq:spline-approximation-design}, and the matrix bounds, the first term has maximum norm of order $h^s$. The first inequality in \eqref{eq:spline-epoch-moments} bounds the second term. For the conditional CDF, let $F_{1,h}(y\mid\bx)=\mathbf B(\bx)^\top\boldsymbol\beta_F(y)$ be the approximation in \eqref{eq:spline-approximation-design}. Decompose
\begin{equation}
	\bR_{F,t}(y)-\bG_{1,t}\boldsymbol\beta_F(y)
	=\{\bR_{F,t}(y)-\nu_t\br_F(y)\}
	+\nu_t\{\br_F(y)-\bGamma_k\boldsymbol\beta_F(y)\}
	+(\nu_t\bGamma_k-\bG_{1,t})\boldsymbol\beta_F(y).
	\label{eq:cdf-population-centering-decomposition}
\end{equation}
The first term in \eqref{eq:cdf-population-centering-decomposition} is bounded by \eqref{eq:spline-epoch-moments}. The second term is $O(Mh^{d+s})$ in each coordinate, since $\Ebb B_j(\bX)\le Ch^d$ and \eqref{eq:spline-approximation-design} holds uniformly over $y$. The last term is bounded by \eqref{eq:spline-epoch-moments} and $\sup_y\|\boldsymbol\beta_F(y)\|_\infty\le C$. Multiplying \eqref{eq:cdf-population-centering-decomposition} by $\bG_{1,t}^{-1}$ and using $\sum_jB_j^{(k)}\equiv1$ gives a bound for the error of the initial CDF estimator uniformly over $(\bx,y)$. Projection and Lemma~\ref{lem:weighted-cdf-postprocessing} then give \eqref{eq:weighted-regression-uniform-tail} with $\bar s=s$ for $u\le c_1Mh^d$. For $u>c_1Mh^d$, enlarge $C$ so that the threshold exceeds the deterministic bound $W_k\le\max\{2L,1\}$. The same adjustment covers the finitely many initial intervals. This proves the common tail bound for the spline estimator.

\subsubsection*{Weighted local-linear regression}

Fix $k$, and write $M=M_k$, $h=h^{(k)}$, and $\bar s=s_{\mathrm{lin}}$. Let $\Gamma_h(\bx)=\Ebb\{K_{\bx,h}(\bX)\mathbf v_{\bx,h}(\bX)\mathbf v_{\bx,h}(\bX)^\top\}$. Lemma~\ref{lem:weighted-local-linear-design} gives the uniform eigenvalue bounds in \eqref{eq:local-linear-population-design}. Let
$$
\begin{aligned}
	\bU_{b,t}(\bx)&=\sum_{i\in\cD_{0,t}}w_iK_{\bx,h}(\bX_i)
	\mathbf v_{\bx,h}(\bX_i)\{Y_i(0)-b_0(\bX_i)\},\\
	\br_F(\bx,y)&=\Ebb\{K_{\bx,h}(\bX)\mathbf v_{\bx,h}(\bX)F_1(y\mid\bX)\}.
\end{aligned}
$$
For fixed $\bx$, apply \eqref{eq:weighted-epoch-scalar-tail} with $g=K_{\bx,h}v_{\bx,h,j}v_{\bx,h,l}$ to each centered entry of $\bG_{a,t}(\bx)$, and with $g=K_{\bx,h} v_{\bx,h,j}$ to each coordinate of $U_{b,t}(\bx)$. Apply \eqref{eq:weighted-cdf-threshold-tail} to the positive and negative parts of $K_{\bx,h}\mathbf v_{\bx,h,j}$, with $\vartheta=h^{\bar s}$, to bound each coordinate of $\bR_{F,t}(\bx,y)-\nu_t \br_F(\bx,y)$ uniformly over $y$.

For uniformity over $\bx$, take a deterministic grid on $[0,1]^d$ with mesh $hM^{-4}$. Its cardinality is at most $C(hM^{-4})^{-d}$, whose logarithm is $O(L_k)$ by the bandwidth choice following \eqref{eq:generic-estimation-epochs}. To pass from the grid to $[0,1]^d$, first note that \eqref{eq:fullsample-weighted-moments} and \eqref{eq:minimax-control-subgaussian} give, for a sufficiently small fixed $\theta>0$,
$$
\Ebb\{\exp(\theta Z_{a,v})\mid\mathcal F_{v-1}\}
\le e^{C\varepsilon_v},
\qquad
\Ebb[\exp\{\theta Z_{0,v}|Y_v(0)-b_0(\bX_v)|\}\mid\mathcal F_{v-1}]\le e^{C\varepsilon_v}.
$$
By \eqref{eq:minimax-control-subgaussian} and H\"older's inequality, the historical sums have uniformly bounded exponential moments because $n_{\mathrm{hist}}$ is fixed. Iterated conditioning, Markov's inequality, and the nonnegativity of the summands therefore yield, with probability at least $1-Ce^{-cM}$,
$$
\sup_{t\in\mathcal T_k}\max_{a\in\{0,1\}}
\sum_{i\in\cD_{a,t}}w_i\le CM,\qquad
\sup_{t\in\mathcal T_k}\sum_{i\in\cD_{0,t}}
w_i|Y_i(0)-b_0(\bX_i)|\le CM .
$$
The functions $K_{\bx,h}v_{\bx,h,j}$ and $K_{\bx,h} v_{\bx,h,j} v_{\bx,h,l}$, extended by zero outside their supports, are Lipschitz in $\bx$ with constant $C/h$. On this event, replacing $\bx$ by a nearest grid point changes each centered Gram entry and each coordinate of $\bU_{b,t}(\bx)$ or $\bR_{F,t}(\bx,y)-\nu_t \br_F(\bx,y)$ by at most $(C/h)(hM^{-4})M=CM^{-3}$. For the last expression, the bound is uniform over $y$ because $\ind\{Y\le y\}\le1$ and $0\le F_1(y\mid\bX)\le1$. Taking $v=C_0L_k+u$ in \eqref{eq:weighted-epoch-scalar-tail} and \eqref{eq:weighted-cdf-threshold-tail}, and taking a union bound over the grid points and coordinates, gives, with probability at least $1-Ce^{-u}$ and for $0\le u\le c_1Mh^d$,
\begin{equation}
	\begin{aligned}
		\sup_{t\in\mathcal T_k,\,\bx\in[0,1]^d} &
		\Big\{\max_a\|G_{a,t}(\bx)-\nu_t\Gamma_h(\bx)\|+\|U_{b,t}(\bx)\|\Big\}
		\le C\{\sqrt{Mh^d(L_k+u)}+L_k+u\},\\
		\sup_{t\in\mathcal T_k,\,\bx\in[0,1]^d,\,y\in\Rbb} &
		\|R_{F,t}(\bx,y)-\nu_tr_F(\bx,y)\|
		\le C\{\sqrt{Mh^d(L_k+u)}+L_k+u+Mh^{d+s_{\mathrm{lin}}}\},
	\end{aligned}
	\label{eq:local-linear-epoch-moments}
\end{equation}
where the matrix and vector norms are the spectral and Euclidean norms, respectively. Here $Ce^{-cM}\le Ce^{-u}$ after decreasing $c_1$, and the grid approximation error $CM^{-3}$ is absorbed into the right-hand sides of \eqref{eq:local-linear-epoch-moments}. Since $Mh^d/L_k\to\infty$, the first inequality in \eqref{eq:local-linear-epoch-moments} and \eqref{eq:local-linear-population-design} imply that \eqref{eq:weighted-regression-safeguard} holds simultaneously for all $t\in\mathcal T_k$ and $\bx\in[0,1]^d$ on the same event, for all sufficiently large $k$ and sufficiently small $c_1>0$. In particular,
$$
\|\bG_{a,t}(\bx)^{-1}\|\le\frac C{Mh^d},\qquad
\sum_{i\in\cD_{a,t}}w_iK_{\bx,h}(\bX_i)
=(\bG_{a,t}(\bx))_{11}\le CMh^d .
$$

We then derive the uniform estimation errors. Let $\boldsymbol\beta_b(\bx)$ be the local approximation vector in Lemma~\ref{lem:weighted-local-linear-design} with $g=b_0$. Its intercept equals $b_0(\bx)$, and its approximation error is $O(h^{\bar s})$ throughout the kernel support. On the event in \eqref{eq:local-linear-epoch-moments}, the normal equations associated with \eqref{eq:weighted-local-linear-fit} give
$$
\begin{aligned}
	&\widehat{\boldsymbol\theta}_{b,t}(\bx)-\boldsymbol\beta_b(\bx)\\
	=&\bG_{0,t}(\bx)^{-1}\sum_{i\in\cD_{0,t}}w_iK_{\bx,h}(\bX_i)\mathbf v_{\bx,h}(\bX_i)
	\{b_0(\bX_i)-\mathbf v_{\bx,h}(\bX_i)^\top\boldsymbol\beta_b(\bx)\}+
	\bG_{0,t}(\bx)^{-1}\bU_{b,t}(\bx).
\end{aligned}
$$
Since $\mathbf e_0^\top\boldsymbol\beta_b(\bx)=b_0(\bx)$ and $\|\mathbf v_{\bx,h}\|\le C$ on the kernel support, \eqref{eq:weighted-regression-safeguard} and the first inequality in \eqref{eq:local-linear-epoch-moments} imply
$$
\begin{aligned}
	|\widetilde b_{0,t}(\bx)-b_0(\bx)|
	&\le\frac C{Mh^d}\left\{h^{\bar s}
	\sum_{i\in\cD_{0,t}}w_iK_{\bx,h}(\bX_i)+\|\bU_{b,t}(\bx)\|\right\}\\
	&\le C\left\{h^{\bar s}
	+\sqrt{\frac{L_k+u}{Mh^d}}+\frac{L_k+u}{Mh^d}\right\}.
\end{aligned}
$$
For the conditional CDF, use the vector $\boldsymbol\beta_F(\bx,y)$ from Lemma~\ref{lem:weighted-local-linear-design} with $g=F_1(y\mid\cdot)$. Uniformly over $(\bx,y)$, this vector is bounded, its intercept equals $F_1(y\mid\bx)$, and
$$
\|\br_F(\bx,y)-\bGamma_h(\bx)\boldsymbol\beta_F(\bx,y)\|\le Ch^{d+\bar s}.
$$
Apply \eqref{eq:cdf-population-centering-decomposition} with $\bR_{F,t}(\bx,y)$, $\br_F(\bx,y)$, $\bGamma_h(\bx)$, and $\boldsymbol\beta_F(\bx,y)$ in place of their spline counterparts. The normal equations associated with \eqref{eq:weighted-local-linear-fit} and \eqref{eq:weighted-regression-safeguard} then give
$$
\begin{aligned}
	|\widetilde F_{1,t}(y\mid\bx)-F_1(y\mid\bx)|
	\le&Ch^{\bar s}+\frac C{Mh^d}
	\Big\{\|\bR_{F,t}(\bx,y)-\nu_t \br_F(\bx,y)\|+\|\bG_{1,t}(\bx)-\nu_t\bGamma_h(\bx)\|\Big\}\\
	\le&C\left\{h^{\bar s}
	+\sqrt{\frac{L_k+u}{Mh^d}}+\frac{L_k+u}{Mh^d}\right\},
\end{aligned}
$$
where the last inequality follows from \eqref{eq:local-linear-epoch-moments}. The two estimation error bounds, the projection in \eqref{eq:weighted-nonparametric-nuisance}, and \eqref{eq:cdf-monotonicity-contraction} establish \eqref{eq:weighted-regression-uniform-tail} for $0\le u\le c_1Mh^d$. For $u>c_1Mh^d$, the bound $W_k\le\max\{2L,1\}$ extends \eqref{eq:weighted-regression-uniform-tail} to all $u\ge0$ after increasing $C$, which also covers the finitely many initial intervals. This proves the common tail bound for the local-linear estimator.

We now use \eqref{eq:weighted-regression-uniform-tail} to prove the almost-sure convergence and the time-averaged mean-square error bounds in Propositions~\ref{prop:spline-nuisance} and \ref{prop:local-linear-nuisance}. For either estimator, integration of the tail probability in \eqref{eq:weighted-regression-uniform-tail}, together with ${h^{(k)}}^{2\bar s+d}\asymp L_k/M_k$, gives
\begin{equation}
	\sup_{P\in\mathcal P^{\mathrm{np}}(s)}\Ebb_P(W_k^2)
	\le C\Big\{{h^{(k)}}^{2\bar s}+\frac{L_k}{M_k{h^{(k)}}^d}
	+\Big(\frac{L_k}{M_k{h^{(k)}}^d}\Big)^2\Big\}
	\le C\Big(\frac{L_k}{M_k}\Big)^{2\bar s/(2\bar s+d)} .
	\label{eq:weighted-epoch-mean-square}
\end{equation}
Taking $u=C_0L_k$ in \eqref{eq:weighted-regression-uniform-tail} makes the probabilities on the left-hand side summable in $k$. The Borel--Cantelli lemma therefore gives $W_k\to0$ almost surely for every $P\in\mathcal P^{\mathrm{np}}(s)$. By the definition of $W_k$, both $\delta_{b,t}^{\infty}$ and $\delta_{F,t}^{\infty}$ converge to zero almost surely. Since both errors are bounded by $W_{k(t)}$, $(\delta_{b,t}^{\infty})^2+(\delta_{F,t}^{\infty})^2\le2W_{k(t)}^2$. Combining \eqref{eq:weighted-epoch-mean-square} with $M_{k(t)}\le m_t<2M_{k(t)}$ in \eqref{eq:generic-estimation-epochs} gives
\begin{equation}
	\sup_{P\in\mathcal P^{\mathrm{np}}(s)}
	\Ebb_P\big[(\delta_{b,t}^{\infty})^2+(\delta_{F,t}^{\infty})^2\big]
	\le C\Big\{\frac{\log(em_t)}{m_t}\Big\}^{2\bar s/(2\bar s+d)} .
	\label{eq:nonparametric-pointwise-time-uniform-risk}
\end{equation}
It remains to average \eqref{eq:nonparametric-pointwise-time-uniform-risk} over $t\le T$. Let $\beta=2\bar s/(2\bar s+d)\in(0,1)$. If $0<\chi<1$, the growth rate in \eqref{eq:asymp_s_T} gives $m_t\asymp t^{1-\chi}$. Since $(1-\chi)\beta<1$, $T^{-1}\sum_{t\le T}\{\log(em_t)/m_t\}^\beta\le C\psi_T(\chi)^\beta$. If $\chi=1$, \eqref{eq:asymp_s_T} gives $m_t\asymp\log(t\vee2)$. The contribution of the first $\lfloor\sqrt T\rfloor$ terms to the average is $O(T^{-1/2})$. For the remaining terms, $\log t\asymp\log T$ and $\log\log t\asymp\log\log T$, so their contribution is at most $C\psi_T(1)^\beta$. This establishes \eqref{eq:spline-average-rates} for spline regression and the corresponding bound with $s$ replaced by $s_{\mathrm{lin}}$ for local-linear regression. The constants are uniform over $P\in\mathcal P^{\mathrm{np}}(s)$.

\subsection{Technical lemmas}
\label{app:nuisance-tools}

The following lemmas are stated under the conditions and estimator specifications in Section~\ref{app:nuisance-estimators}. The constants are uniform over $\mathcal P^{\mathrm{np}}(s)$.

\begin{lemma}[Constant reproduction and CDF postprocessing]
	\label{lem:weighted-cdf-postprocessing}
	For $\nu_t>0$, suppose that \eqref{eq:weighted-regression-safeguard} holds; both weighted regressions contain constants. With the prescribed estimates, \eqref{eq:weighted-nonparametric-nuisance} therefore defines a distribution function $\widehat F_{1,t}(\cdot\mid\bx)$ for every $\bx$. For any distribution function $F(\cdot\mid\bx)$,
	\begin{equation}
		\sup_y|\widehat F_{1,t}(y\mid\bx)-F(y\mid\bx)|
		\le\sup_y|\widetilde F_{1,t}(y\mid\bx)-F(y\mid\bx)|=:e .
		\label{eq:cdf-monotonicity-contraction}
	\end{equation}
\end{lemma}

\begin{proof}
	For spline and local-linear regression, respectively,
	$$
	\bG_{1,t}\mathbf1=\sum_{i\in\cD_{1,t}}w_i\mathbf B_t(\bX_i),
	\qquad
	\bG_{1,t}(\bx)\mathbf e_0
	=\sum_{i\in\cD_{1,t}}w_iK_{\bx,h}(\bX_i)
	\mathbf v_{\bx,h}(\bX_i).
	$$
	Thus a constant response is fitted exactly. Since the weights and Gram matrices do not depend on $y$, the initial CDF estimate is a right-continuous step function whose jumps occur at the observed treated outcomes and whose limits at $-\infty$ and $+\infty$ are $0$ and $1$. The estimate has the same right-continuity and limit properties. Projection
	onto $[0,1]$ followed by a running supremum preserves right-continuity and these limits and yields a nondecreasing function. To prove \eqref{eq:cdf-monotonicity-contraction}, let $e$ denote its right-hand side. Projection cannot increase the error, and monotonicity of $F$ gives
	$$
	\widehat F_{1,t}(y\mid\bx)\le\sup_{z\le y}\{F(z\mid\bx)+e\}=F(y\mid\bx)+e.
	$$
	Taking $z=y$ gives the reverse bound $\widehat F_{1,t}(y\mid\bx)\ge F(y\mid\bx)-e$.
\end{proof}

\begin{lemma}[Weighted moments and maximal concentration]
	\label{lem:weighted-epoch-concentration}
	Let $Z_{a,v}=\ind\{A_v=a\}\varepsilon_v/\pi_{a,v}(\bX_v)$. Then $0\le Z_{a,v}\le p_a^{-1}$ and, for every deterministic bounded measurable function $f$,
	\begin{equation}
		\begin{aligned}
			\Ebb\{Z_{a,v}f(\bX_v,Y_v(a))\mid\mathcal F_{v-1}\}
			&=\varepsilon_v\Ebb f(\bX,Y(a)),\\
			\Ebb\{Z_{a,v}^2f(\bX_v,Y_v(a))^2\mid\mathcal F_{v-1}\}
			&\le p_a^{-1}\varepsilon_v\Ebb f(\bX,Y(a))^2 .
			\label{eq:fullsample-weighted-moments}
		\end{aligned}
	\end{equation}
	Fix $k$, write $M=M_k$ and $h=h^{(k)}$, and let $g(\cdot)$ be deterministic with $|g(\cdot)|\le C$ and $P_X(\operatorname{supp}g)\le Ch^d$. Consider the centered sum
	$$
	U_t=\sum_{i\in\cD_{a,t}}w_ig(\bX_i)\phi_i-\nu_t\Ebb\{g(\bX)\phi\},
	\qquad \phi\in\{1,\ \ind\{Y(1)\le y\},\ \ind\{Y(1)<y\}\},
	$$
	where $a=1$ for the last two choices of $\phi$ and $y$ is fixed.
	Alternatively, consider the control residual sum
	$$
	U_t=\sum_{i\in\cD_{0,t}}w_ig(\bX_i)\{Y_i(0)-b_0(\bX_i)\}.
	$$
	In either case,
	\begin{equation}
		\Pbb\Big\{\sup_{t\in\mathcal T_k}|U_t|>C\big(\sqrt{Mh^dv}+v\big)\Big\}
		\le2e^{-v},\qquad v\ge1.
		\label{eq:weighted-epoch-scalar-tail}
	\end{equation}
\end{lemma}

\begin{proof}
	Sequential randomization and independent sampling give the first identity in \eqref{eq:fullsample-weighted-moments} by conditioning first on the current covariates and potential outcomes. The second follows from $\pi_{a,v}(\bx)\ge\varepsilon_vp_a$. For a historical observation, the corresponding marginal relations hold with $\varepsilon_v$ replaced by $1$ and $p_a$ by $\pi_a^{\mathrm{hist}}$. In particular, the expected weighted Gram matrix equals $\nu_t$ times its population counterpart, and $\nu_t\asymp M_k$ uniformly over $t\in\mathcal T_k$ for all sufficiently large $k$.
	
	For the bounded-response sum, \eqref{eq:fullsample-weighted-moments} implies that the increments $U_{t+1}-U_t=Z_{a,t}g(\bX_t)\phi_t-\varepsilon_t\Ebb\{g(\bX)\phi\}$ have conditional mean zero, are bounded in absolute value by $C$, and have conditional second moment at most $C\varepsilon_t h^d$, given $\mathcal F_{t-1}$. Expanding the exponential gives, for $|\theta|\le c$,
	$$
	\Ebb[\exp\{\theta(U_{t+1}-U_t)\}\mid\mathcal F_{t-1}]
	\le 1+C\theta^2\varepsilon_t h^d
	\le\exp(C\theta^2\varepsilon_t h^d).
	$$
	For the control residual sum, independent sampling in Section~\ref{subsec:causal-setting} and \eqref{eq:generic-sequential-randomization} allow \eqref{eq:minimax-control-subgaussian} to be applied conditionally on $(\bX_t,A_t,\mathcal F_{t-1})$. Thus
	\begin{align*}
		\Ebb[\exp\{\theta(U_{t+1}-U_t)\}\mid\mathcal F_{t-1}]
		&\le\Ebb[\exp\{\sigma^2\theta^2Z_{0,t}^2g(\bX_t)^2/2\}
		\mid\mathcal F_{t-1}]\\
		&\le 1+C\theta^2\Ebb[Z_{0,t}^2g(\bX_t)^2\mid\mathcal F_{t-1}]\le\exp(C\theta^2\varepsilon_t h^d),
	\end{align*}
	where the last line uses the boundedness of $Z_{0,t}g(\bX_t)$ and \eqref{eq:fullsample-weighted-moments}. For each historical observation, the same marginal bounds hold with
	exploration factor $1$ and assignment probability $\pi_a^{\mathrm{hist}}$. For fixed $n_{\mathrm{hist}}$, H\"older's inequality gives, after adjusting $C$ and $c$,
	$$
	\Ebb\exp(\theta U_1)\le\exp(C\theta^2n_{\mathrm{hist}}h^d),
	\qquad |\theta|\le c.
	$$
	The conditional bounds above and $\nu_{t+1}-\nu_t=\varepsilon_t$ give
	\begin{align*}
		&\Ebb[
		\exp\{\theta U_{t+1}-C\theta^2h^d\nu_{t+1}\}
		\mid\mathcal F_{t-1}]=\exp\{\theta U_t-C\theta^2h^d\nu_{t+1}\}
		\Ebb[\exp\{\theta(U_{t+1}-U_t)\}\mid\mathcal F_{t-1}]\\
		\leq&\exp\{\theta U_t-C\theta^2h^d\nu_t\}.
	\end{align*}
	Hence, for each fixed $|\theta|\le c$, $\exp\{\theta U_t-C\theta^2h^d\nu_t\}$ is a nonnegative supermartingale with respect to $(\mathcal F_{t-1})_{t\ge1}$ and has initial expectation at most $1$. Since \eqref{eq:generic-estimation-epochs} implies $\nu_t\le(n_{\mathrm{hist}}+2)M$ for $t\in\mathcal T_k$, applying the maximal inequality to the processes with parameters $\theta$ and $-\theta$ yields, for $z>0$ and $0<\theta\le c$,
	$$
	\Pbb\Big\{\sup_{t\in\mathcal T_k}|U_t|\ge z\Big\}
	\le2\exp\{-\theta z+C_1\theta^2Mh^d\}.
	$$
	Taking $\theta=\min\{c,z/(2C_1Mh^d)\}$ and $z=C(\sqrt{Mh^dv}+v)$, with $C$ sufficiently large, gives \eqref{eq:weighted-epoch-scalar-tail}.
\end{proof}

\begin{lemma}[Uniform concentration over outcome thresholds]
	\label{lem:weighted-threshold-concentration}
	Fix $k$, $M=M_k$, and $h=h^{(k)}$, and let $g$ satisfy the conditions of Lemma~\ref{lem:weighted-epoch-concentration}. Define
	$$
	V_{g,t}(y)=\sum_{i\in\cD_{1,t}}w_ig(\bX_i)\ind\{Y_i\le y\}-\nu_t\Ebb\{g(\bX)\ind\{Y(1)\le y\}\}.
	$$
	For every $0<\vartheta\le1$ and $v\ge1$,
	\begin{equation}
		\Pbb\Big[\sup_{t\in\mathcal T_k,\,y\in\Rbb}|V_{g,t}(y)|
		>C\big\{\sqrt{Mh^d[v+\log(e/\vartheta)]}+v+\log(e/\vartheta)
		+Mh^d\vartheta\big\}\Big]\le Ce^{-v}.
		\label{eq:weighted-cdf-threshold-tail}
	\end{equation}
\end{lemma}

\begin{proof}
	First suppose that $g\ge0$, and let $\mu_g(D)=\Ebb\{g(\bX)\ind\{Y(1)\in D\}\}$. Its total mass is at most $Ch^d$. If this mass is zero, the weighted sums vanish almost surely. Otherwise, quantiles of $\mu_g$ give at most $C/\vartheta$ threshold indicators, together with the constants $0$ and $1$, that bracket every $\ind\{Y(1)\le y\}$ with a $\mu_g$-integral gap of at most $\vartheta\mu_g(\Rbb)$. At atoms larger than this gap, include both strict and non-strict indicators.
	
	Because the weights and $g$ are nonnegative, the empirical sum at each $y$ lies between the two bracketing sums. Their expectation terms differ by at most $CMh^d\vartheta$. Apply \eqref{eq:weighted-epoch-scalar-tail} to the finite collection and
	take a union bound to obtain \eqref{eq:weighted-cdf-threshold-tail}. For signed $g$, apply the same argument separately to its positive and negative parts and adjust $C$. These quantiles are used only in the proof, not in computing the estimator.
\end{proof}

With $\vartheta=h^{\bar s}$ and $v=C_0L_k+u$, a further union bound over any collection whose cardinality is polynomial in $h^{-1}$ and $M_k$ gives logarithmic terms of order $L_k+u$.

\begin{lemma}[Spline approximation and matrix stability]
	\label{lem:weighted-spline-design}
	Fix $k$ and write $h=h^{(k)}$. The basis $\mathbf B^{(k)}$ has $J_k\asymp h^{-d}$ nonnegative, bounded functions, each supported on $O(1)$ neighboring mesh cubes and overlapping with at most $O(1)$ other basis functions. The basis functions sum to one. For each $g\in\mathcal H^s([0,1]^d,L)$, there is a coefficient vector $\boldsymbol\beta_g$ such that
	\begin{equation}
		\|g-{\mathbf B^{(k)}}^\top\boldsymbol\beta_g\|_\infty\le Ch^s,
		\qquad\|\boldsymbol\beta_g\|_\infty\le C,
		\qquad
		ch^d\bI\preceq\bGamma_k:=\Ebb\{\mathbf B^{(k)}(\bX)\mathbf B^{(k)}(\bX)^\top\}
		\preceq Ch^d\bI.
		\label{eq:spline-approximation-design}
	\end{equation}
	For $a\in\{0,1\}$ and $t\in\mathcal T_k$, whenever $\nu_t>0$ and \eqref{eq:weighted-regression-safeguard} holds,
	$$
	\|\bG_{a,t}^{-1}\|_\infty\le \frac{C}{\nu_th^d},
	\qquad
	\|\bG_{a,t}\mathbf1\|_\infty\le C\nu_th^d,
	$$
	where the matrix norm is the maximum absolute row sum.
\end{lemma}

\begin{proof}
	To establish the eigenvalue bounds, restrict the basis functions to each mesh cube and rescale the cube to unit size. The Gram matrices of the nonzero basis functions then have eigenvalues bounded above and away from zero, uniformly over the finitely many interior and boundary configurations. Summing over the mesh cubes and applying the bounds on the covariate density gives the matrix inequalities in \eqref{eq:spline-approximation-design}. For the approximation bound, choose a mesh cube $Q_j$ in the support of each basis function. Let $\Lambda_j(g)$ be the coefficient corresponding to that basis function in the local $L^2(Q_j)$ polynomial projection. Norm equivalence in finite dimensions gives $|\Lambda_j(g)|\le C\|g\|_{\infty,Q_j}$. The operator $\sum_jB_j^{(k)}\Lambda_j(g)$ reproduces polynomials of the required degree. Comparing $g$ with its Taylor polynomial on an $O(h)$ neighborhood and using $\sum_jB_j^{(k)}\equiv1$ proves the approximation bound, including at the boundary.
	
	If $\nu_t>0$ and \eqref{eq:weighted-regression-safeguard} holds, the eigenvalues of $A=G_{a,t}/(\nu_th^d)$ lie in a fixed
	positive interval. Choose $c_0>0$ such that $R=I-c_0A$ has spectral norm at most $\rho<1$. Since an entry of $A$ can be nonzero only for basis functions with overlapping support, each row of $R^j$ has at most $C(j+1)^d$ nonzero entries. Hence
	$$
	\|A^{-1}\|_\infty\le c_0\sum_{j\ge0}\|R^j\|_\infty\le C\sum_{j\ge0}(j+1)^{d/2}\rho^j<\infty .
	$$
	It follows that $\|G_{a,t}^{-1}\|_\infty\le C/(\nu_th^d)$ and $\|G_{a,t}\mathbf1\|_\infty\le C\nu_th^d$.
\end{proof}

\begin{lemma}[Local-linear approximation and population design]
	\label{lem:weighted-local-linear-design}
	Let
	$$\bGamma_h(\bx)=\Ebb\{K_{\bx,h}(\bX)\mathbf v_{\bx,h}(\bX)
	\mathbf v_{\bx,h}(\bX)^\top\}.$$ 
	Then
	\begin{equation}
		ch^d\bI_{d+1}\preceq\bGamma_h(\bx)\preceq Ch^d\bI_{d+1},
		\qquad\bx\in[0,1]^d,\ 0<h\le h_0 .
		\label{eq:local-linear-population-design}
	\end{equation}
	For $g\in\mathcal H^s([0,1]^d,L)$, define
	$$
	\boldsymbol\beta_g(\bx)=
	\begin{cases}
		(g(\bx),\mathbf0^\top)^\top,&s\le1,\\
		(g(\bx),h\nabla g(\bx)^\top)^\top,&s>1.
	\end{cases}
	$$
	Uniformly over $\bx\in[0,1]^d$,
	$\|\boldsymbol\beta_g(\bx)\|\le C$,
	$\mathbf e_0^\top\boldsymbol\beta_g(\bx)=g(\bx)$, and
	$$
	|g(\bz)-\mathbf v_{\bx,h}(\bz)^\top\boldsymbol\beta_g(\bx)|
	\le Ch^{s_{\mathrm{lin}}}
	\quad\text{whenever }K_{\bx,h}(\bz)\ne0.
	$$
	Consequently,
	$$
	\left\|\Ebb\{K_{\bx,h}(\bX)\mathbf v_{\bx,h}(\bX)g(\bX)\}
	-\Gamma_h(\bx)\boldsymbol\beta_g(\bx)\right\|
	\le Ch^{d+s_{\mathrm{lin}}}.
	$$
	The constants are uniform over $g$, and hence over $y$ when
	$g=F_1(y\mid\cdot)$.
\end{lemma}

\begin{proof}
	For the lower bound in \eqref{eq:local-linear-population-design},
	the change of variables $\bz=\bx+h\bu$ gives
	$$
	h^{-d}\Gamma_h(\bx)
	=\int_{\{\bu:\,\bx+h\bu\in[0,1]^d\}}
	K(\bu)\bp(\bu)\bp(\bu)^\top f_X(\bx+h\bu)\,d\bu .
	$$
	The integration domain contains one of finitely many fixed orthant portions of a small ball on which $K$ is bounded below. On each such portion, the integral of $\bp(\bu)\bp(\bu)^\top$ is positive definite. Together with $f_X\ge c_X$, this proves the lower bound uniformly in $\bx$, including at the boundary. The upper bound follows from $f_X\le C_X$ and the boundedness and bounded support of $K$.

	The approximation bounds follow from the H\"older condition when $s\le1$ and a first-order Taylor expansion when $s>1$. For $s>2$, bounded second derivatives give a remainder of order $h^2$. The extension to a fixed neighborhood in
	Definition~\ref{def:holder-class} ensures that these arguments also apply at boundary points. Boundedness of the derivatives and $h\le1$ give the coefficient bound. Finally, $\|\mathbf v_{\bx,h}\|\le C$ on the kernel support, whose $P_X$-probability is at most $Ch^d$. Integrating the approximation error therefore gives the stated population remainder.
\end{proof}

\section{Additional Details for the Numerical Analysis}
\label{app:jester-details}

The first two subsections provide the weighted regression
estimators and supplementary regret fits for the simulation
study in Section~\ref{sec:simulation-settings}. The remaining
three subsections describe the online nuisance estimation,
finite-population control reference and oracle, and evaluation
criteria for the real-data analysis in
Section~\ref{subsec:jester-real-data}.

\subsection{Weighted regression estimators}
\label{app:simulation-estimators}
This subsection describes the four weighted regression methods used to estimate the control reference $b_0$ in the simulation study in Section~\ref{sec:simulation-settings}.
For each treatment condition $a\in\{0,1\}$, let $\mathcal I_{a,t}$ index the observations available before time $t$, and let $w_i$ denote their observation weights. We consider the following four regression estimators.

\begin{enumerate}
    \item[LL] Weighted local-linear regression. At a query point $\bx$, define $\bu_{i,t}(\bx)=(\bX_i-\bx)/h_t$ and $\bp(\bu)=(1,\bu^\top)^\top$, where $h_t$ is the bandwidth. With local weights $q_{i,t}(\bx)=w_iK\{\bu_{i,t}(\bx)\}$, the coefficients minimize $\sum_{i\in\mathcal I_{a,t}}q_{i,t}(\bx)\{Y_i-\bp(\bu_{i,t}(\bx))^\top\boldsymbol\theta\}^2$. The fitted intercept gives $\widetilde f_{a,t}(\bx)=\mathbf e_0^\top\widehat{\boldsymbol\theta}_{a,t}(\bx)$, where $\mathbf e_0=(1,0,\ldots,0)^\top$.

    \item[B-spline] Weighted B-spline regression with a growing basis. Let $\mathbf B_t(\bx)$ denote the B-spline basis vector, whose dimension increases as observations accumulate. The coefficients are fitted by $\widehat{\boldsymbol\theta}_{a,t}\in\argmin_{\boldsymbol\theta}\sum_{i\in\mathcal I_{a,t}}w_i\{Y_i-\mathbf B_t(\bX_i)^\top\boldsymbol\theta\}^2$, giving $\widetilde f_{a,t}(\bx)=\mathbf B_t(\bx)^\top\widehat{\boldsymbol\theta}_{a,t}$.

    \item[RF] weighted random forest regression. If $c_{i,m}$ is the number of times observation $i$ is sampled for tree $m$, its fitting weight is $v_{i,m}=c_{i,m}w_i$. For condition $a$ at time $t$, let $\mathcal L_m(\bx)$ contain the observation indices in the leaf receiving $\bx$. The tree prediction is $\widehat\mu_{a,t,m}(\bx)=\{\sum_{i\in\mathcal L_m(\bx)}v_{i,m}Y_i\}/\{\sum_{i\in\mathcal L_m(\bx)}v_{i,m}\}$, and the forest averages the $M$ tree predictions, $\widetilde f_{a,t}(\bx)=M^{-1}\sum_{m=1}^M\widehat\mu_{a,t,m}(\bx)$.

    \item[kNN] Weighted nearest-neighbor regression. For each query $\bx$, observations are ordered by Euclidean distance $\|\bX_i-\bx\|_2$. The neighborhood $\mathcal N_{a,t}(\bx)$ is the smallest initial segment of this ordering satisfying $\sum_{i\in\mathcal N_{a,t}(\bx)}w_i\ge b_t$, where $b_t$ is a prescribed weight-mass target. The prediction is $\widetilde f_{a,t}(\bx)=\{\sum_{i\in\mathcal N_{a,t}(\bx)}w_iY_i\}/\{\sum_{i\in\mathcal N_{a,t}(\bx)}w_i\}$. 
\end{enumerate}

For each method, since $b_0=f_0$ in our simulation models, we set
$\widetilde b_{0,t}(\bx)=\widetilde f_{0,t}(\bx)$ and obtain
$\widehat b_{0,t}(\bx)$ by applying the projection in
\eqref{eq:weighted-nonparametric-nuisance}.

\subsection{Supplementary log--log fits for cumulative regret}
\label{app:simulation-loglog-fits}

This subsection provides supplementary log--log fits for the
cumulative regret trajectories presented in
Figures~\ref{fig:simulation-regret}--\ref{fig:stepsize-sensitivity}.

Figure~\ref{fig:simulation-regret-fit} presents supplementary log--log fits for the mean cumulative gap-weighted assignment regret shown in Figure~\ref{fig:simulation-regret}. The slope
from a linear fit on the log--log scale is reported for each curve as a summary of its empirical growth rate. All fitted regret slopes are below one over the plotted range. In M1
and M2, LL and B-spline yield smaller slopes than RF and kNN; in M3, LL has the smallest slope, while the other three methods have similar slopes. The bottom row of Figure~\ref{fig:simulation-regret-fit} shows that the exploration effect depends on the model. The fitted slopes decrease with $\chi$ in M1 and M2. In M3, however, the slope
increases with $\chi$. Despite this increase, larger $\chi$ yields lower cumulative regret over the plotted horizon.

\begin{figure}[htbp]
    \centering
    \includegraphics[width=\textwidth]{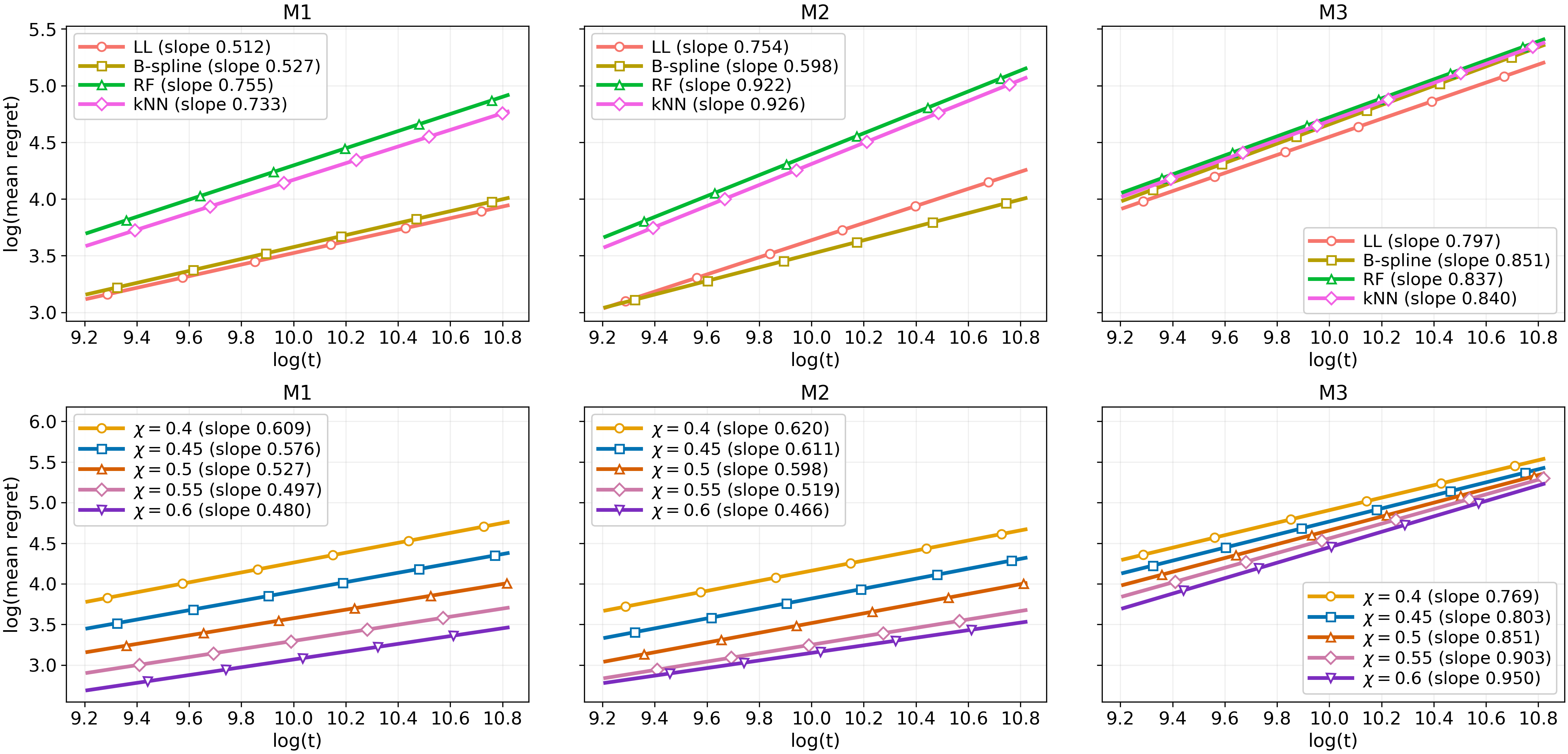}
    \caption{Supplementary log--log fits for the mean cumulative
    gap-weighted assignment regret in Models M1--M3,
    corresponding to Figure~\ref{fig:simulation-regret}.
    The top row compares four control-reference estimators at
    $\chi=0.5$, with weighted B-spline estimation of $F_1$
    throughout. The bottom row varies $\chi$, using weighted
    B-spline estimation for both nuisance functions.
    The lines show the fitted linear trends over the displayed
    time interval, and the legends report their slopes.}
    \label{fig:simulation-regret-fit}
\end{figure}

Figure~\ref{fig:history-sensitivity-fit} supplements the
historical-sample-size comparison in
Figure~\ref{fig:history-sensitivity}. As $n_{\mathrm{hist}}$
increases from $1000$ to $4000$, the fitted slope decreases
from $0.570$ to $0.501$ in M1, from $0.692$ to $0.511$ in M2,
and from $0.771$ to $0.735$ in M3. Thus, larger historical
samples are associated with both lower cumulative regret
toward the end of the simulated horizon and smaller fitted
growth rates over the displayed interval.

\begin{figure}[htbp]
    \centering
    \includegraphics[width=\textwidth]{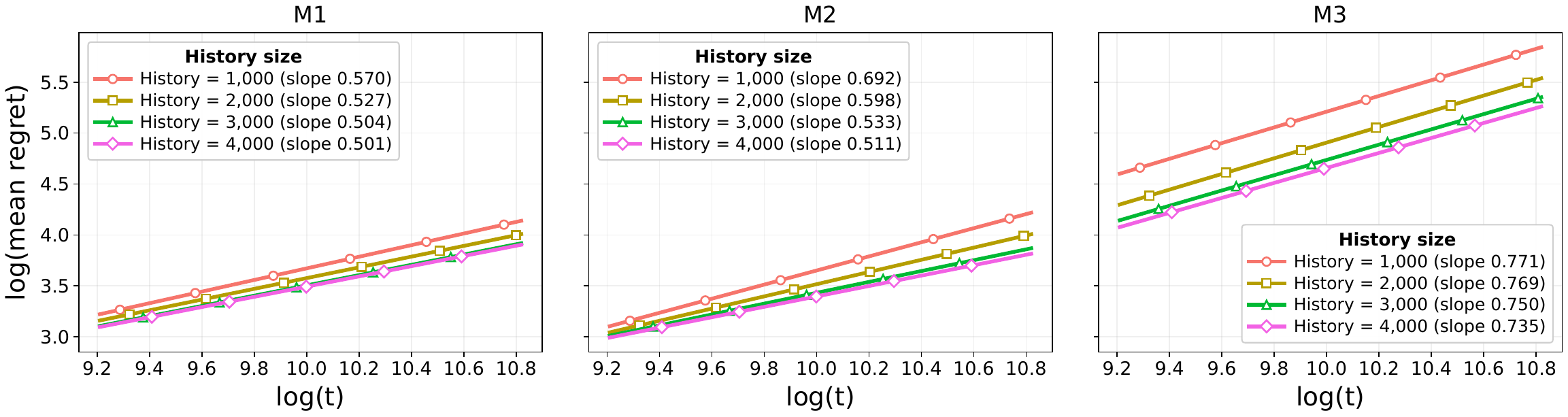}
    \caption{Supplementary log--log fits for the
    historical-sample-size comparison in
    Figure~\ref{fig:history-sensitivity}.
    The panels correspond to Models M1--M3, with
    $n_{\mathrm{hist}}\in\{1000,2000,3000,4000\}$.
    Both nuisance functions are estimated by weighted
    B-spline regression, with $\chi=0.5$ in M1 and M2 and
    $\chi=0.4$ in M3. The lines show fitted linear trends
    in the log mean cumulative gap-weighted assignment regret
    against $\log t$ over the displayed time interval.
    The legends report the fitted slopes.}
    \label{fig:history-sensitivity-fit}
\end{figure}

Figure~\ref{fig:stepsize-sensitivity-fit} supplements the
step-size comparison in Figure~\ref{fig:stepsize-sensitivity}.
Across the five values of $\zeta$, the fitted slopes range
from $0.519$ to $0.535$ in M1, from $0.597$ to $0.601$ in M2,
and from $0.762$ to $0.772$ in M3. These small within-model
differences are consistent with the limited sensitivity of
the empirical regret trajectories to the step-size exponent.

\begin{figure}[htbp]
    \centering
    \includegraphics[width=\textwidth]{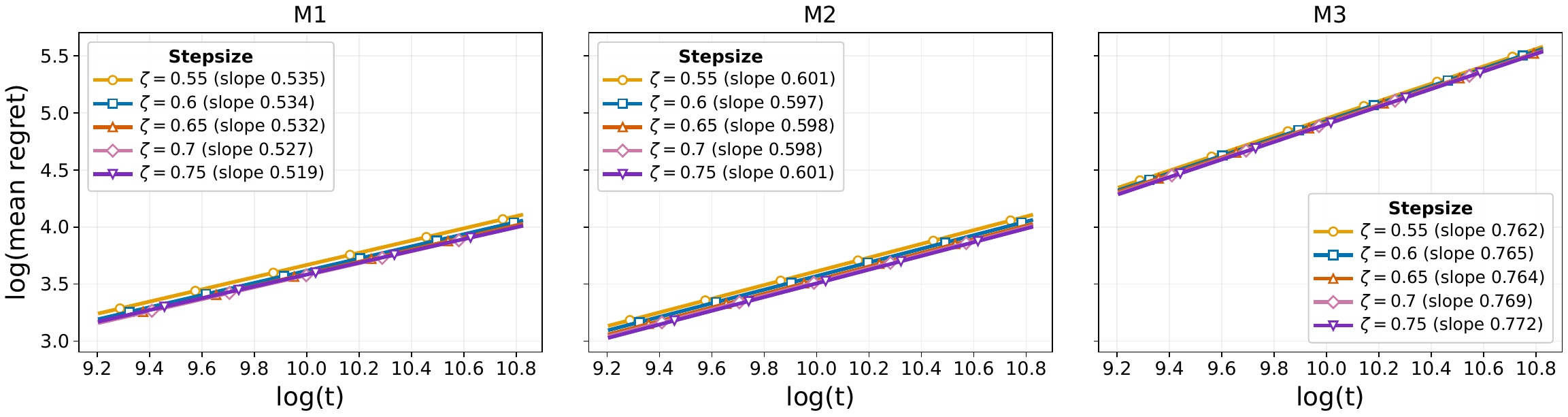}
    \caption{Supplementary log--log fits for the step-size
    comparison in Figure~\ref{fig:stepsize-sensitivity}.
    The panels correspond to Models M1--M3, with
    $\zeta\in\{0.55,0.60,0.65,0.70,0.75\}$.
    Both nuisance functions are estimated by weighted
    B-spline regression, with $\chi=0.5$ in M1 and M2 and
    $\chi=0.4$ in M3. The lines show fitted linear trends
    in the log mean cumulative gap-weighted assignment regret
    against $\log t$ over the displayed time interval.
    The legends report the fitted slopes.}
    \label{fig:stepsize-sensitivity-fit}
\end{figure}
\subsection{Online nuisance estimation}
\label{app:jester-online-nuisance}

For the real-data analysis in
Section~\ref{subsec:jester-real-data}, the nuisance estimators
use the treatment-specific samples $\cD_{a,t}$ defined in
Section~\ref{subsec:plugin-selection}, indexed by
$\mathcal I_{a,t}$, and the fitting weights $w_i$ in
\eqref{eq:weighted-regression-weights}. The observation weights are computed using the assignment
probabilities at the time of data collection and remain
fixed in subsequent refits.

\paragraph{Control reference.}
The control reference is estimated using a weighted tensor-product spline.
Let $\mathbf B(\bx)$ denote the tensor-product B-spline basis evaluated
at $\bx$.  Immediately before time $t$, define
\begin{equation}
\begin{aligned}
(\widehat c_{0,t},\widehat{\boldsymbol\beta}_{0,t})
\in\argmin_{c,\boldsymbol\beta}
\left\{
\frac{
\sum_{i\in\mathcal I_{0,t}}
w_i
\left[
Y_i-c-\mathbf B(\bX_i)^\top\boldsymbol\beta
\right]^2
}{
\sum_{i\in\mathcal I_{0,t}}w_i
}
+
\rho\|\boldsymbol\beta\|_2^2
\right\},
\end{aligned}
\label{eq:jester-online-b0-fit}
\end{equation}
where $\rho\ge0$ is the penalty parameter and the intercept is not penalized.
The estimated control reference is
\begin{equation}
    \widehat b_{0,t}(\bx)
    =
    \widehat c_{0,t}
    +
    \mathbf B(\bx)^\top
    \widehat{\boldsymbol\beta}_{0,t}.
    \label{eq:jester-online-b0}
\end{equation}

\paragraph{Conditional treated-outcome distribution.}
For a query point $\bx$, let
$\mathcal N_{1,t}(\bx)\subset\mathcal I_{1,t}$ index a local neighborhood of
the observed treated users, formed according to Euclidean distance in the
two-dimensional scaled covariate space.  We estimate the conditional
distribution of the treatment outcome by the weighted local empirical CDF
\begin{equation}
    \widehat F_{1,t}(y\mid\bx)
    =
    \frac{
    \sum_{i\in\mathcal N_{1,t}(\bx)}
    w_i\ind\{Y_i\le y\}
    }{
    \sum_{i\in\mathcal N_{1,t}(\bx)}w_i
    }.
    \label{eq:jester-online-F1}
\end{equation}
Substituting \eqref{eq:jester-online-b0} and \eqref{eq:jester-online-F1}
into \eqref{eq:paits-plugin-local-null-risk} gives the local null risk
estimate.  In this calculation, every outcome in the local treated
neighborhood is compared with the same query-specific control reference
$\widehat b_{0,t}(\bx)$, because the conditional distribution is evaluated
at the reference for the query user.

\subsection{Finite-population control reference and oracle}
\label{app:jester-population-oracle}

The control reference and local null risk estimates used for offline
evaluation are constructed once from the retained population of $N$ users
described in Section~\ref{subsec:jester-real-data} and held fixed across
all online replay orders.

For the control reference, we use cross-fitting.  Partition the population
into folds, and let $k(i)$ denote the fold containing user $i$.  Let
$\widehat b_0^{(-k)}$ be the control spline fitted using all users outside fold
$k$.  The fixed control reference for user $i$ is
\begin{equation}
    \widetilde b_{0,i}
    =
    \widehat b_0^{(-k(i))}(\bX_i).
    \label{eq:jester-pop-b0}
\end{equation}
Thus, user $i$'s own joke-17 rating is not used to construct its control
reference. We form a fixed local empirical distribution for the treated
outcome.  Let $\mathcal N_{\mathrm{pop}}(\bx)$ denote the local neighborhood
of $\bx$ in the complete population.  Define
\begin{equation}
    \widetilde F_{1,\mathrm{pop}}(y\mid\bx)
    =
    \frac{1}{|\mathcal N_{\mathrm{pop}}(\bx)|}
    \sum_{j\in\mathcal N_{\mathrm{pop}}(\bx)}
    \ind\{Y_j(1)\le y\},
    \label{eq:jester-pop-F1}
\end{equation}
and the corresponding finite-population local null risk by
\begin{equation}
    \widetilde\ell_i
    =
    \widetilde F_{1,\mathrm{pop}}
    (\widetilde b_{0,i}\mid\bX_i).
    \label{eq:jester-pop-risk}
\end{equation}

Using the estimated finite-population risk scores
$\{\widetilde\ell_i\}_{i=1}^N$, we construct the corresponding plug-in
oracle policy following Lemma~\ref{lem:oracle-assignment-optimality}.
Because the estimated scores are discrete, randomization is allowed among
users at the boundary when necessary.  We denote the resulting oracle cutoff
by $\widetilde\lambda^\star$.

\subsection{Evaluation criteria}
\label{app:jester-evaluation}

Suppose user $i_s$ arrives at online time $s$.  The cumulative FDP is
computed as in Section~\ref{subsec:treatment-error-regret}, using the
fixed control reference $\widetilde b_{0,i_s}$ in place of $b_0(\bX_s)$.
Thus, an assignment with $A_s=1$ is counted as a false discovery when
$Y_{i_s}(1)\le\widetilde b_{0,i_s}$.  The complete ratings are used only
for offline evaluation; PAITS receives the assigned outcome at each arrival.

For assignment efficiency, we use the pathwise finite-population counterpart
of the gap-weighted regret in
\eqref{eq:weighted-selection-regret}:
\begin{equation}
    \Regret_t^{\mathrm{pop}}
    =
    \sum_{s=1}^t
    \left|
    \widetilde\ell_{i_s}
    -
    \widetilde\lambda^\star
    \right|
    \left|
    A_s-\widetilde A_{i_s}^\star
    \right|,
    \label{eq:jester-gap-regret}
\end{equation}
where $\widetilde A_i^\star$ is the plug-in finite-population oracle
assignment.
The convention used for users exactly
at the oracle boundary does not affect \eqref{eq:jester-gap-regret}, because
their gap weight is zero.
For the descriptive log--log plots, we fit
$\log\Regret_t^{\mathrm{pop}}=a+b\log t+e_t$ over time points with positive
regret.  The fitted slope $b$ summarizes finite-horizon regret growth.
\end{document}